\documentclass[11pt]{article}

\usepackage{float}
\usepackage[all]{xy}

\usepackage{amsmath,amssymb,amsthm}
\usepackage{latexsym}
\usepackage[dvipdfmx]{graphicx}
\usepackage[ruled]{algorithm2e}
\usepackage{algorithmic} 
\usepackage{authblk}
\usepackage{color}
\usepackage{relsize}
\usepackage{url}
\usepackage[colorlinks,linkcolor=blue,anchorcolor=green,citecolor=green]{hyperref}
\usepackage{caption} 
\usepackage{tikz}
\usepackage{yhmath}
\usepackage{mathrsfs,booktabs}

\usepackage[T1]{fontenc}
\usepackage[utf8]{inputenc}
\usepackage[font=small,labelfont=bf,tableposition=top]{caption}
\usepackage{booktabs}
\usepackage{threeparttable} 
\usepackage{multirow}

\makeatletter

\newcommand{\Rmnum}[1]{\expandafter\@slowromancap\romannumeral #1@}
\makeatother

\renewcommand{\algorithmicrequire}{\textbf{Input:}}

\newtheorem{theorem}{Theorem}

\newtheorem{lemma}{Lemma}

\newtheorem{proposition}{Proposition}
\newtheorem{definition}{Definition}

\newtheorem{assumption}{Assumption}
\usepackage{color}

\usepackage{tikz,pgfplots}
\pgfplotsset{compat=newest}
\pgfplotsset{plot coordinates/math parser=false,trim axis left}
\newlength\figureheight
\newlength\figurewidth

\newcommand{\eins}{\boldsymbol{1}}

\newcommand{\argmax}{\operatornamewithlimits{arg \, max}}
\newcommand{\argmin}{\operatornamewithlimits{arg \, min}}
\newcommand{\supp}{\mathrm{supp}}

\usepackage{siunitx}
\usepackage{subfig}

\author[]{Hongwei Wen}
\author[]{Annika Betken}
\author[]{Hanyuan Hang}

\date{\today}

\affil[]{Faculty of Electrical Engineering, Mathematics and Computer Science \\ 
University of Twente, The Netherlands \\
{}
}

\begin{document}

\title{Class Probability Matching Using Kernel Methods for \\ Label Shift Adaptation}



\maketitle

%

\allowdisplaybreaks

\begin{abstract}In domain adaptation, covariate shift and label shift problems are two distinct and complementary tasks. In covariate shift adaptation where the differences in data distribution arise from variations in feature probabilities, existing approaches naturally address this problem based on \textit{feature probability matching} (\textit{FPM}). However, for label shift adaptation where the differences in data distribution stem solely from variations in class probability, current methods still use FPM on the $d$-dimensional feature space to estimate the class probability ratio on the one-dimensional label space. 
	To address label shift adaptation more naturally and effectively, inspired by a new representation of the source domain's class probability, we propose a new framework called  \textit{class probability matching} (\textit{CPM}) which matches two class probability functions on the one-dimensional label space to estimate the class probability ratio, fundamentally different from FPM operating on the $d$-dimensional feature space. Furthermore, by
	incorporating the kernel logistic regression into the CPM framework to estimate the conditional probability, we propose an algorithm called \textit{class probability matching using kernel methods} (\textit{CPMKM}) for label shift adaptation. From the theoretical perspective, we establish the optimal convergence rates of CPMKM with respect to the cross-entropy loss for multi-class label shift adaptation. From the experimental perspective, comparisons on real datasets demonstrate that CPMKM outperforms existing FPM-based and maximum-likelihood-based algorithms.
\end{abstract}

\section{Introduction} \label{sec::Introduction}

The current success of machine learning relies on the availability of a large amount of labeled data. However, high-quality labeled data are often in short supply. 
Therefore, we need to borrow labeled data or extract knowledge from some related domains to help a machine learning algorithm achieve better performance in the domain of interest, which is called \textit{domain adaptation} \cite{mansour2009domain,ying2018transfer,zhang2020collaborative}. 
Domain adaptation can be applied to a wide range of areas such as image analysis \cite{long2017deep}, natural language processing \cite{devlin2018bert}, medical diagnosis \cite{xu2011survey} and recommendation systems \cite{pan2010transfer}.  
A typical protocol of domain adaptation involves two domains of data: a large amount of labeled data from a source distribution $P$ and unlabeled data from a target distribution $Q$ on the product space $\mathcal{X} \times \mathcal{Y}$. The task is to conduct classification in the target domain based on both domains of data.

In domain adaptation, \textit{covariate shift} \cite{shimodaira2000improving, sugiyama2007covariate, kanamori2009least, kpotufe2021marginal} and \textit{label shift} \cite{saerens2002adjusting, storkey2009training, tasche2017fisher, maity2022minimax} are common sources of performance degradation when adapting models to new domains. Understanding these shifts helps diagnose why a model might perform poorly in a target domain and provides insights into which adaptation techniques are most appropriate.
As the names suggest, these two shifts are distinct and complementary tasks. 
On the one hand, covariate shift indicates the distribution shift of the covariate $x$, that is, the distribution of the covariate $x$ varies ($q(x) \neq p(x)$) while the conditional probabilities remain the same ($q(y|x) = p(y|x)$). 
On the other hand, label shift indicates 
the distribution shift of the label $y$,
that is, the distribution of the label $y$ changes ($q(y) \neq p(y)$) while the class-conditional probabilities do not change $(q(x|y) = p(x|y))$. 
\cite{scholkopf2012causal} points out that 
covariate shift corresponds to causal learning (predicting effects) whereas label shift corresponds to anticausal learning (predicting causes).

Under these two distribution shift assumptions, the most common starting point for existing methods used in domain adaptation is to estimate the joint probability ratio $q(x, y)/p(x, y)$ between the source and target domains. 
Since the data from the source distribution $p(x,y)$ are observed, if we have knowledge of the probability ratio $q(x,y)/p(x,y)$, then the information from the source domain can be transferred to the target domain to facilitate predictions.
According to the conditional probability formula, one finds that under covariate shift, the joint probability ratio is equal to the feature probability ratio $q(x)/p(x)$, since
\begin{align} \label{eq::featureratio}
	\frac{q(x, y)}{p(x, y)} 
	= \frac{q(y|x) \cdot q(x)}{p(y|x) \cdot p(x)} 
	= \frac{q(x)}{p(x)},
\end{align}
whereas, in contrast, under label shift, the joint probability ratio becomes the class probability ratio $q(y)/p(y)$, since
\begin{align} \label{eq::labelratio}
	\frac{q(x, y)}{p(x, y)} 
	= \frac{q(x|y) \cdot q(y)}{p(x|y) \cdot p(y)} 
	= \frac{q(y)}{p(y)}. 
\end{align}
Equation \eqref{eq::featureratio} reveals that the difference in the source and target domain distributions under covariate shift arises solely from the difference in feature probability, whereas Equation
\eqref{eq::labelratio} indicates that
under label shift it stems from the difference in label probabilities.

For these reasons, when addressing the covariate shift problem, as stated in \eqref{eq::featureratio}, the goal is to estimate the feature probability ratio. Therefore,
existing methods, as described e.g. in  \cite{huang2006correcting}, typically start from matching feature probabilities and thus  this type of methods can be referred to as \textit{feature probability matching} (\textit{FPM}) methods. 
FPM constructs a matching equation between the feature probability $q(x)$ and the weighted feature probability $p(x)$ from the source domain to estimate the feature probability ratio.
In order to implement FPM, \textit{kernel mean matching} (\textit{KMM}) \cite{huang2006correcting} minimizes the distance between the kernel mean of reweighted source data and target data. Moreover, \cite{stojanov2019low} first use a mapping function to reduce the feature to a low-dimensional representation and then use KMM for estimating the ratio of the representation. Furthermore, \cite{martin2023double}  adaptively estimate the feature probability ratio $ p(x)/q(x)$ or $q(x)/p(x)$ according to their values.

On the other hand, for label shift problems, as stated in \eqref{eq::labelratio}, the goal is to estimate the class probability ratio, which is a probability ratio on label $Y$ rather than feature $X$. Despite this, existing matching methods still follow the FPM framework, starting from $q(x)$ and obtaining the label probability ratio through matching the feature probability $q(x)$ and the weighted class-conditional feature probability $p(x|y)$. For example, \cite{zhang2013domain} borrows the kernel mean matching method used in covariate shift problem \cite{huang2006correcting}, while \cite{guo2020ltf} trains generative adversarial networks to implicitly learn the feature distribution. However, for large-scale datasets, the computational cost of these two methods can be extremely high. To reduce the computational complexity, \cite{lipton2018detecting,azizzadenesheli2019regularized,tian23a} introduce a mapping function $h$ to transform the feature variable $X$ into a low-dimensional variable $h(X)$. Then, moment matching is applied to match the probability of the transformed feature $q(h(x))$ with the weighted class-conditional probability $p(h(x)|y)$ to obtain the class probability ratio. However, the optimal choice of the mapping function $h$ remains uncertain.

Under such background, we establish a new representation for the label probability $p(y)$ by utilizing the representation of $q(x)$ in FPM and then introduce a new \textit{class probability matching} (\textit{CPM}) framework to estimate the class probability ratio for label shift adaptation. 
In contrast to FPM, which matches two distributions on a $d$-dimensional feature space $\mathcal{X}$, CPM is a more straightforward and natural idea for label shift adaptation, since it only requires matching two distributions on a one-dimensional label space $\mathcal{Y}$. 
More specifically, CPM only needs to solve an equation system due to the discreteness of the label space, which effectively avoids potential issues associated with FPM in the feature space. Since CPM requires information about the conditional probability $p(y|x)$, we apply truncated \textit{kernel logistic regression} (\textit{KLR}) to estimate it in the source domain, where KLR is truncated downwards to ensure its CE loss bounded. By incorporating CPM with truncated KLR, we propose a new algorithm named \textit{class probability matching using kernel methods} (\textit{CPMKM}) for label shift adaptation. Specifically, the initial step is to estimate the class probability ratio based on the KLR estimator, while the subsequent step is to obtain the corresponding classifier for the target domain.

The contributions of this paper are summarized as follows.

\textit{(i)}
Starting from a representation of the class probability $p(y)$, we construct the new matching framework CPM for estimating the class probability ratio $q(y)/p(y)$, which avoids potential issues associated with FPM methods. More specifically, we first use the law of total probability to establish a representation of $p(y)$. Then by taking full advantage of the representation of $q(x)$ in FPM, the feature probability ratio $p(x)/q(x)$ in the representation of $p(y)$ can be expressed as the reciprocal of a linear combination of the conditional probability function $p(y|x)$, where the coefficient in front of $p(y|x)$ is precisely the class probability ratio $q(y)/p(y)$. Taking a step further, we obtain a new representation for $p(y)$, which is the expectation of a function concerning $p(y|x)$ and $q(y)/p(y)$ with respect to the probability measure $Q_X$. Based on this new representation, we introduce the CPM that aligns two distributions of the one-dimensional label variable $Y$ for label shift adaptation. In this way, our CPM effectively avoids potential issues associated with FPM methods which aim to match two distributions in the $d$-dimensional feature space. Finally, by incorporating KLR into the CPM framework to estimate the conditional probability, we obtain our new algorithm CPMKM for label shift adaptation.

\textit{(ii)} 
From the theoretical perspective, we establish the optimal convergence rates for CPMKM for label shift adaptation by establishing the optimal rates for truncated KLR, which to the best of our knowledge, is the first convergence result of KLR w.r.t.~the unbounded CE loss.
More precisely, we first show that the excess CE risk of CPMKM depends on the excess CE risk of the truncated KLR and the class probability ratio estimation error. 
Under the linear independence assumption,
we show the identifiability of CPM and that the class probability ratio estimation error depends on both the excess CE risk of the truncated KLR and the sample size in the target domain. 
Therefore, to establish the convergence rates of CPMKM, it suffices to derive the convergence rates of the truncated KLR, which can be achieved by establishing a new oracle inequality for the truncated estimator w.r.t.~the CE loss. 
To cope with the unboundedness of the CE loss, we decompose the CE loss into an upper part and a lower part depending on whether the true conditional probability $p(y|x)$ is greater or less than a certain value. 
The upper part of the CE loss is bounded for $p(y|x)$ and thus the concentration inequality can be applied to the loss difference for analyzing the excess risk on this part. 
On the other hand, since the CE loss of $p(y|x)$ in the lower part is unbounded, we apply the concentration inequality to the loss of the truncated estimator rather than the loss difference for analysis on this part.
As a result, we succeed in establishing a new oracle inequality with a finite sample error for the truncated KLR w.r.t.~the CE loss.
By deriving the approximation error of the truncated KLR, we are able to obtain its optimal convergence rates. 
Finally, by utilizing the convergence rates of the truncated KLR, we obtain the optimal convergence rates for CPMKM for label shift adaptation.

\textit{(iii)}
Through numerical experiments under various label shift scenarios, we find that our CPMKM outperforms existing FPM-based methods and maximum-likelihood-based approaches in both the class probability estimation error and the classification accuracy in the target domain, especially for the dataset with a large number of classes. Furthermore, we explore the effect of the sample size on the performance of compared methods. Specifically, with a fixed number of source domain data, we observe an initial performance improvement as the sample size of unlabeled target domain data increases, followed by a stabilization phase. This trend verifies the convergence rates established for label shift adaptation.

The remainder of this paper is organized as follows.
In Section \ref{sec::preliminary}, we formulate the domain adaptation problem, state the label shift assumption, and revisit the FPM framework.
In Section \ref{sec::Methodology}, we develop the new matching framework CPM that directly matches on the the label $Y$ to estimate the class probability ratio $q(y)/p(y)$.
By incorporating KLR with the matching framework CPM, we propose the algorithm CPMKM for label shift adaptation. 
In Section \ref{sec::TheoreticalResults}, we establish the convergence rates of CPMKM and provide some comments and discussions concerning theoretical results.
In Section \ref{sec::ErrorAnalysis}, we present the error analysis for CPMKM.
In Section \ref{sec::Experiments}, we conduct some numerical experiments to illustrate the superiority of our proposed CPMKM over compared methods.
All the proofs of Sections \ref{sec::TheoreticalResults} and \ref{sec::ErrorAnalysis} can be found in Section \ref{sec::Proofs}.
We conclude this paper in Section \ref{sec::Conclusion}.

\section{Preliminaries} \label{sec::preliminary}

\subsection{Notations}

For $1 \leq p < \infty$, the $L_p$-norm of $x = (x_1, \ldots, x_d)$ is defined as $\|x\|_p := (|x_1|^p + \ldots + |x_d|^p)^{1/p}$, and the $L_{\infty}$-norm is defined as $\| x \|_{\infty} := \max_{i=1,\ldots,d} |x_i|$.
For any $x \in \mathbb{R}^d$ and $r > 0$, we explicitly denote $B_r(x) := B(x,r) := \{ x' \in \mathbb{R}^d : \|x' - x\|_2 \leq r \}$ as the closed ball centered at $x$ with radius $r$. In addition, denote $\mu(A)$ as the Lebesgue measure of the set $A \subset \mathbb{R}^d$.
We use the notation $a_n \lesssim b_n$ to denote that there exists a constant $c \in (0, 1]$ such that $a_n \leq c^{-1} b_n$, for all $n \in \mathbb{N}$. 
Similarly, $a_n \gtrsim b_n$ denotes that there exists some constant $c \in (0, 1]$ such that $a_n \geq c b_n$. 
In addition, the notation $a_n \asymp b_n$ means that there exists some positive constant $c \in (0, 1]$, such that $c b_n\leq a_n \leq c^{-1} b_n$, for all $n \in \mathbb{N}$.
In addition, the cardinality of a set $A$ is denoted by $\#(A)$.
For any integer $M \in \mathbb{N}$ denote $[M] := \{1, 2, \ldots, M\}$. For any $a,b \in \mathbb{R}$, we denote $a \wedge b := \min\{a,b\}$ and $a \vee b := \max\{a,b\}$ as the smaller and larger value of $a$ and $b$, respectively. Denote the $(M-1)$-dimensional simplex as $\Delta^{M-1} := \{\theta \in \mathbb{R}^M: \sum_{m=1}^M \theta_m = 1, \theta_m \geq 0, m\in [M]\}$.  

For the domain adaptation problem, let the input space $\mathcal{X} \subset \mathbb{R}^d$ and the output space $\mathcal{Y} = [M]$. Moreover, let $P$ and $Q$ be the source and target distribution defined on $\mathcal{X} \times \mathcal{Y}$, respectively. 
For the target domain, we denote $q(y) := Q_Y(y) := Q(Y = y)$, $y \in [M]$ as the class probability and $q(x) := Q_X(x) := Q(X = x)$ as the marginal probability. 
It is well-known that the conditional probability $q(y|x) := Q(Y = y | X = x)$ is the \textit{optimal predictor} and the corresponding \textit{optimal classifier} on the target domain is 
\begin{align}\label{OptimalClassifier}
	h_q(x) := \argmax_{y \in [M]} q(y|x).
\end{align}
Finally, let $q(x|y) := Q(X = x | Y = y)$ denote the class-conditional probability. 
The notations for the source domain 
such as $p(y)$, $p(x)$, $p(y|x)$, $h_p(x)$, and $p(x|y)$,
can be defined analogously.

\subsection{Label Shift Adaptation}

In this paper, we aim to solve the domain adaptation problem under the label shift setting \cite{saerens2002adjusting,storkey2009training}, where the class-conditional probabilities of $P$ and $Q$ are the same whereas the class probabilities  differ. 
\begin{assumption}[\textbf{Label shift}]\label{ass:labelshift}
	Let $P$ and $Q$ be two probability distributions defined on $\mathcal{X} \times \mathcal{Y}$. 
	We say that $P$ and $Q$ satisfy the label shift assumption if 
	$p(x|y) = q(x|y)$
	and 
	$p(y) \neq q(y)$.
\end{assumption}

The label shift assumption is made from the perspective of the label variable $Y$, the distribution of features remains the same for a fixed class, but there is a shift in the overall distribution of labels across different domains. 
We give an example to illustrate the label shift problems.
For example, imagine that the feature $x$ represents symptoms of a disease and that the label $y$ indicates whether a person has been infected with the disease.
Moreover, assume that the distributions $P$ and $Q$ correspond to the joint distribution of symptoms and infections in different hospitals, e.g. in different locations, that adopt distinct prevention and control measures so that the disease prevalence differs, i.e. $p(y) \neq q(y)$.
At the same time, it is reasonable to assume that the symptoms of the disease and the mechanism that symptoms caused by diseases are the same in both places, i.e. $p(x|y) = q(x|y)$. 
To make a diagnostic model based on data from one of the hospitals working in the other hospital it becomes therefore crucial to study label shift adaptation.

In domain adaptation problems, the labeled data from the target domain is not accessible, i.e., we only observe labeled data points $D_p :=(X_i,Y_i)_{i=1}^{n_p} \in\mathcal{X}\times\mathcal{Y}$ from the source distribution $P$ and unlabeled data points $D_q^u :=(X_i)_{i=n_p+1}^{n_p+n_q}$ from the target distribution $Q$. 
Based on the observations $D := (D_p,  D_q^u)$, our goal is to find a classifier $\widehat{h}_q$ for the target domain. Note that, for this, it suffices to find the estimators $\widehat{q}(y|x)$ such that the induced classifier is given by
\begin{align}\label{eq::fhat}
	\widehat{h}_q(x) 
	:= \argmax_{y\in[M]} \widehat{q}(y|x).
\end{align}

In order to evaluate the quality of the estimated predictor $\widehat{q}(y|x)$, we employ the commonly used cross-entropy loss $L_{\mathrm{CE}}(y, \widehat{q}(\cdot|x)):= -\log \widehat{q}(y|x)$. Then, the corresponding risk is given by $\mathcal{R}_{L_{\mathrm{CE}},Q}(\widehat{q}(y|x)) := \int_{\mathcal{X} \times \mathcal{Y}} L_{\mathrm{CE}}(y, \widehat{q}(\cdot|x)) \, dq(x,y)$ and the minimal risk is defined as $\mathcal{R}_{L_{\mathrm{CE}},Q}^* := \min_{\widehat{q}(y|x)} \mathcal{R}_{L_{\mathrm{CE}},Q}(\widehat{q}(y|x))$. It is well-known that the conditional probability $q(y|x)$ is the optimal predictor which achieves the minimal risk w.r.t.~$L_{\mathrm{CE}}$ and $Q$, i.e., $\mathcal{R}_{L_{\mathrm{CE}},Q}^* = \mathcal{R}_{L_{\mathrm{CE}},Q}(q(y|x)) = - \mathbb{E}_{x \sim q} \sum_{y \in [M]} q(y|x)\log q(y|x)$.

\subsection{Feature Probability Matching} \label{subsec::FPM}

As discussed in the introduction, under the covariate shift assumption, 
as indicated in \eqref{eq::featureratio},
the joint probability ratio $q(x,y)/p(x,y)$ transforms into the feature probability ratio 
\begin{align*}
	w^*(x) := q(x)/p(x).
\end{align*}
In other words, the difference between the distributions of the source and target domain is solely derived from the difference in feature probabilities in this case. Consequently, if we can obtain the feature probability ratio $w^*(x)$, information from the source domain can be transferred for the prediction on the target domain.
A direct method to estimate $w^*(x)$ is separately estimating the feature probabilities $p(x)$ and $q(x)$ in both domains and then calculating their ratio. However, estimating density explicitly in the feature space is difficult for high-dimensional data. To address this challenge, existing methods typically employ matching probability density functions in the feature space, a technique known as \textit{feature probability matching} (\textit{FPM}).
Specifically, they noticed that $q(x)$ can be represented as
\begin{align}\label{eq::qXreprecovariate}
	q(x) 
	= (q(x) / p(x)) \cdot p(x) 
	= w^*(x) p(x).
\end{align}
Therefore, $w^*(x)$ can be obtained by finding a weight function $w : \mathcal{X} \to \mathbb{R}$ satisfying the equation given by
\begin{align}\label{eq::matchXcovariate}
	q(x) 
	= w(x) p(x) 
	=: p^w(x).
\end{align}
By constructing the matching equation \eqref{eq::matchXcovariate}, FPM relates the feature probability $q(x)$ on the target domain to the weighted form of feature probability $p(x)$ on the source domain. Consequently, the estimation problem of $w^*(x)$ is transformed into solving the matching equation \eqref{eq::matchXcovariate} concerning feature probability.
Since the difference between the source and target domains in covariate shift only arises from the difference between feature probability density functions $p(x)$ and $q(x)$ as shown in \eqref{eq::featureratio}. 
Therefore, FPM as in \eqref{eq::matchXcovariate} ingeniously addresses the issue of covariate shift.

On the other hand, under the label shift assumption, as indicated in \eqref{eq::labelratio}, the joint probability ratio $q(x,y)/p(x,y)$ transforms into the class probability ratio 
\begin{align}\label{eq::wstar}
	w^*(y) 
	:= q(y)/p(y), 
\end{align}
which is solely related to the one-dimensional label distributions of the source and target domains.
However, for estimating the class probability ratio $w^*(y)$, many existing works still start from the representation of $q(x)$ to construct the feature matching equation.
Specifically, they noted that, 
different from \eqref{eq::qXreprecovariate},
$q(x)$ can be alternatively represented as 
\begin{align}\label{eq::qXrepre}
	q(x) 
	& = \sum_{y=1}^M q(y) q(x|y) 
	\qquad \qquad \qquad \, \, \, 
	(\text{law of total probability})
	\nonumber\\
	& = \sum_{y=1}^M w^*(y) p(y) q(x|y) 
	\qquad \qquad
	(\text{definition of $w^*(y)$})
	\nonumber\\
	& = \sum_{y=1}^M w^*(y) p(y) p(x|y) 
	\qquad \qquad
	(\text{label shift assumption}),
\end{align}
i.e., $q(x)$ is a linear combination of $\{ p(y) p(x|y) \}_{y\in [M]}$. 
If $\{ p(y) p(x|y) \}_{y\in[M]}$ are linearly independent, then $\{ w^*(y) \}_{y\in [M]}$ are the unique coefficients in the linear combination. 
In this case, motivated by the representation \eqref{eq::qXrepre}, FPM aims to determine $w^*(y)$ by finding the weight function $w := (w(y))_{y\in [M]}$ satisfying 
\begin{align}\label{eq::matchX}
	q(x) 
	= \sum_{y=1}^M w(y) p(y) p(x|y) 
	=: p^w(x).
\end{align}
Equation \eqref{eq::matchX} determines the class probability ratio $w^*(y)$ by matching $q(x)$ with the weighted combination of $p(x|y)$. Combining \eqref{eq::matchX} and \eqref{eq::matchXcovariate}, we observe that the form of FPM under label shift is similar to its form under covariate shift, both starting from the representation of the feature probability $q(x)$.

However, the label shift problem, as a complementary issue to the covariate shift problem, has a fundamentally different learning goal. Specifically, in the label shift problem, as illustrated in \eqref{eq::labelratio}, the differences between the source and target domains arise solely from the distinct class probabilities $p(y)$ and $q(y)$ of the one-dimensional labels. In contrast, in the covariate shift problem, as depicted in \eqref{eq::featureratio}, the discrepancy between the source and target domains stems exclusively from the probability functions $p(x)$ and $q(x)$ on the $d$-dimensional feature space. Consequently, FPM is not an ingenious approach under the label shift assumption.
Furthermore, implementation challenges arise for FPM since \eqref{eq::matchX} operates on the $d$-dimensional feature space. Kernel methods \cite{zhang2013domain}, which are commonly used in FPM for high-dimensional datasets, result in high computational complexity when dealing with large-scale datasets. To address this, moment matching is applied to a transformed feature space with lower dimensionality. However, the optimal choice of the transformation map $h$ remains uncertain across different datasets \cite{lipton2018detecting,azizzadenesheli2019regularized,tian23a}.
Therefore, to tackle these issues, we propose a new matching framework for the label shift problem that operates from the representation of class probability $p(y)$ to estimate the class probability ratio $w^*(y)$ in Section \ref{subsec::CPM}, which enables us to estimate class probability ratio in a more natural and effective way.

\section{Methodology} \label{sec::Methodology}

In this section, we present our CPMKM algorithm for label shift adaptation. More precisely, in Section \ref{subsec::CPM}, we introduce the new framework CPM, which starts from the class probability function $p(y)$ on the source domain for label shift adaptation.
Then, in Section \ref{sec::MultiLogReg}, we formulate the truncated kernel logistic regression for the estimation of the conditional probability $p(y|x)$.
Finally, in Section \ref{subsec::CPMKM}, we incorporate the estimates from Section \ref{sec::MultiLogReg} into the CPM framework in Section \ref{subsec::CPM}, resulting in our main algorithm CPMKM for label shift adaptation.

\subsection{Class Probability Matching for Label Shift Adaptation}
\label{subsec::CPM}

As discussed in Section \ref{subsec::FPM}, constructing a matching equation regarding class probability is a natural and effective method in addressing the label shift problem.
Therefore, we start with the representation of the class probability $p(y)$ since 
$p(y)$ can be easily and directly estimated from the data. 
To be specific, by the law of total probability, we have
\begin{align}\label{eq::pyrepre}
	p(y) 
	= \int_{\mathcal{X}} p(x) p(y|x)\, dx 
	= \int_{\mathcal{X}} \frac{p(x)}{q(x)} q(x) p(y|x)\, dx, 
	\qquad 
	y \in [M].
\end{align}
At first glance, the term $p(x)/q(x)$ on the right-hand side of equation \eqref{eq::pyrepre} appears to be defined in the feature space of dimension $d$. However, under the label shift assumption, this term $p(x)/q(x)$ can actually be expressed in terms of the class probability ratio $w^*(y)$ and conditional probability $p(y|x)$, i.e.~
\begin{align}\label{eq::repre}
	\frac{p(x)}{q(x)} 
	& = \frac{p(x)}{\sum_{y=1}^M w^*(y) p(y) p(x|y)}  
	\qquad \qquad \,
	(\text{representation of $q(x)$ in \eqref{eq::qXrepre}}) 
	\nonumber\\
	& = \frac{1}{\sum_{y=1}^M w^*(y) p(y|x)} 
	\qquad \qquad \qquad 
	(\text{Bayes formula}).
\end{align}
The above expression indicates that, for any $x\in \mathcal{X}$, the feature probability ratio $p(x)/q(x)$ is expressed as the reciprocal of a linear combination of $p(y|x)$, and the coefficient in front of $p(y|x)$ is precisely the class probability ratio $w^*(y)$. It is worth pointing out that $p(x)$ and $q(x)$ are two density functions defined on $d$-dimensional continuous feature spaces, while $w^*(y)$ and $p(y|x)$ are defined on a one-dimensional discrete label space. Therefore, \eqref{eq::repre} ingeniously represents the density ratio on $d$-dimensional features using probabilities on a one-dimensional label.
By substituting the expression in \eqref{eq::repre} into \eqref{eq::pyrepre}, we get
\begin{align*}
	p(y) 
	= \int_{\mathcal{X}} \frac{p(y|x)}{\sum_{m=1}^M w^*(m) p(m|x)}q(x) \, dx 
	= \mathbb{E}_{X\sim q} \frac{p(y|X)}{\sum_{m\in [M]} w^*(m) p(m|X)}, 
	\qquad 
	y \in [M].
\end{align*}
The above expression provides a new representation for the class probability $p(y)$. Building on this new representation, we introduce \textit{class probability matching} (CPM) for label shift adaptation. Specifically, we aim to find a weight vector $w = (w(m))_{m\in [M]}$ satisfying
\begin{align}\label{eq::equationsystem}
	p(y) 
	= \mathbb{E}_{X \sim q} \frac{p(y|X)}{\sum_{m=1}^M w(m) p(m|X)} 
	=: p_q^w(y),
	\qquad 
	y \in [M].
\end{align}
In contrast with FPM that matches two distributions on the $d$-dimensional feature $X$ in \eqref{eq::matchX}, CPM \eqref{eq::equationsystem} matches two distributions of the one-dimensional label variable $Y$. Since the label variable $Y$ only takes $M$ discrete values, CPM in \eqref{eq::equationsystem} only needs to match $M$ equations in the label space.  
This effectively avoids the potential issues associated with matching in the feature space mentioned in Section \ref{subsec::FPM}. 
Moreover, it is worth pointing out that CPM in \eqref{eq::equationsystem} is a natural way to deal with label shift problems. Specifically, from the formula of the joint probability ratio \eqref{eq::labelratio}, i.e., $q(x,y)/p(x,y)=w^*(y)$, the difference between the source and target distributions only stems from the difference in class probabilities. 
Therefore, directly constructing the matching equation on the representation of class probability $p(y)$ as in \eqref{eq::equationsystem} is a straightforward idea to obtain $w^*(y)$ for label shift adaptation.

\begin{figure*}[!h]
	\centering
	\captionsetup[subfigure]{justification=centering, captionskip=2pt} 
	\subfloat[][Feature Probability Matching] { 
		\includegraphics[width=0.45\textwidth, trim=0 80 0 100,clip]{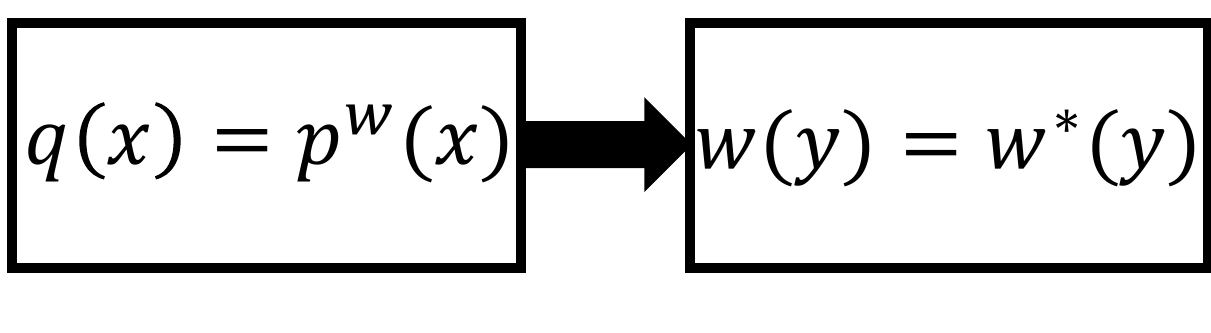} 
		\label{fig::FPM}
	} 
	\subfloat[][Class Probability Matching] { 
		\includegraphics[width=0.45\textwidth, trim=0 80 0 100,clip]{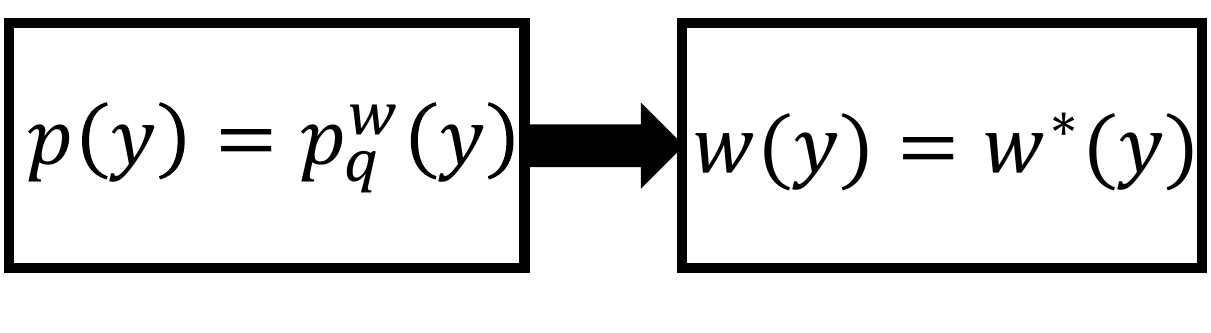} 
		\label{fig::CPM}
	} 
	\caption{Two distinct matching frameworks for label shift adaptation.}
\label{fig:FPMCPM}
\end{figure*}


Next, we discuss how to obtain a plug-in classifier for predicting the target domain using the class probability ratio $w^*(y)$.
Clearly, in order to obtain the plug-in classifier $h_q(x)$ in \eqref{OptimalClassifier} on the target domain, it suffices to estimate the conditional probability function $q(y|x)$. 
Notice that $q(y|x)$ can be represented as 
\begin{align}\label{eq::etaPshifttoQ}
	q(y|x) 
	& = \frac{q(y)q(x|y)}{\sum_{m=1}^M  q(m) q(x|m)} 
	\qquad \qquad \qquad 
	(\text{Bayes formula}) 
	\nonumber\\
	& = \frac{q(y) p(x|y)}{\sum_{m=1}^M q(m) p(x|m)}  
	\qquad \qquad \qquad 
	(\text{label shift assumption})
	\nonumber\\
	& = \frac{q(y)p(y|x)p(x)/p(y)}{\sum_{m=1}^M q(m) p(m|x)p(x)/p(m)}
	\quad \, \,  \,
	(\text{conditional probability formula}) 
	\nonumber\\
	& = \frac{(q(y)/p(y))p(y|x)}{\sum_{m=1}^M (q(m)/p(m)) p(m|x)}
	\qquad \, \, \, \
	(\text{simplifying the fraction})
	\nonumber\\ 
	& = \frac{w^*(y) p(y|x)}{\sum_{m=1}^M  w^*(m) p(m|x)} 
	\qquad \qquad \quad \, \,
	(\text{definition of $w^*(y)$}).
\end{align}
The above equation indicates that the conditional probability function $q(y|x)$ on distribution Q can be expressed in terms of the conditional probability function $p(y|x)$ on distribution $P$ and the class probability ratio $w^*(y)$. Therefore, to estimate $q(y|x)$, we only need to estimate $w^*(y)$ and $p(y|x)$, respectively.
In Section \ref{sec::MultiLogReg}, we employ the kernel logistic regression (KLR) to estimate $p(y|x)$, and subsequently, in Section \ref{subsec::CPMKM}, we provide estimates for $w^*(y)$ and $q(y|x)$ to obtain the final plug-in classifier.

\subsection{Kernel Logistic Regression for Conditional Probability Estimation}
\label{sec::MultiLogReg}

The logistic regression model arises from the desire to model the conditional probabilities of the $M$ classes via linear functions in $x$, while at the same time ensuring that they sum to one and remain in $[0,1]$ \cite{hastie2009elements}.
More precisely, logistic regression assumes a linear relationship between the input features and the log-odds of the target variable, which implies that the decision boundaries consist of parts of several linear hyperplanes.
However, this assumption can be restrictive when the true decision boundary is non-linear.

To deal with this issue, in this subsection, we investigate the kernel logistic regression (KLR) \cite{zhu2005kernel} for conditional probability estimation, which implicitly maps the input features into a higher-dimensional space, allowing it to model non-linear decision boundaries. 
To be specific, let $f:=(f_m)_{m\in [M]}$ be the collection of score functions from reproducing kernel Hilbert space (RKHS) $ H$ induced by Gaussian kernel function $k(x,x'):=\exp(-\|x-x'\|_2/\gamma^2)$ for $x, x' \in \mathbb{R}^d$, and some bandwidth parameter $\gamma$, denoted as 
\begin{align} \label{def::SpaceF}
	\mathcal{F}
	:= \{ f := (f_m)_{m=1}^M: f_m \in  H, m \in [M-1], f_M = 0 \}. 
\end{align}
Then KLR uses
\begin{align}\label{eq::pfmx}
	p_f(m|x) 
	:= \frac{\exp(f_m(x))}{\sum_{j=1}^M \exp(f_j(x))}, 
	\qquad 
	m \in [M],
\end{align}
to model the conditional probability function $p(m|x)$.
To prevent the cross-entropy loss of $p_f(\cdot |x)$, i.e.~$L_{\mathrm{CE}}(y, p_f(\cdot|x))=-\log p_f(y|x)$ from exploding, we need to truncate $p_f(\cdot|x))$ downwards. 
To be specific, given $t \in (0,1/(2M))$, we define $p_f^t(\cdot|x)$ as \
\begin{align}\label{eq::etakdelta}
	p_f^t(m|x)
	:= 
	\begin{cases}
		t, & \text{ if } p_f(m|x)<t,
		\\
		\displaystyle p_f(m|x) - (p_f(m|x) - t)\frac{\sum_{j: p_f(j|x)<t }(t-p_f(j|x))}{\sum_{\ell:p_f(\ell|x)\geq t }(p_f(\ell|x)-t)}, & \text{ if } p_f(m|x)\geq t.
	\end{cases}
\end{align}
The conditional probability function $p_f(m|x)$ for $m$ less than $t$ is truncated at $t$, while those $p_f(m|x)$ greater than $t$ are proportionally adjusted to ensure that $\sum_{m\in [M]} p_f^t(m|x)=1$.
It is easily shown that for all $m\in [M]$ and $x\in \mathcal{X}$, there holds $p_f^t(m|x) \geq t$ and therefore the value of the CE loss is bounded, i.e., $L_{\mathrm{CE}}(y, p_f^t(\cdot|x)) = -\log p_f^t(y|x) \leq -\log t$.

Now, given a regularization parameter $\lambda>0$, the kernel logistic regression estimator $f_{D_p}$ and the optimal bandwidth parameter $\gamma^*$ are obtained through
\begin{align}\label{eq::KLR}
	(f_{D_p}, \gamma^*) 
	:= \argmin_{f\in \mathcal{F}, \gamma>0} 
	\lambda \sum_{j=1}^{M-1} \|f_j\|_H^2 + \frac{1}{n_p} \sum_{i=1}^{n_p} L_{\mathrm{CE}}(y_i, p_f^t(\cdot|x_i)),
\end{align}
where $\|f_j\|_{H}$ denotes the RKHS norm of $f_j$. 
Then the truncated KLR is given by 
\begin{align}\label{eq::hatetaP}
	\widehat{p}(m|x) := p_{f_{D_p}}^t(m|x), \quad m\in [M].
\end{align}

It is worth pointing out that \eqref{eq::etakdelta} provides a new approach to deal with the issues of unboundedness of the CE loss. In the existing literature, \cite{bos2022convergence} employs the truncated CE loss, which directly truncates the CE loss to a specific threshold value if the CE loss exceeds this threshold. However, when the truncated CE loss is small, its true CE loss may be extremely large. In contrast, by operating the truncation on $p_f(m|x)$, the CE loss of the truncated conditional probability $p_f^t(\cdot|x)$ is always upper bounded, which leads to a good and stable estimation of $p(m|x)$.

\subsection{Class Probability Matching Using Kernel Methods for Label Shift Adaptation} \label{subsec::CPMKM}

In this section, based on our proposed framework named \textit{class probability matching} in \eqref{eq::equationsystem} from Section \ref{subsec::CPM}, we introduce an algorithm named \textit{class probability matching using kernel methods} (\textit{CPMKM}) for label shift adaptation. 
The proposed algorithm consists of two parts: the initial phase is to estimate the class probability ratio $w^*$, while the subsequent step is to obtain the corresponding classifier for the target domain.

\paragraph{Estimating the Class Probability Ratio $\boldsymbol{w}^*$.}
On the one hand, based on the source domain data $D_p$, the left-hand side of \eqref{eq::equationsystem}
can be estimated by 
\begin{align}\label{eq::directCounting}
	\widehat{p}(y)
	:= \frac{1}{n_p} \sum_{i=1}^{n_p} \eins \{ Y_i = y \},
	\qquad 
	y \in [M],
\end{align}
where $\eins\{ Y_i = y \}$ denotes the indicator function which takes $1$ if $Y_i = y$ and otherwise is $0$.

On the other hand, with the aid of target domain samples $D_q^u$, the right-hand side of \eqref{eq::equationsystem} can be approximated by 
\begin{align*}
	\frac{1}{n_q} \sum_{X_i \in D_q^u} 
	\frac{p(y|X_i)}{\sum_{m=1}^M w(m) p(m|X_i)}. 
\end{align*}
In order to estimate the class membership probabilities $p(k|X_i)$ of target domain data $X_i \in D_q^u$ in the source domain, we first fit a truncated KLR with source data $D_p$ to get $\widehat{p}(y|x)$ in \eqref{eq::hatetaP} as an estimator of the predictor $p(y|x)$.
Then by using the estimator $\widehat{p}(y|x)$ to predict the class membership probabilities $\widehat{p}(y|X_i)$ for all target domain samples $X_i \in D_q^u$, $p_q^w(y)$ in \eqref{eq::equationsystem} can be estimated by 
\begin{align}\label{eq::hatpihatQw}
	\widehat{p}_q^w(y)
	:= \frac{1}{n_q} \sum_{X_i \in  D_q^u} \frac{\widehat{p}(y|X_i)}{\sum_{m=1}^M w(m) \widehat{p}(m|X_i)},
	\qquad y \in [M].
\end{align} 
Since the probability ratio $w^* = (w^*(y))_{y\in [M]}$ is a solution of matching on $Y$ in \eqref{eq::equationsystem}, $w^*$ can be estimated by matching 
the estimates $\widehat{p}(y)$ in \eqref{eq::directCounting} and 
$\widehat{p}_q^w(y)$ in \eqref{eq::hatpihatQw}.
In order to match $\widehat{p}(y)$ and $\widehat{p}_q^w(y)$, we have to 
find the solution $\widehat{w}:=(\widehat{w}(y))_{y\in[M]}$ to the following minimization problem
\begin{align}\label{eq::whatminimizer}
	\widehat{w} := \argmin_{w\in \mathbb{R}^M, w\geq 0} \sum_{y=1}^M \bigl| \widehat{p}(y) -  \widehat{p}_q^w(y) \bigr|^2,
\end{align}  
which can be solved by the Limited-memory Broyden-Fletcher-Goldfarb-Shanno with Box constraints (L-BFGS-B) algorithm \cite{liu1989limited}.

\paragraph{The Classifier for the Target Domain.}
By using the representation of $q(y|x)$ in \eqref{eq::etaPshifttoQ},
the KLR estimator $\widehat{p}(y|x)$ in \eqref{eq::hatetaP},
and the estimation of the class probability ratio $\widehat{w}$ in \eqref{eq::whatminimizer}, 
the conditional probability in the target domain can be estimated by 
\begin{align}\label{eq::hatetaQ}
	\widehat{q}(y|x) 
	= \frac{\widehat{w}_y \widehat{p}(y|x)}{\sum_{m=1}^M \widehat{w}(m) \widehat{p}(m|x)},
\end{align}
which leads to the plug-in classifier
$\widehat{h}_q(x)$ as in \eqref{eq::fhat}.

The above procedures for label shift adaptation can be summarized in Algorithm \ref{alg::estimatebeta}.

\begin{algorithm}[h]
	\caption{Class Probability Matching Using Kernel Methods (CPMKM)}
	\label{alg::estimatebeta}
	\begin{algorithmic}
		\renewcommand{\algorithmicrequire}{\textbf{Input:}}
		\renewcommand{\algorithmicensure}{\textbf{Output:}}
		\REQUIRE Source domain samples $D_p := (X_i, Y_i)_{i=1}^{n_p}$;
		\\
		\quad \,\,\,\,\,\,\,\,
		Target domain samples $D_q^u :=(X_i)_{i=n_p+1}^{n_p+n_q}$;\\
		Compute the class probability estimation $\widehat{p}(y)$ of the source domain in \eqref{eq::directCounting};\\
		Fit the KLR estimator $\widehat{p}(y|x)$ with the source domain samples $D_p$;\\
		Compute $\widehat{p}(y|X_i)$ on target domain samples $X_i \in D_q^u$;\\ 
		Compute $\widehat{p}_q^w(y)$ in \eqref{eq::hatpihatQw} for $y\in [M]$;
		\\
		Obtain $\widehat{w}$ by solving the minimization problem in \eqref{eq::whatminimizer};\\
		Compute the conditional probability in the target domain $\widehat{q}(y|x)$ in \eqref{eq::hatetaQ};\\
		Obtain the plug-in classifier $\widehat{h}_{q}$ in \eqref{eq::fhat}.
		\ENSURE Predicted labels $\{ \widehat{h}_{q}(X_i) \}_{i=n_p+1}^{n_p+n_q}$.
	\end{algorithmic}
\end{algorithm}

\section{Theoretical Results} \label{sec::TheoreticalResults}

In this section, we establish the convergence rates of the CPMKM and comparing methods for discussion.
In Section \ref{subsec::RatesSource}, in order to obtain the convergence rates for the CPMKM predictor $\widehat{q}(y|x)$ in the target domain, we begin by establishing the convergence rate of the truncated KLR predictor $\widehat{p}(y|x)$ in the source domain.
Building upon the above results, we establish the convergence rate of the CPMKM predictor under mild assumptions in Section \ref{subsec::RatesTarget}. In addition, we derive the lower bound of the label shift problem, which matches the rates achieved by CPMKM, thereby demonstrating its minimax optimality. In Section \ref{subsec::Comparison}, we make some comments and discussions to show our distinction from the existing work.

\subsection{Convergence Rates of KLR in the Source Domain}\label{subsec::RatesSource}

Before we proceed, we need to introduce the following restrictions on the distribution $P$ to characterize which properties of a distribution most influence the performance of KLR.

\begin{assumption}\label{ass::predictor}
	We make the following assumptions on probability distributions $P$. 
	\begin{enumerate}
		\item[\textbf{(i)}] \textbf{[H\"{o}lder Smoothness]} 
		Assume that for any $x, x' \in \mathcal{X}$, there exists a H\"older constant $c_{\alpha} \in (0,\infty)$ and $\alpha\in [0,1]$ such that $|p(m|x') - p(m|x)| \leq c_{\alpha} \|x' - x\|_2^{\alpha}$ for all $m \in [M]$.
		\item[\textbf{(ii)}] \textbf{[Small Value Bound]}
		Assume that for all $t \in (0, 1]$, there exists a constant $c_{\beta} > 0$ such that $P_X(p(m|X) \leq t) \leq c_{\beta} t^{\beta}$ for all $m \in [M]$.
	\end{enumerate}
\end{assumption}

The smoothness assumption \textit{(i)} on the conditional probability function $p(m|x)$ is a common assumption adopted for classification \cite{chaudhuri2014rates, doring2017rate, xue2018achieving, khim2020multiclass}. 
In fact, previous work \cite{kpotufe2021marginal, maity2022minimax, cai2021transfer} adopted it to study the transfer learning or domain adaptation for nonparametric classification under covariate shift, label shift, and posterior shift respectively.
By \textit{(i)} we see that when $\alpha$ is small, the conditional probability function $p(m|x)$ fluctuates more sharply, which results in the difficulty of estimating $p(m|x)$ accurately and thus leads to a slower convergence rates.
The small value bound assumption \ref{ass::predictor} \textit{(ii)}, taken from \cite{bos2022convergence}, quantifies the size of the set in which the conditional probabilities $p(m|x)$ are small. As the conditional probability $p(m|x)$ approaches zero, the value of $-\log p(m|x)$ grows towards infinity at an increasing rate. Therefore, the accuracy of estimating small conditional probabilities has a crucial impact on the value of the CE loss. Hence, the multi-class classification problem w.r.t.~the CE loss exhibits faster convergence rates when the probability of the region with small conditional probabilities is low (i.e. when $\beta$ is large).

\begin{theorem}\label{thm::convergencerate}
	Let Assumption \ref{ass::predictor} hold and KLR for conditional density estimation $\widehat{p}(y|x)$ as in \eqref{eq::hatetaP}. If we choose
	$\lambda \asymp n_p^{-1}$,
	$\gamma \asymp n_p^{-1/((1+\beta\wedge 1)\alpha+d)}$, and
	$t \asymp n_p^{-\alpha/((1+\beta\wedge 1)\alpha+d)}$,
	then there exists some $N \in \mathbb{N}$ such that for any $n_p \geq N$ and for any $\xi \in (0,1/2)$, there holds
	\begin{align*}
		\mathcal{R}_{L_{\mathrm{CE}},P}(\widehat{p}(y|x))-\mathcal{R}_{L_{\mathrm{CE}},P}^*
		\lesssim n_p^{-\frac{(1+\beta\wedge 1)\alpha}{(1+\beta\wedge 1) \alpha + d}+\xi}
	\end{align*}
	with probability $P^{n_p}$ at least $1-1/n_p$.
\end{theorem}

Notice that the CE loss measures the accuracy of the conditional probability estimator $\widehat{p}(y|x)$, while the classification loss measures how well we classify the samples. 
Since our CPMKM uses the conditional probability estimation $\widehat{p}(m|x)$ for estimating the class probability ratio and building the classifier in the target domain, 
Theorem \ref{thm::convergencerate} establishes the convergence rates of truncated KLR w.r.t.~the CE loss instead of the classification loss. 
It is worth pointing out that the theoretical results of the excess risk w.r.t.~the CE loss of $\widehat{p}(m|x)$ supply the key to analyzing the prediction error of CPMKM in the target domain.

The following theorem presents the lower bound result of multi-class classification under Assumption \ref{ass::predictor}.

\begin{theorem}\label{thm::lowerKLR}
	Let $\mathcal{F}^{D_p}$ be the set of all measurable predictors $f: \mathcal{X} \to \Delta^{M-1}$ and $\mathcal{P}$ be a collection of all distribution $P$ which satisfies Assumptions \ref{ass::predictor}. In addition, let a learning algorithm that accepts data $D_p$ and outputs a predictor be denoted as $\mathcal{A}: (\mathcal{X} \times \mathcal{Y})^{n_p} \to \mathcal{F}^{D_p}$. Then we have
	\begin{align*}
		\inf_{\mathcal{A}: (\mathcal{X} \times \mathcal{Y})^{n_p} \to \mathcal{F}^{D_p}} \sup_{P \in \mathcal{P}}  \mathcal{R}_{L_{\mathrm{CE}}, P}(\mathcal{A}(D_p)) - \mathcal{R}_{L_{\mathrm{CE}}, P}^*
		\gtrsim n_p^{-\frac{(1+\beta\wedge 1)\alpha}{(1+\beta\wedge 1)\alpha+d}}
	\end{align*}
	with probability $P^{n_p}$ at least $(3 - 2 \sqrt{2}) / 8$.
\end{theorem}

Theorem \ref{thm::lowerKLR} together with Theorem \ref{thm::convergencerate} illustrates that the convergence rates of KLR shown in Theorem \ref{thm::convergencerate} is minimax optimal up to an arbitrary small order $\xi$.

\subsection{Convergence Rates of CPMKM for Predicting in the Target Domain} \label{subsec::RatesTarget}

In this section, we establish the convergence rates of CPMKM under the label shift setting (Assumption \ref{ass:labelshift}) and some regular assumptions. In addition to Assumption \ref{ass::predictor} concerning the conditional probability $p(y|x)$, we also need some commonly used assumptions on the marginal distributions of $P$ and $Q$ for establishing the convergence rates of CPMKM.

\begin{assumption}\label{ass::marginal}
	We make the following assumptions on the marginal distributions of the source distribution $P$ and target distribution $Q$. 
	\begin{enumerate}
		\item[\textbf{(i)}] \textbf{[Non-zero Class Probability]}
		Assume that the class probabilities $p(m) > 0$ holds for all $m \in [M]$.
		\item[\textbf{(ii)}] \textbf{[Strong Density Assumption]}
		Assume that for any $x$ with $q(x) > 0$, there exists an $r_0 >0$ and $c_{-} > 0$ such that for any $0 \leq r < r_0$, there holds $P(B(x,r)) \geq c_{-} \mu(B(x,r))$.
		\item[\textbf{(iii)}] \textbf{[Marginal Ratio Assumption]} Assume that for any $x \in \mathcal{X}$ with $q(x) > 0$, there exist some constant $\underline{c} > 0$ such that $q(x)/p(x) \geq \underline{c}$.
	\end{enumerate}
\end{assumption}

Notice that \textit{(i)} only assumes that $p(y) > 0$, $\forall y\in[M]$, but it does not necessarily require that $q(y) > 0$, $\forall y\in[M]$, which turns out to be more realistic, see also \cite{zhang2013domain,lipton2018detecting}.
Assumption \ref{ass::marginal} \textit{(ii)} is a weaker version of the commonly used strong density assumption \cite{audibert2007fast} which assumes that $p(x)$ is lower bounded for all $x \in \supp(P_X)$.  
Assumption \ref{ass::marginal}  \textit{(iii)} assumes that the marginal density ratio $q(x)/p(x)$ is bounded from below for all $x$ with $q(x) > 0$. This assumption can be derived by Definition 4 in \cite{maity2022minimax}, which assumes that $q(x)$ is lower bounded for all $x$ satisfying $q(x) > 0$ and $p(x)$ is upper bounded for any $x$.

\begin{assumption}[\textbf{Linear Independence}] \label{ass::LinearIndependence} 
	We assume that the class-conditional probability density functions $\{q(\cdot|y): y\in [M]\}$ are linearly independent.  
\end{assumption}

In other words, if the equation $\sum_{y=1}^M \alpha_y q(x|y) = 0$ 
holds for all $x\in \mathcal{X}$ with $\alpha_y \in \mathbb{R}$, $y\in [M]$, then we have $\alpha_y = 0$ for all $y\in [M]$.
In fact, Assumption \ref{ass::LinearIndependence} is a standard and widely-used assumption in the label shift adaptation problem, e.g. \cite{zhang2013domain, iyer2014maximum}.
Assumption \ref{ass::LinearIndependence} guarantees the identifiability of the class probability $q(y)$ if $q(x)$ and $q(x|y)$ are known. 
To be specific, the equation $\sum_{y=1}^M \theta_y q(x|y) = q(x)$ holds for all $x\in \mathcal{X}$ with $\theta := (\theta_y)_{y\in [M]} \in \Delta^{M-1}$ if and only if $\theta_y = q(y)$, $y\in [M]$. Further discussion can be found in Section 2.1 of \cite{gong2016domain, iyer2014maximum}.

In addition to the above assumptions, we also need the following regularity assumption, which is taken from Condition 1 in \cite{garg2020unified}.

\begin{assumption}[Regularity]
	\label{ass::regularity}
	Assume that for sufficient large sample size $n_p$ and $n_q$, for any $x \in \mathrm{supp}(Q_X)$, i.e., for any $x$ satisfying $q(x) > 0$, there exists some universal constant $c_R > 0$ such that
	\begin{align*}
		\sum_{m=1}^M \widehat{w}(m) \widehat{p}(m|x) \geq c_R, 
		\qquad
		\sum_{m=1}^M w^*(m) \widehat{p}(m|x) \geq c_R.
	\end{align*}
\end{assumption}

By \eqref{eq::repre}, Assumption \ref{ass:labelshift} and \ref{ass::marginal}  \textit{(iii)}, we obtain $\sum_{m=1}^M w^*(m) p(m|x) = q(x)/p(x) \geq \underline{c}$.
If $\widehat{p}(m|x)$ is a good estimate of $p(m|x)$, then $\widehat{w}$ in \eqref{eq::whatminimizer} and the true class probability ratio $w^*$ are close.
Consequently, $\sum_{m=1}^M \widehat{w}(m) \widehat{p}(m|x)$ and $\sum_{m=1}^M w^*(m) \widehat{p}(m|x)$ can also be lower-bounded by a constant.
Further discussion about the justification of Assumption \ref{ass::regularity} can be found in \cite{garg2020unified}.

In what follows, we present the convergence rate of the conditional probability estimator in the target domain $\widehat{q}(m|x)$ in \eqref{eq::hatetaQ}.

\begin{theorem}\label{thm::rateQ}
	Let Assumptions \ref{ass:labelshift}, \ref{ass::predictor}, \ref{ass::marginal}, \ref{ass::LinearIndependence}, and \ref{ass::regularity} hold. Moreover, let $\widehat{q}(y|x)$ be the estimator as in \eqref{eq::hatetaQ}. Then there exists an $N' \in \mathbb{N}$ such that for any $n_p \wedge n_q \geq N'$ and for any $\xi > 0$, there holds
	\begin{align}\label{eq::rateCPMKM}
		\mathcal{R}_{L_{\mathrm{CE}},Q}(\widehat{q}(y|x)) -\mathcal{R}_{L_{\mathrm{CE}},Q}^* 
		\lesssim n_p^{-\frac{(1+\beta\wedge 1)\alpha}{(1+\beta\wedge 1)\alpha+d}+\xi} + \log n_q / n_q
	\end{align}
	with probability $P^{n_p} \otimes Q_X^{n_q}$ at least $1-1/n_p-1/n_q$.
\end{theorem}

Theorem \ref{thm::rateQ} shows that 
up to the arbitrarily small constant $\xi$, we can see from \eqref{eq::rateCPMKM} that the convergence rate of CPMKM depends on the larger term of $n_p^{-(1+\beta\wedge 1)\alpha/((1+\beta\wedge 1)\alpha+d)}$ and $n_q^{-1}$. 
In practical applications, we usually have fixed large sample size $n_p$ in the source domain and gradually increased sample size $n_q$ over time in the target domain.
In this case, as the sample size $n_q$ increases from zero to the order of $n_p^{\alpha(1+\beta\wedge 1)/(\alpha(1+\beta\wedge 1)+d)}$, the convergence rate \eqref{eq::rateCPMKM} becomes faster and finally reaches the order of $n_p^{-\alpha(1+\beta\wedge 1)/(\alpha(1+\beta\wedge 1)+d)}$. 
However, even if the order of $n_q$ continues to increase, the order of the convergence rate remains the same. This implies that a certain amount of target domain sample is enough for CPMKM, and beyond a certain threshold, more unlabeled samples from the target domain can no longer improve the performance.

In the following, we establish the lower bound on the convergence rate of the excess risk in the label shift problem for an arbitrary learning algorithm with access to $n_p$ labeled source data and $n_q$ unlabeled target data.

\begin{theorem}\label{thm::lower}
	Let $\mathcal{F}^D$ be the set of all measurable predictors $f: \mathcal{X} \to \Delta^{M-1}$ built upon $D$ and $\mathcal{T}$ be a collection of pairs of distribution $(P, Q)$ which satisfies Assumptions \ref{ass:labelshift} and \ref{ass::predictor}. In addition, let a learning algorithm that accepts data $D:=(D_p, D_q^u)$ and outputs a predictor $f$ be denoted as $\mathcal{A}: (\mathcal{X} \times \mathcal{Y})^{n_p}\otimes \mathcal{X}^{n_q} \to \mathcal{F}^D$. Then we have
	\begin{align*}
		\inf_{\mathcal{A}: (\mathcal{X} \times \mathcal{Y})^{n_p} \otimes \mathcal{X}^{n_q} \to \mathcal{F}^D} \sup_{(P,Q) \in \mathcal{T}}  \mathcal{R}_{L_{\mathrm{CE}}, Q}(\mathcal{A}(D)) - \mathcal{R}_{L_{\mathrm{CE}}, Q}^*
		\gtrsim n_p^{-\frac{(1+\beta\wedge 1)\alpha}{(1+\beta\wedge 1)\alpha+d}} + n_q^{-1}
	\end{align*}
	with probability $P^{n_p} \otimes Q_X^{n_q}$ at least $(3 - 2 \sqrt{2}) / 8$.
\end{theorem}

Theorems \ref{thm::rateQ} and \ref{thm::lower} show that our CPMKM is able to achieve the minimax optimal rates for the label shift adaptation problem.

\subsection{Comments and Discussions} \label{subsec::Comparison}

\subsubsection{Comments on Convergence Rates of KLR for Conditional Probability Estimation}

Note that the logarithmic function is unbounded. Therefore, if the KLR estimator $p_{f_D}(\cdot|x)$ is close to zero, the CE risk can become arbitrarily large. 
To address this issue, \cite{bos2022convergence} propose the truncated CE risk. Specifically, given a pre-specified threshold $B>0$, the difference between the CE loss of KLR estimator $p_{f_D}(\cdot|x)$ and that of the true probability $p(\cdot|x)$, i.e., $L_{\mathrm{CE}}(y, p_{f_D}(\cdot|x)) - L_{\mathrm{CE}}(y, p(\cdot|x)) = \log (p(y|x)/p_{f_D}(y|x))$ is truncated by $B$. 
Then the truncated excess risk of KLR estimator $p_{f_D}(\cdot|x)$ is 
\begin{align}\label{eq::trunrisk}
	\mathbb{E}_{(X,Y)\sim P}\big(\log (p(Y|X)/p_{f_D}(Y|X)) \wedge B\big),
\end{align}
However, even if the truncated CE risk of the KLR estimator is small, its true CE risk may be extremely large
since $\log (p(y|x)/p_{f_D}(y|x))$ may be significantly larger than $B$.

The truncated CE risk can be decomposed into the sample error and approximation error, originating from the randomness of the data and the approximation capability of the function space.
In the error analysis of \cite{bos2022convergence}, the sample error bound is shown to increase as the threshold $B$ increases, while the approximation error is irrelevant to the threshold $B$ when $B$ is larger than some constant. Therefore, the upper bound of the truncated CE risk grows linearly with the threshold $B$ as in Theorem 3.3 of \cite{bos2022convergence}. As a result, the optimal convergence rate can be obtained with $B \asymp \log n_p$. 
However, if we take $B = \infty$ to convert the truncated CE risk to the true CE risk, the risk bound becomes infinity. In other words, no convergence rates of the conditional probability estimator w.r.t.~ the CE loss can be obtained.

In contrast to truncating the CE risk, we truncate the KLR estimator $p_{f_D}(y|x)$ in \eqref{eq::pfmx} downwards such that its truncated estimator $p^t_{f_D}(y|x)$ in \eqref{eq::etakdelta} is larger than a pre-defined threshold $t$. Therefore, the CE loss of $p^t_{f_D}(y|x)$ is upper bounded by $-\log t$, which enables us to directly analyze the excess risk of $p^t_{f_D}(y|x)$ w.r.t.~the CE loss, i.e.  
\begin{align*}
	\mathbb{E}_{(X,Y)\sim P}\big(\log (p(Y|X)/p^t_{f_D}(Y|X))\big).
\end{align*}

Then, we illustrate how to establish the optimal rates of the CE risk for our truncated estimator $p^t_{f_D}(y|x)$ in \eqref{eq::etakdelta}.
In our theoretical analysis, the effect of the truncation threshold $t$ on the excess risk of $p^t_{f_D}(y|x)$ is two-fold. 
On the one hand, as shown in Theorem \ref{thm::oracle} in Section \ref{sec::sampleerror}, the sample error bound of $p^t_{f_D}(y|x)$ decreases as the threshold $t$ increases.
This is because a larger $t$ leads to a smaller upper bound for the CE loss of the truncated estimator, which yields a smaller sample error bound. 
On the other hand, a sufficiently small $t$ enables the truncated estimator $p_{f_D}^t(y|x)$ to effectively approximate the conditional probability $p(y|x)$ that is close to zero, which leads to a small approximation error that only depends on the kernel bandwidth.
According to the trade-off between the sample error and approximation error, we are able to establish the optimal convergence rate of the CE risk for the truncated KLR estimator by choosing an appropriate threshold $t \asymp n_p^{-\alpha/((1+\beta\wedge 1)\alpha+d)}$ as in Theorem \ref{thm::convergencerate}. 
It is worth noting that this is the first convergence rate of KLR w.r.t.~the unbounded CE loss to the best of our knowledge, 
and therefore our result is stronger than that of the truncated CE loss established in the previous work \cite{bos2022convergence}.

\subsubsection{Comments on Convergence Rates for Label Shift Adaptation}

For the label shift adaptation problem in the context of binary classification, \cite{maity2022minimax} shows that some existing algorithms achieve the optimal convergence rates with respect to the misclassification ($0-1$) loss. To this end, they adopt the commonly-used margin condition \cite{tsybakov2004optimal,audibert2007fast} focusing on the region near the boundary, where the conditional probability $p(y|x)$ is close to $1/2$.

In this paper, we analyze the conditional probability estimator w.r.t.~the CE loss rather than the classifier w.r.t.~the misclassification loss. 
The reason for choosing to study the conditional probability estimator lies in its ability to provide us with a notion of confidence compared to classifiers that only predict labels of the test data \cite{bos2022convergence}. 
In fact, if the largest conditional class probability is close to one, then the class with the largest conditional probability is likely to be the true label. 
On the other hand, if the largest conditional class probabilities are close to each other, the prediction results of the classifier are not reliable.
In Theorem \ref{thm::rateQ}, optimal convergence rates of the conditional probability estimator $\widehat{q}(y|x)$ in the target domain are established w.r.t~the CE loss for label shift adaptation.
Therefore, our theoretical results are fundamentally different from that of \cite{maity2022minimax}.

Finally, it is worth noting that the convergence mentioned here is of type “with high probability”, which is stronger than the results of type “in expectation” in \cite{maity2022minimax}. This is because we use techniques from the approximation theory \cite{cucker2007learning} and arguments from the empirical process theory \cite{vandervaart1996weak,Kosorok2008introduction}.

\subsubsection{Comparison with FPM methods} \label{subsec::ComparisonX}

As discussed in Section \ref{subsec::FPM}, \textit{kernel mean matching} (\textit{KMM}) \cite{huang2006correcting}  estimates the class probability ratio for label shift by minimizing the distance between the kernel mean of reweighted source w3 and target data. 
Combining our analysis with the results of \cite{iyer2014maximum}, we are able to establish the same convergence rates of the KMM method as that of our CPMKM under the label shift assumption. 
However, the primary limitation of the KMM method is its computational inefficiency, particularly when applied to large-scale datasets. To be specific, in order to compute the distance of two feature probability $q(x)$ and $p^w(x)$ in \eqref{eq::matchX} in terms of their kernel embedding means, KMM requires the calculation of the inversion of Gram matrix of the source domain data, whose complexity is the order of the cube of $n_p$ by \cite{lipton2018detecting}. This thereby constrains the scalability of the method. 
To reduce the computational cost, \cite{lipton2018detecting, tian23a} introduced a mapping function $h$ to transform the feature variable $X$ into a low-dimensional variable $h(X)$ 
and match the probability of the transformed feature $q(h(x))$ with the weighted class-conditional probability $p(h(x)|y)$ to obtain the class probability ratio.
By combining our analysis with the results of \cite{garg2020unified}, 
we are also able to establish the same convergence rates for their method as ours under the additional assumption that $p(h(x)|y)$, $y\in[M]$ are linearly independent,
This is a stronger assumption than Assumption \ref{ass::LinearIndependence} which only requires the linear independence of $p(x|y)$, $y\in[M]$. 
Moreover, it is a challenging problem to determine the specific form of $h$ in \cite{lipton2018detecting, tian23a}.

In contrast, our CPMKM matches two class probabilities $p(y)$ and $p_q^w(y)$ defined on the discrete one-dimension label space $\mathcal{Y}$ as in \eqref{eq::equationsystem}, instead of matching two feature probabilities $q(x)$ and $p^w(x)$ as in \eqref{eq::matchX}.
As a result, it is not required to compute the inverse of the Gram matrix for kernel mean matching or apply additional mapping to reduce the dimensionality.
In fact, our CPMKM can estimate the class probability ratio by directly constructing the estimators $\widehat{p}(y)$ and $\widehat{p}_q^w(y)$ and minimizing their $L_2$-distance.

\section{Error Analysis}\label{sec::ErrorAnalysis}

In this section, we begin by conducting an error analysis on the kernel logistic regression, as discussed in Section \ref{sec::analysisKLR}. Specifically, we present the upper bounds for both the approximation error and sample error in subsections \ref{sec::approxerror} and \ref{sec::sampleerror}, respectively. 
Furthermore, we delve into the error analysis for the class probability ratio estimation and excess risk of the CPMKM in the target domain, which can be found in Section \ref{sec::priorestmation} and Section \ref{sec::analysisQresult}, respectively.

\subsection{Error Analysis for CPMKM in the Source Domain} \label{sec::analysisKLR}

\subsubsection{Bounding the Approximation Error} \label{sec::approxerror}

\begin{proposition}\label{prop::approx}
	Let Assumptions \ref{ass:labelshift} and \ref{ass::predictor} hold with the H\"older exponent $\alpha$, the bandwidth of the Gaussian kernel $\gamma \in (0,(2M)^{-1/\alpha})$ and the truncation threshold $t \leq  \gamma^{\alpha}$. Furthermore, let the function space $\mathcal{F}$ be defined as in \eqref{def::SpaceF}. Then there exists an $f_0 \in \mathcal{F}$ such that 
	\begin{align}\label{eq::approxt}
		\mathcal{R}_{L_{\mathrm{CE}},P}(p^t_{f_0}(y|x)) - \mathcal{R}_{L_{\mathrm{CE}},P}^* 
		\lesssim \gamma^{\alpha(1+\beta\wedge 1)}.
	\end{align}
\end{proposition}

Proposition \ref{prop::approx} shows that when $t\leq \gamma^{\alpha}$, there exists a score function $f_0$ such that the approximation error of $p^t_{f_0}(y|x)$ is bounded by $\gamma^{\alpha(1+\beta\wedge 1)}$, which is independent of $t$. 
This shows that a sufficiently small $t$ not only makes the CE loss of $p_{f_0}^t(y|x)$ bounded by $-\log t$, but also ensures the upper bound of the approximation error of $p_{f_0}^t(y |x)$.

\subsubsection{Bounding the Sample Error} \label{sec::sampleerror}

The existing oracle inequalities require either the supremum bound of the loss function, see e.g., Theorem 7.16 in \cite{steinwart2008support}, or the boundedness of the absolute difference of the loss between the estimator and the Bayes function, see e.g., Theorem 7.2 in \cite{steinwart2008support} and Theorem 3.5 in \cite{bos2022convergence}. However, the CE loss is an unbounded loss function that does not satisfy the above two boundedness conditions. 
To cope with the unboundedness of the CE loss, we investigate the truncated conditional probability estimator $p_f^t(y|x)$ as in \eqref{eq::etakdelta}, which is always larger than the threshold $t$ and thus the CE loss of  $p_f^t(y|x)$ is bounded by $-\log t$ for any $f\in \mathcal{F}$. However, since the true probability $p(y|x)$ can be arbitrarily close to zero, its CE loss can be extremely large and  violates the boundedness condition, making the existing oracle inequalities inapplicable.

Therefore, in this paper, to analyze the excess CE risk of the truncated KLR, we decompose the unbounded CE loss into an upper part and a lower part depending on whether the true conditional probability is greater or less than a certain value $\delta \in (0,1)$. More precisely, these two parts of the CE loss are respectively defined as
\begin{align}
	(L^u_{\mathrm{CE}} \circ p_f^t)(x, y) 
	& := L^u_{\mathrm{CE}}(y, p_f^t(\cdot|x)) := \eins\{p(y|x) \geq \delta\}(-\log p^t_f(y|x)),
	\label{eq::Llarge}
	\\
	(L^l_{\mathrm{CE}} \circ p_f^t)(x, y) 
	& := L^l_{\mathrm{CE}}(y, p_f^t(\cdot|x)) 
	:= \eins\{p(y|x) < \delta\}(-\log p^t_f(y|x)).
	\label{eq::Lsmall}
\end{align}
Since the upper part of the CE loss of $p(y|x)$ is bounded by $-\log \delta$, the supremum bound of $L^u_{\mathrm{CE}}$ in \eqref{eq::Llarge} is finite and thus the excess CE risk on this part can be analyzed by applying the concentration inequality to the loss difference $L^u_{\mathrm{CE}} \circ p_f^t - L^u_{\mathrm{CE}} \circ p$ for any $f\in \mathcal{F}$. On the other hand, although the lower part of the CE loss difference $L^l_{\mathrm{CE}} \circ p_f^t - L^l_{\mathrm{CE}} \circ p$
is unbounded due to the unbounded term $L^l_{\mathrm{CE}} \circ p$, we can apply the concentration inequality to the loss of the truncated estimator $L^l_{\mathrm{CE}}\circ p^t_f$ rather than to the loss difference $L^l_{\mathrm{CE}}\circ p^t_f - L^l_{\mathrm{CE}}\circ p$ for analysis on this part. 
As a result, we manage to establish a new oracle inequality with a finite sample error bound for the truncated KLR w.r.t.~the CE loss as presented in the following theorem.

\begin{theorem}\label{thm::oracle}
	Let $f_{D_p}$ be defined as in \eqref{eq::KLR}, the truncated KLR estimator $\widehat{p}(y|x)$ be defined as in \eqref{eq::hatetaP}, and  
	$p^t_{f}(y|x)$ be defined as in \eqref{eq::etakdelta} with the truncation threshold $t < 1/(2M)$. Furthermore, let $\mathcal{F}$ be defined as in \eqref{def::SpaceF}.
	For any $f_0 \in \mathcal{F}$, $\xi \in (0,1/2)$ and $\zeta>0$, there holds  
	\begin{align}\label{eq::orcleKLR}
		&\lambda \|f_{D_p}\|^2_{H} + \mathcal{R}_{L_{\mathrm{CE}}, P}(\widehat{p}(y|x))-\mathcal{R}_{L_{\mathrm{CE}}, P}^* 
		\nonumber\\
		&\lesssim (\lambda\|f_0\|^2_{H} + \mathcal{R}_{L_{\mathrm{CE}}, P}(p^t_{f_0}(y|x)) - \mathcal{R}_{L_{\mathrm{CE}}, P}^*) + (-\log t)\cdot (t^2  + \lambda^{-\xi} \gamma^{-d}  n_p^{-1} + \zeta/n_p)
	\end{align}
	with probability at least $1-4e^{-\zeta}$. 
\end{theorem}

Theorem \ref{thm::oracle} shows that the excess CE risk of $\widehat{p}(y|x)$ is bounded by the sum of the approximation error and sample error, which correspond to the two terms on the right-hand side of \eqref{eq::orcleKLR}.
Since the approximation error bound is presented by Proposition \ref{prop::approx}, the excess CE risk of $\widehat{p}(y|x)$ can be obtained in Theorem \ref{thm::convergencerate} by using Theorem \ref{thm::oracle}.
Furthermore, \eqref{eq::orcleKLR} implies that the truncation on the conditional probability estimator is necessary for a finite sample error bound w.r.t.~the CE loss.
Moreover, the sample error increases with $-\log t$, which implies that a larger truncation threshold $t$ yields a smaller sample error bound.

It is worth noting that the oracle inequality established in \cite[Theorem 3.5]{bos2022convergence} for the truncated CE risk in \eqref{eq::trunrisk} can not be generalized to the CE risk. 
Since their sample error bound grows linearly with the truncation threshold $B$, the sample error bound becomes infinite for the CE risk. 
By contrast, the oracle inequality in Theorem \ref{thm::oracle} is established w.r.t.~the CE loss, which is essentially different from that in \cite{bos2022convergence}.

\subsection{Error Analysis of CPMKM for Class Probability Ratio Estimation}
\label{sec::priorestmation}

The following theorem demonstrates the uniqueness of the solution to our class probability matching on label $Y$.

\begin{theorem}[\textbf{Identifiability}] \label{lem::hold}
	Let Assumptions \ref{ass:labelshift} and \ref{ass::LinearIndependence} hold. Moreover, let the probability ratio $w^*(y)$ be as in \eqref{eq::wstar}. Then the equation system \eqref{eq::equationsystem} holds if and only if $w(y) = w^*(y)$, $y\in [M]$.
\end{theorem}

The following Proposition \ref{prop::decompweighterror} shows the upper bound of $L_2$-norm error of class probability ratio estimation is associated with the excess risk of KLR.

\begin{proposition}\label{prop::decompweighterror}
	Let Assumptions \ref{ass:labelshift}, \ref{ass::marginal}, \ref{ass::LinearIndependence} and \ref{ass::regularity} hold. Furthermore, let the conditional probability estimator $\widehat{p}(y|x)$ be as in \eqref{eq::hatetaP}. 
	Moreover, let $w^*:=(w^*(y))_{y\in [M]}$ be as in \eqref{eq::wstar} and
	$\widehat{w}$ be the solution to \eqref{eq::whatminimizer}. 
	Then with probability at least $1-1/n_p-1/n_q$, there holds
	\begin{align*}
		\|\widehat{w} - w^*\|_2^2
		\lesssim \mathcal{R}_{L_{\mathrm{CE}},P}(\widehat{p}(y|x)) - \mathcal{R}_{L_{\mathrm{CE}},P}^* + \log n_q / n_q + \log n_p / n_p .
	\end{align*} 
\end{proposition}

\subsection{Error Analysis for CPMKM in the Target Domain} \label{sec::analysisQresult}

To derive the excess risk of $\widehat{q}(y|x)$ in \eqref{eq::hatetaQ}, let us define 
\begin{align}\label{eq::tildeetaQ}
	\widetilde{q}(m|x)
	= \frac{w^*(m) \widehat{p}(m|x)}{\sum_{j=1}^M  w^*(j) \widehat{p}(j|x)}.
\end{align}
Then we are able to make the error decomposition for the excess risk of $\widehat{q}(y|x)$ as 
\begin{align}\label{eq::ExcessRiskQdecomp}
	\begin{split}
		\mathcal{R}_{L_{\mathrm{CE}},Q}(\widehat{q}(y|x)) - \mathcal{R}_{L_{\mathrm{CE}},Q}^* 
		& \leq
		\bigl| \mathcal{R}_{L_{\mathrm{CE}},Q}(\widehat{q}(y|x)) - \mathcal{R}_{L_{\mathrm{CE}},Q}(\widetilde{q}(y|x)) \bigr|
		\\
		& \phantom{=}
		+ \mathcal{R}_{L_{\mathrm{CE}},Q}(\widetilde{q}(y|x)) - \mathcal{R}_{L_{\mathrm{CE}},Q}^*. 
	\end{split}
\end{align}

The following Propositions \ref{prop::excesshattildediff} and \ref{prop::excesstildetruediff} provide the upper bound of these two terms in the right-hand side of \eqref{eq::ExcessRiskQdecomp}, respectively.

\begin{proposition}\label{prop::excesshattildediff}
	Let Assumptions \ref{ass:labelshift} hold. Moreover, let $\widehat{q}(y|x)$ and $\widetilde{q}(y|x)$ be defined as in \eqref{eq::hatetaQ} and \eqref{eq::tildeetaQ}, respectively. Then we have 
	\begin{align*}
		\big|\mathcal{R}_{L_{\mathrm{CE}},Q}(\widehat{q}(y|x))-\mathcal{R}_{L_{\mathrm{CE}},Q}(\widetilde{q}(y|x))\big|  
		\lesssim \mathcal{R}_{L_{\mathrm{CE}},P}(\widehat{p}(y|x)) - \mathcal{R}_{L_{\mathrm{CE}},P}^* + \|w^* - \widehat{w}\|_2^2.
	\end{align*}
\end{proposition}

\begin{proposition}\label{prop::excesstildetruediff}
	Let Assumptions \ref{ass:labelshift} hold. Moreover, let $\widetilde{q}(y|x)$ and $\widehat{p}(y|x)$ be as in \eqref{eq::tildeetaQ} and \eqref{eq::hatetaP}, respectively. Then we have 
	\begin{align*}
		\mathcal{R}_{L_{\mathrm{CE}},Q}(\widetilde{q}(y|x))-\mathcal{R}_{L_{\mathrm{CE}},Q}^*
		\lesssim \mathcal{R}_{L_{\mathrm{CE}},P}(\widehat{p}(y|x))-\mathcal{R}_{L_{\mathrm{CE}},P}^*.
	\end{align*}
\end{proposition}

From Propositions \ref{prop::excesshattildediff}, \ref{prop::excesstildetruediff}, and the error decomposition \eqref{eq::ExcessRiskQdecomp}, we can see that the excess CE risk of $\widehat{q}(y|x)$ depends on the excess CE risk of $\widehat{p}(y|x)$ and the error of class probability ratio estimation $\|w^* - \widehat{w}\|_2^2$, which have been analyzed in Sections \ref{sec::analysisKLR} and \ref{sec::priorestmation}, respectively. As a result, we are able to establish the convergence rates of $\widehat{q}(y|x)$ in the target domain, as presented in Theorem \ref{thm::rateQ}.

\section{Experiments} \label{sec::Experiments}

In this section, we conduct numerical experiments to show the performance of our proposed CPMCM for label shift adaptation. In Section \ref{subsec::datasets}, we introduce the real-world datasets and the procedure for generating samples in the source domain and target domain. The compared methods and the evaluation metrics are presented in Section \ref{subsec::methods} and \ref{subsec::metrics}, respectively. The experimental results of different methods under various label shift scenarios are presented in Section \ref{subsec::expresults} to show the empirical superiority of CPMKM over other methods. Moreover, we verify the convergence rates of CPMKM for the label shift adaptation through experiments with different sample sizes.

\subsection{Datasets} \label{subsec::datasets}

We use multi-class benchmark datasets {\tt Dionis} and {\tt Volkert} collected from the OpenML Science Platform \cite{vanschoren2014openml} as well as datasets {\tt Covertype} and {\tt Gas Sensor} from the UCI Machine Learning Repository \cite{Dua:2019}. 
Based on the benchmark datasets, we construct the labeled data from the source domain and the unlabeled data from the target domain under the label shift setting as follows. First, we resample $n_p$ samples from the original dataset according to the uniform class probability $p(y)$ on all classes to form the source domain data $D_p$. 
In order to generate the unlabeled data for the target domain, $D_q^u$, we first resample $n_q$ samples from the remaining dataset according to the class probability $q(y)$. Subsequently, we remove the labels of these $n_q$ selected samples to create the unlabeled target domain data $D_q^u$. Additionally, following the same procedure of generating $D_q^u$, we generate $n_t$ unlabeled test data $D_t$ for evaluating the classification accuracy in the target domain. For repeating experiments of each method, we randomly sub-sample ten different $D_p$ with different random seeds and then train ten different models respectively.
For each model trained with the source domain data $D_p$, we randomly sub-sample ten target domain data $D_q^u$ and test data $D_t$ with different random seeds for each label shift adaptation task.
Therefore, the total number of repetition is $100$.

To generate the class probability of the target domain $q(y)$ with a significant shift from the  uniform class probability $p(y)$ in the source domain, we randomly choose a subset of $M_q$ classes from the total $M$ classes. The class probabilities $q(y)$ for the remaining $M-M_q$ classes are set to be zero. Subsequently, we generate class probabilities for the selected $M_q$ classes using the Dirichlet distribution with parameter $\alpha$. Obviously, the severity of the label shift increases with a smaller $\alpha$. In this paper, we explore Dirichlet shifts with $\alpha \in \{1, 2, 5, 10\}$.
The detailed descriptions of datasets 
and the label shift adaptation setups
are presented as follows and listed in Table \ref{table::DatasetsDescription}.

\begin{table}[htbp]
	\centering
	\captionsetup{justification=centering}
	\caption{Data Descriptions and Label Shift Adaptation Setups}
	\vspace{-2mm}
	\begin{tabular}{c||rrrrr|rrrr}
		\toprule
		Dataset  & $n$ & $d$ & $M$ & $n_{\tt max}$ & $n_{\tt min}$ & $n_p$ & $n_q$ & $n_t$ & $M_q$ \\
		\midrule
		{\tt Dionis} & $416188$ & $61$ & $355$ & $2469$ & $878$ & $14200$ & $7100$ & $17750$ & $100$ \\
		{\tt Volkert} & $58310$ & $181$ & $10$ & $12806$ & $1361$ & $10000$ & $5000$ & $25000$ & $4$ \\
		{\tt Covertype} & $581012$ & $54$ & $7$ & $283301$ & $2724$ & $7000$ & $7000$ & $21000$ & $5$ \\
		{\tt Gas Sensor} & $13910$ & $128$ & $6$ & $3009$ & $1,641$ & $3000$ & $1000$ & $3000$ & $4$
		\\
		\bottomrule
	\end{tabular}
	\label{table::DatasetsDescription}
\end{table}

\subsection{Comparison Methods} \label{subsec::methods}

We consider the following methods for label shift adaptation.
\begin{itemize}
	\item \textbf{KMM} \cite{zhang2013domain}. The estimation of class probability ratio involves a constrained optimization with automatic hyper-parameter selection. 
	\item \textbf{BBSE} \cite{lipton2018detecting}. The proposed method can use arbitrary black box predictors for class probability ratio estimation. Here we use the KLR classifier with rbf kernel as the black box prediction. 
	\item \textbf{RLLS} \cite{azizzadenesheli2019regularized}.
	Similar to BBSE, RLLS can also use arbitrary black box predictors for class probability ratio estimation. However, different from BBSE, RLLS learn the ratio in a regularized way to compensate for the high estimation error in the low target sample scenarios.
	\item \textbf{ELSA} \cite{tian23a}. The proposed method conducts ratio estimation by using a moment-matching framework based on the geometry of the influence function under a semiparametric model.
	\item \textbf{MLLS} \cite{alexandari2020maximum, garg2020unified}. Different from matching-based methods KMM, BBSE, and RLLS, MLLS estimates the class probability ratio by maximizing the log likelihood. 
	\item \textbf{CPMKM} (Ours): We use the class probability matching framework proposed in Section \ref{sec::Methodology} equipped with kernel logistic regression for estimating the ratio.
\end{itemize}
To implement the compared method KMM, we use the code in \url{http://people.tuebingen.mpg.de/kzhang/Code-TarS.zip}.
To implement the compared methods BBSE, RLLS and MLLS, we use the code in \url{https://github.com/kundajelab/labelshiftexperiments} provided by \cite{alexandari2020maximum}.
We implement the compared method ELSA according to the algorithmic description in \cite{tian23a}.
For all methods, we reweight the KLR predictor $\widehat{p}(y|x)$ by using the ratio estimation $\widehat{w}$ to get the predictor $\widehat{q}(y|x)$ in the target domain via \eqref{eq::hatetaQ} and plug-in classifier by \eqref{eq::fhat}.

We mention that in the process of fitting KLR to the source domain $D_p$, we select two hyper-parameters for KLR, including the regularization parameter $C$ and the kernel coefficient $\gamma$, by using five-fold cross-validation.
The criteria of five-hold cross-validation is the CE loss instead of the commonly-used mis-classification loss since our primary goal is to estimate $p(y|x)$ well. 
To be specific, the hyper-parameter $C$ is selected from 7 numbers spaced evenly on a log scale from $10^{-6}$ to $10^0$, and $\gamma$ is selected from 7 numbers spaced evenly on a log scale from $2^{-6}$ to $2^{0}$. Moreover, the truncation parameter is set as $t = 10^{-8}$.

\subsection{Evaluation Metrics} \label{subsec::metrics}

We consider the following two metrics for the evaluations of label shift adaptation problems.
To be specific, the first one {\tt ACC} is the classification accuracy in the target domain.
The second one {\tt MSE} is used to measure the estimation error of the class probability $q(y)$ in the target domain. Note that a larger {\tt ACC} and a smaller {\tt MSE} indicate better performance. 
\begin{itemize}
	\item[\textit{(i)}] {\tt ACC} is the accuracy of classifier in the target domain evaluated on the test data $D_t$. Mathematically speaking, {\tt ACC} equals $n_t^{-1} \sum_{(X_i,Y_i) \in D_t} \eins\{\widehat{h}_q(X_i) = Y_i\}$. 
	\item[\textit{(ii)}] {\tt MSE} is the mean squared error between the true class probability $q(y)$ and the estimated class probability $\widehat{q}(y)$. Since different methods have different estimation $\widehat{p}(y)$, their class probability ratio estimation $\widehat{w}$ satisfies different normalization condition $\sum_{m\in [M]} w(m) \widehat{p}(m) = 1$. Therefore, for a fair comparison, we choose to compare the MSE for $q(y)$ instead of comparing the MSE for class probability ratio $w^*$. For each method, we use normalized $\widehat{w}(y)\widehat{p}(y)$ as the estimate of $q(y)$, $y\in [M]$. 
\end{itemize}

\subsection{Experimental Results} \label{subsec::expresults}

Tables \ref{tab::res_ACC} and \ref{tab::res_MSE} present the results of the compared methods for different label shift scenarios on four datasets. From Tables \ref{tab::res_ACC} and \ref{tab::res_MSE} we find that the proposed CPMKM method has the lowest estimation error of the class probability $q(y)$ and highest classification accuracy in most cases, which shows that matching method on $Y$ is superior to compared methods in performance. It is worth noting that on the {\tt Dionis} dataset with the most number of classes, the advantage of our CPMKM over the compared method is the most notable.

Moreover, to explore the effect of different sample sizes on the performance of compared methods, we fix the sample size $n_p$ in the source domain and vary the sample size $n_q$ in the target domain. We conduct experiments on the dataset {\tt Dionis} since it has the largest number of classes. Figure \ref{fig:mean and std of nets} presents the performance of the compared methods for the fixed source domain sample size $n_p = 14200$ and $n_p = 7100$ under Dirichlet parameter $\alpha = 10$. 

From Figure \ref{fig:mean and std of nets}, we can see that if we fix $n_p = 14200$, for all methods, the accuracy improves and the MSE decreases as the target domain sample size $n_q$ increases from $880$ to $3550$.

\begin{table}[H]
	\centering
	\captionsetup{justification=centering} 
	\caption{Comparison of {\tt ACC} Performance with Other Methods.}
	\resizebox{0.77\textwidth}{!}{
		\label{tab::res_ACC}
		\begin{tabular}{c|cccccc}
			\toprule
			Dataset & {KMM} & {BBSE} & {RLLS} & {ELSA} & {MLLS} & {CPMKM (Ours)} \\
			\hline
			& \multicolumn{6}{c}{$\alpha=1$} \\
			\hline
			\multirow{2}{*}{{\tt Dionis}} & 76.20 & 82.37 & 82.45 & 81.58 & 82.36 & \textbf{82.81}  \\
			& (2.13) & (1.45) & (1.40) & (2.91) & (1.71) & \textbf{(1.40)} \\
			\hline
			\multirow{2}{*}{{\tt Volkert}} & 70.19 & 69.41 & 69.50 & 70.03 & 70.18 & \textbf{70.57} \\
			& (2.80) & (1.47) & (1.37) & (1.38) & (1.41) & \textbf{(1.20)} \\
			\hline
			\multirow{2}{*}{{\tt Covertype}} & \textbf{82.42} & 81.66 & 81.70 & 82.23 & 82.41 & 82.25 \\
			& \textbf{(1.93)} & (2.79) & (2.71) & (2.91) & (2.98) & (2.95) \\
			\hline
			\multirow{2}{*}{{\tt Gas Sensor}} & 95.19 & 96.35 & 96.39 & 96.31 & 96.35 & \textbf{96.43} \\
			& (1.77) & (0.91) & (0.86) & (0.91) & (0.92) & \textbf{(0.97)} \\
			\midrule
			& \multicolumn{6}{c}{$\alpha=2$} \\
			\hline
			\multirow{2}{*}{{\tt Dionis}} & 76.67 & 82.18 & 82.23 & 81.39 & 82.09 & \textbf{82.59} \\
			& (2.08) & (1.53) & (1.48) & (2.84) & (1.56) & \textbf{(1.36)} \\
			\hline
			\multirow{2}{*}{{\tt Volkert}} & 70.58 & 69.99 & 70.09 & 70.69 & 70.78 & \textbf{70.96} \\
			& (2.76) & (1.38) & (1.29) & (1.31) & (1.37) & \textbf{(1.33)} \\
			\hline
			\multirow{2}{*}{{\tt Covertype}} & \textbf{79.50} & 78.94 & 78.99 & 79.43 & 79.50 & 79.45 \\
			& \textbf{(2.61)} & (1.90) & (1.83) & (1.98) & (2.07) & (2.06) \\
			\hline
			\multirow{2}{*}{{\tt Gas Sensor}} & 95.87 & 96.16 & 96.18 & 96.16 & 96.18 & \textbf{96.21} \\
			& (1.48) & (1.09) & (1.00) & (1.01) & (1.01) & \textbf{(1.10)} \\
			\midrule
			
			& \multicolumn{6}{c}{$\alpha=5$} \\
			\hline
			\multirow{2}{*}{{\tt Dionis}} & 77.80 & 81.88 & 81.90 & 81.18 & 81.78 & \textbf{82.49} \\
			& (1.71) & (1.21) & (1.12) & (5.72) & (1.42) & \textbf{(1.15)} \\
			\hline
			\multirow{2}{*}{{\tt Volkert}} & 69.42 & 68.33 & 68.53 & 69.08 & 69.16 & \textbf{69.66} \\
			& (1.89) & (1.04) & (0.94) & (0.94) & (0.95) & \textbf{(0.82)} \\
			\hline
			\multirow{2}{*}{{\tt Covertype}} & 78.46 & 78.01 & 78.05 & 78.32 & \textbf{78.46} & 78.35 \\
			& (2.02) & (1.90) & (1.83) & (1.98) & \textbf{(2.07)} & (2.06) \\
			\hline
			\multirow{2}{*}{{\tt Gas Sensor}} & 96.34 & 96.37 & 96.39 & 96.40 & 96.38 & \textbf{96.42} \\
			& (0.98) & (1.00) & (1.01) & (1.04) & (1.05) & \textbf{(1.26)} \\
			\midrule
			& \multicolumn{6}{c}{$\alpha=10$} \\
			\hline
			\multirow{2}{*}{{\tt Dionis}} & 78.13 & 81.66 & 81.70 & 81.03 & 81.75 & \textbf{82.33} \\
			& (1.20) & (0.89) & (0.90) & (2.23) & (1.19) & \textbf{(0.95)} \\
			\hline
			\multirow{2}{*}{{\tt Volkert}} & 68.68 & 67.84 & 68.01 & 68.52 & 68.65 & \textbf{69.04} \\
			& (2.06) & (1.07) & (0.97) & (0.98) & (0.99) & \textbf{(1.03)} \\
			\hline
			\multirow{2}{*}{{\tt Covertype}} & 78.04 & 77.69 & 77.73 & 77.92 & \textbf{78.05} & 78.02 \\
			& (1.87) & (1.74) & (1.70) & (1.83) & \textbf{(1.90)} & (1.87) \\
			\hline
			\multirow{2}{*}{{\tt Gas Sensor}} & 96.28 & 96.37 & 96.38 & 96.39 & 96.40 & \textbf{96.42} \\
			& (1.08) & (1.02) & (0.99) & (0.99) & (1.01) & \textbf{(1.04)} \\
			\bottomrule
		\end{tabular}
	}
	\begin{tablenotes}
		\footnotesize
		\item[*] For each dataset and each $\alpha$, the best result is marked in \textbf{bold}.
	\end{tablenotes}
\end{table}

\begin{table}[H]
	\centering
	\captionsetup{justification=centering} 
	\caption{Comparison of {\tt MSE} Performance with Other Methods.}
	\resizebox{0.87\textwidth}{!}{
		\label{tab::res_MSE}
		\begin{tabular}{c|cccccc}
			\toprule
			Dataset & {KMM} & {BBSE} & {RLLS} & {ELSA} & {MLLS} & {CPMKM (Ours)} \\
			\hline
			& \multicolumn{6}{c}{$\alpha=1$} \\
			\hline
			\multirow{2}{*}{{\tt Dionis}} & 5.30e-6 & 2.54e-6 & 2.49e-6 & 3.57e-6 & 3.00e-6 & \textbf{2.39e-6} \\
			& (1.81e-6) & (8.26e-7) & (8.12e-7) & (1.43e-6) & (1.49e-6) & \textbf{(9.24e-7)} \\
			\hline
			\multirow{2}{*}{{\tt Volkert}} & 5.76e-4 & 1.03e-3 & 1.00e-3 & 7.22e-4 & 4.07e-4 & \textbf{3.36e-4} \\
			& (2.32e-4) & (4.26e-4) & (2.31e-4) & (3.44e-4) & (2.14e-4) & \textbf{(2.34e-4)} \\
			\hline
			\multirow{2}{*}{{\tt Covertype}} & \textbf{4.66e-4} & 1.13-3 & 1.02e-3 & 6.90e-4 & 4.73e-4 & 5.42e-4 \\
			& \textbf{(6.39e-4)} & (1.25e-3) & (1.10e-3) & (9.92e-4) & (8.81e-4) & (1.05e-3) \\
			\hline
			\multirow{2}{*}{{\tt Gas Sensor}} & 9.31e-3 & 9.16e-3 & 9.15e-3 & 9.12e-3 &  9.10e-3 & \textbf{9.09e-3} \\
			& (1.15e-4) & (1.04e-2) & (1.04e-2) & (1.04e-2) & (1.03e-2) & \textbf{(1.03e-2)} \\
			\midrule
			& \multicolumn{6}{c}{$\alpha=2$} \\
			\hline
			\multirow{2}{*}{{\tt Dionis}}  & 4.30e-6 & 2.47e-6 & 2.44e-6 & 2.93e-6 & 2.51e-6 & \textbf{2.28e-6} \\
			& (1.07e-6) & (6.55e-7) & (6.46e-7) & (9.60e-6) & (1.19e-6) & (6.73e-7) \\
			\hline
			\multirow{2}{*}{{\tt Volkert}} & 5.31e-4 & 9.86e-4 & 7.94e-4 & 6.67e-4 & 3.74e-4 & \textbf{3.23e-4} \\
			& (2.89e-4) & (4.04e-4) & (3.01e-4) & (3.15e-4) & (1.84e-4) & \textbf{(2.68e-4)} \\
			\hline
			\multirow{2}{*}{{\tt Covertype}} & 5.47e-4 & 9.41e-4 & 9.28e-4 & 6.63e-4 & \textbf{5.06e-4} & 5.90e-4 \\
			& (9.42e-4) & (1.24e-3) & (1.08e-3) & (9.92e-4) & \textbf{(8.81e-4)} & (1.02e-3) \\
			\hline
			\multirow{2}{*}{{\tt Gas Sensor}} & 8.23e-3 & 8.21e-3 & 8.20e-3 & 8.19e-3 & 8.19e-3 & \textbf{8.18e-3}\\
			& (1.02e-2) & (1.03e-2) & (1.03e-2) & (1.03e-2) & (1.02e-2) & \textbf{(1.02e-2)}\\
			\midrule
			& \multicolumn{6}{c}{$\alpha=5$} \\
			\midrule
			\multirow{2}{*}{{\tt Dionis}} & 3.55e-6 & 2.15e-6 & 2.10e-6 & 2.63e-6 & 2.44e-6 & \textbf{1.93e-6} \\
			& (6.93e-7) & (5.96e-7) & (5.87e-7) & (5.34e-6) & (1.05e-6) & \textbf{(5.97e-7)} \\
			\hline
			\multirow{2}{*}{{\tt Volkert}} & 5.43e-4 & 1.02e-3 & 8.54e-4 & 7.33e-4 & 4.11e-4& \textbf{3.21e-4} \\
			& (2.47e-4) & (3.85e-4) & (3.39e-4) & (3.10e-4) & (1.90e-4) & \textbf{(1.78e-4)} \\
			\hline
			\multirow{2}{*}{{\tt Covertype}} & 3.05e-4 & 5.32e-4 & 5.12e-4 & 3.86e-4 & \textbf{2.93e-4} & 3.32e-4 \\
			& (4.99e-4) & (6.48e-4) & (6.06e-4) & (5.37e-4) & \textbf{(4.65e-4)} & (5.39e-4) \\
			\hline
			\multirow{2}{*}{{\tt Gas Sensor}} & 6.43e-3 & 6.39e-3 & 6.39e-3 & 6.39e-3 & 6.38e-3 & \textbf{6.38e-3} \\
			& (8.04e-3)  & (8.16e-3) & (8.15e-3) & (8.16e-3) & (8.16e-3) & \textbf{(8.15e-3)} \\
			\midrule
			& \multicolumn{6}{c}{$\alpha=10$} \\
			\hline
			\multirow{2}{*}{{\tt Dionis}} & 3.33e-6 & 1.90e-6 & 1.86e-6 & 2.46e-6 & 2.04e-6 & \textbf{1.68e-6} \\
			& (5.89e-7) & (2.87e-7) & (2.78e-7) & (1.29e-6) & (6.04e-7) & \textbf{(2.86e-7)} \\
			\hline
			\multirow{2}{*}{{\tt Volkert}} & 6.41e-4 & 1.01e-3 & 8.52e-4 & 7.56e-4 & 4.05e-4 & \textbf{3.22e-4} \\
			& (2.68e-4) & (3.70e-4) & (3.28e-4) & (3.24e-4) & (2.12e-4) & \textbf{(1.47e-4)} \\
			\hline
			\multirow{2}{*}{{\tt Covertype}} & 2.35e-4 & 3.56e-4 & 3.47e-4 & 2.71e-4 & \textbf{2.16e-4} & 2.30e-4 \\
			& (3.05e-4) & (3.16e-4) & (3.04e-4) & (2.77e-4) & \textbf{(2.81e-4)} & (2.86e-4) \\
			\hline
			\multirow{2}{*}{{\tt Gas Sensor}} & 3.58e-3 & 3.52e-3 & 3.51e-3 & 3.51e-3 & 3.50e-3 & \textbf{3.50e-3} \\
			& (4.52e-3) & (4.50e-3) & (4.49e-3) & (4.49e-3) & (4.48e-3) & \textbf{(4.48e-3)} \\
			\bottomrule
		\end{tabular}
	}
	\begin{tablenotes}
		\footnotesize
		\item[*] For each dataset and each $\alpha$, the best result is marked in \textbf{bold}.
	\end{tablenotes}
\end{table}

If $n_q$ continues to increase from $3550$ to $14420$, we find that the performance almost keeps steady. This trend verifies the convergence rates established in Theorem \ref{thm::rateQ}, which shows that if we fix $n_p$ and gradually increase $n_q$ from zero to infinity, the convergence rate becomes faster firstly and then keep constant after $n_q$ reaches a threshold smaller than $n_p$. Moreover, our CPMKM outperforms other compared methods across different sample sizes $n_q$ both on classification accuracy and class probability estimation. 
Since KMM performs significantly worse than other methods on {\tt Dionis}, its curve is excluded from Figure \ref{fig:mean and std of nets} to enhance clarity in presenting the distinctions among the performance curves of various methods.

\begin{figure*}[!h]
	\centering
	\captionsetup[subfigure]{justification=centering, captionskip=2pt} 
	\subfloat[][{\tt ACC} for $n_p = 14200$] { 
		\includegraphics[width=0.49\linewidth]{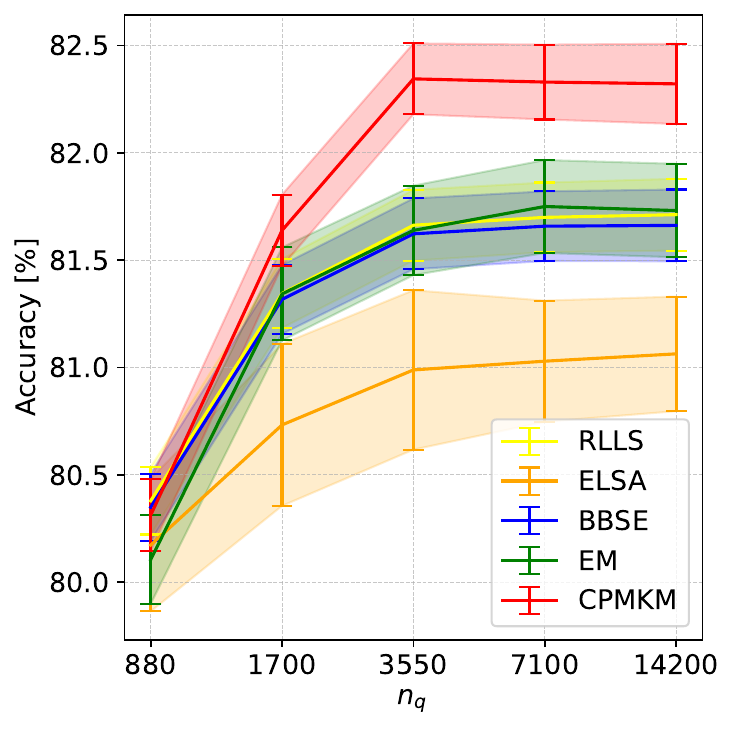} 
			\label{fig:mean and std of net14}
	} 
	\subfloat[][{\tt MSE} for $n_p = 14200$] { 
		\includegraphics[width=0.49\linewidth]{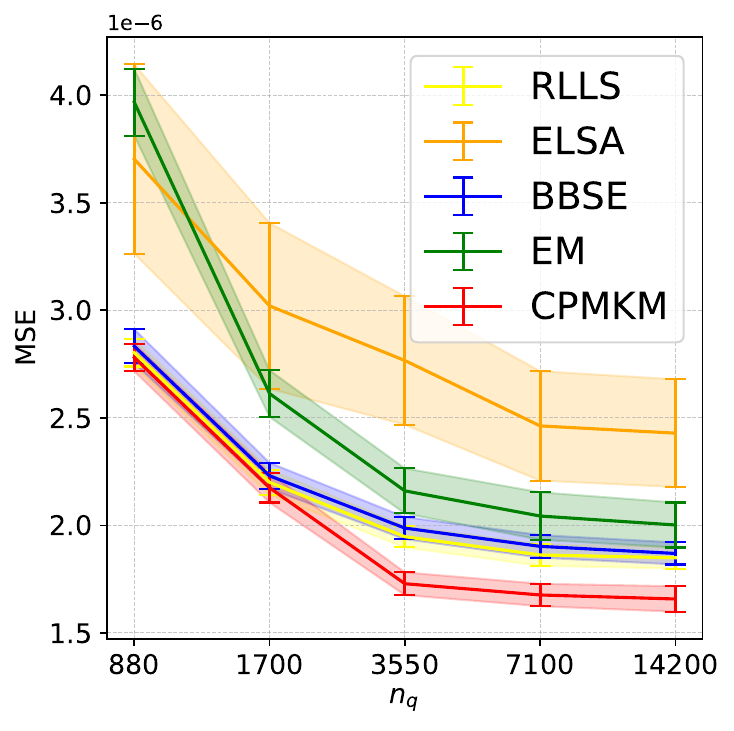}
		\label{fig:mean and std of net24}
	} 
\\
	\subfloat[][{\tt ACC} for $n_p = 7100$] { 
	\includegraphics[width=0.49\linewidth]{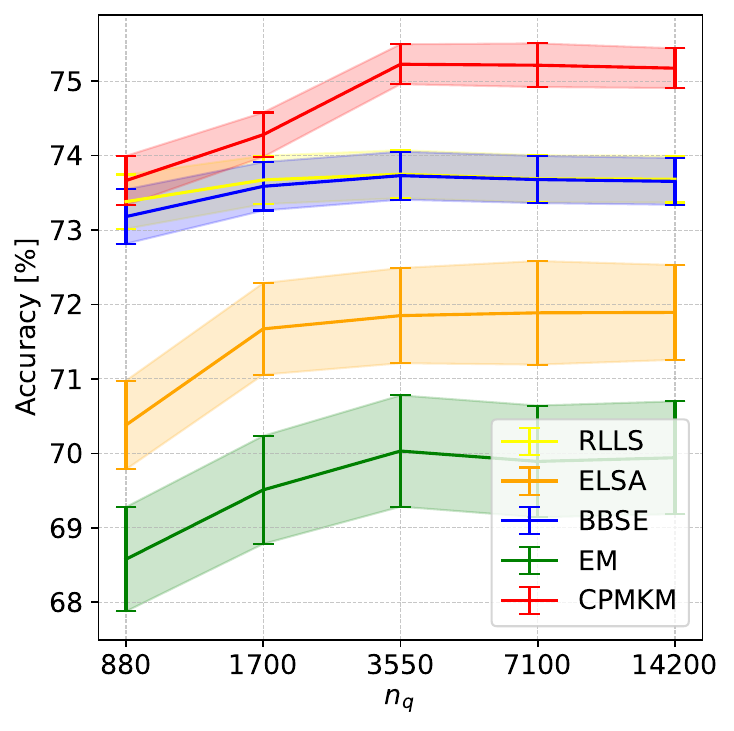} 
	\label{fig:mean and std of net34}
} 
\subfloat[][{\tt MSE} for $n_p = 7100$] { 
	\includegraphics[width=0.49\linewidth]{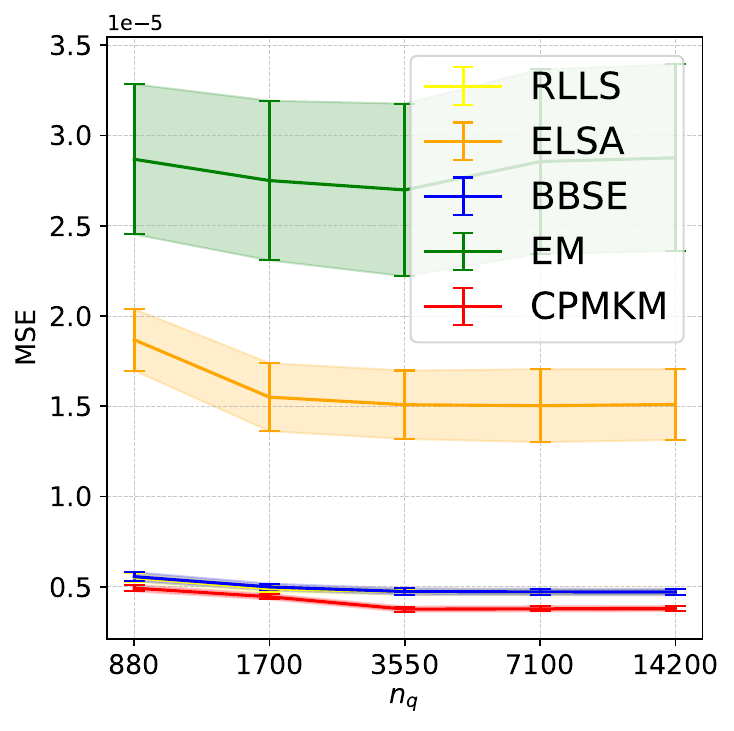}
	\label{fig:mean and std of net44}
} 
	\caption{Performance on {\tt Dionis} dataaset under Dirichlet shift $\alpha = 10$ for varied $n_q$.}
	\label{fig:mean and std of nets}
\end{figure*}

%
%

\section{Proofs} \label{sec::Proofs}

In this section, we present the proofs related to Section \ref{sec::analysisKLR}, \ref{subsec::RatesSource}, \ref{sec::priorestmation}, \ref{sec::analysisQresult}, and \ref{subsec::RatesTarget} in Sections \ref{sec::proofKLR}-\ref{sec::proofRate}, respectively. To be specific, Section \ref{sec::proofKLR}
presents the proofs related to the error analysis for truncated KLR in the source domain. The proofs of the convergence rates of KLR and the lower bound of multi classification w.r.t.~the CE loss are presented in Secction \ref{sec::proofrateP}. 
Section \ref{sec::proofprior} provides the proof for bounding the estimation error of class probability ratio. In Section \ref{sec::proofanalysisQ}, we prove that the excess risk of CPMKM depends on the estimation error of class probability ratio and the excess risk of the truncated KLR in the source domain. Finally, the proofs of the convergence rates of CPMKM and the lower bound of label shift adaptation are provided in Section \ref{sec::proofRate}.

\subsection{Proofs Related to Section \ref{sec::analysisKLR}} \label{sec::proofKLR}

\subsubsection{Proofs Related to Section \ref{sec::approxerror}}

In order to prove Proposition \ref{prop::approx}, we need to construct a score function $f_0 \in \mathcal{F}$. To this end, we need to introduce the following notations. Given $\tau \in (0, 1/(2M))$ and $m\in [M]$, we define the truncated conditional probability function by
\begin{align}\label{eq::etaPdelta}
	p^\tau(m|x)
	:= 
	\begin{cases}
		\tau, & \text{ if } p(m|x)< \tau,
		\\
		\displaystyle p(m|x) - \frac{p(m|x) - \tau}{\sum_{\ell:p(\ell|x)\geq \tau}(p(\ell|x)-\tau)} 
		\sum_{j:p(j|x)< \tau}(\tau-
		p(j|x)), & \text{ if } p(m|x)\geq \tau.
	\end{cases}
\end{align}
Then the corresponding truncated score function is defined by
\begin{align}\label{eq::fdelta}
	f_m^{*\tau}(x)
	:= \log \frac{p^\tau(m|x)}{p^\tau(M|x)}.
\end{align}
For any fixed $\gamma > 0$, we define the function $K : \mathbb{R}^d \to \mathbb{R}$ by
\begin{align}\label{Conv}
	K(x) :=
	\biggl( \frac{2}{ \gamma^2 \pi} \biggr)^{d/2} 
	\exp \biggl( - \frac{2\|x\|_2^2}{\gamma^2} \biggr).
\end{align}
Then we define the convolution of $f_m^{*\tau}$ and $K$ as 
\begin{align}\label{eq::ftildem}
	\widetilde{f}_m^\tau(x) 
	:= (K * f_m^{*\tau})(x) 
	:= \int_{\mathbb{R}^d} K(x - z) f_m^{*\tau}(z) \, dz, 
	\qquad 
	m \in [M],
\end{align}
and the score function $\widetilde{f}^\tau:= (\widetilde{f}_m^\tau)_{m\in[M]}$.
In the following, we show that the constructed score function $f_0 := \widetilde{f}^{\tau}$ with $\tau := \gamma^{\alpha}$ lies in the function space $\mathcal{F}$ and $f_0$ satisfies the approximation error bound $\mathcal{R}_{L_{\mathrm{CE}},P}(p^t_{f_0}(y|x)) - \mathcal{R}_{L_{\mathrm{CE}},P}^* 
\lesssim \gamma^{\alpha(1+\beta\wedge 1)}$
as in \eqref{eq::approxt} of Proposition \ref{prop::approx}. To this end, we need to introduce the following notations. Denote 
\begin{align}\label{eq::g1delta}
	g_m^\tau(x)
	:= \log p^\tau(m|x).
\end{align}
Then $f_m^{*\tau}(x)$ in \eqref{eq::fdelta} can be expressed as 
\begin{align}\label{eq::gk-gK}
	f_m^{*\tau}(x)
	= g_m^\tau(x) - g_M^\tau(x).
\end{align}

To analyze the approximation error of $\widetilde{f}^{\tau}$, we first need the following lemma.

\begin{lemma}\label{lem::lipg1}
	Let Assumptions \ref{ass:labelshift} and \ref{ass::predictor} hold. 
	Moreover, let $g_m^\tau$ be defined as in \eqref{eq::g1delta}. Then for any $x, x' \in \mathcal{X}$, we have
	\begin{align*}
		|g_m^\tau(x)-g_m^\tau(x')| \leq \frac{\log \big(1+ c_L\|x-x'\|_2^{\alpha}\big)}{p^\tau(m|x) \wedge p^\tau(m|x')}.
	\end{align*}
\end{lemma}

To prove Lemma \ref{lem::lipg1}, we need
the following Lemmas \ref{lem::ln} and \ref{lem::holder}.

\begin{lemma}\label{lem::ln}
	For any $c>1$ and $z>0$, we have
	$\log (1+cz) \leq c\log (1+z)$.
\end{lemma}

\begin{proof}[Proof of Lemma \ref{lem::ln}]
	For any $c>1$ and $z>0$, define 
	$h(z) := \log (1+cz) - c\log (1+z)$.
	Then we have
	\begin{align*}
		h'(z) = \frac{c}{1+cz} - \frac{c}{1+z} = \frac{c(1-c)z}{(1+cz)(1+z)} < 0
	\end{align*}
	and thus $h$ is decreasing on $[0,\infty]$. Therefore, for any $z > 0$, there holds $h(z) < h(0) = 0$. Then the definition of $h$ yields the conclusion.
\end{proof}

\begin{lemma}\label{lem::holder}
	Let Assumptions \ref{ass:labelshift} and \ref{ass::predictor} hold. Moreover, let $p^\tau(m|x)$ be defined as in \eqref{eq::etaPdelta}. Then for any $x,x'\in\mathcal{X}$, there holds
	\begin{align*}
		|p^\tau(m|x)-p^\tau(m|x')| 
		\leq (4M+1)|p(m|x) - p(m|x')| 
		\leq c_L\|x-x'\|_2^{\alpha},
	\end{align*} 
	where $c_L := (4 M + 1) c_{\alpha}$.
\end{lemma}

\begin{proof}[Proof of Lemma \ref{lem::holder}]
	Given $\tau \in (0,1/(2M))$, we define the label set $\mathcal{M}_x^\tau := \{ m \in [M] : p_f(m|x) < \tau\} \subset [M]$.
	By \eqref{eq::etaPdelta} and Assumption \ref{ass::predictor} \textit{(i)}, for any $m \in \mathcal{M}_{p, x}^\tau$, there holds 
	\begin{align*}
		|p^\tau(m|x)-p^\tau(m|x')| 
		= 
		\begin{cases}
			0, & \text{ if } m \in \mathcal{M}_{p, x'}^\tau,
			\\
			p(m|x') - \tau 
			< p(m|x') - p(m|x) 
			\leq c_{\alpha} \|x - x'\|_2^{\alpha},
			& \text{ if } m \notin \mathcal{M}_{p, x'}^\tau.
		\end{cases}
	\end{align*}
	For any $m \notin \mathcal{M}_{p, x}^\tau$ and $m \in \mathcal{M}_{p, x'}^\tau$, by using Assumption \ref{ass::predictor} \textit{(i)}, we get 
	$$
	|p^\tau(m|x) - p^\tau(m|x')| 
	= p^\tau(m|x) - \tau 
	\leq p(m|x) - p(m|x')
	\leq c_{\alpha} \|x - x'\|_2^{\alpha}.
	$$
	Otherwise, for any $m \notin \mathcal{M}_{p, x}^\tau$ and $m \notin \mathcal{M}_{p, x'}^\tau$, since $\tau \in (0, 1 / (2 M))$, then for any $x \in \mathcal{X}$, we have 
	\begin{align}\label{eq::residuelower}
		\sum_{\ell \notin \mathcal{M}_{p,x}^\tau}(p(\ell|x)
		- \tau) - \sum_{j \in \mathcal{M}_{p,x}^\tau}(\tau- p(j|x)) 
		= \sum_{\ell=1}^M  (p(\ell|x) - \tau) 
		= 1 - M \tau 
		> 1/2.
	\end{align}
	Therefore, by the triangle inequality and Assumption \ref{ass::predictor} \textit{(i)}, there holds
	\begin{align}\label{eq::diffxx'}
		& \bigl| p^\tau(m|x) - p^\tau(m|x') \bigr| 
		\nonumber\\
		& = \bigl| \bigl( p^\tau(m|x) - \tau \bigr) - \bigl( p^\tau(m|x') - \tau \bigr) \bigr| 
		\nonumber\\
		& = \biggl| \bigl( p(m|x) - \tau \bigr) \biggl( 1 -\frac{\sum_{j \in \mathcal{M}_{p,x}^\tau} \bigl( \tau - p(j|x) \bigr)}{\sum_{\ell \notin \mathcal{M}_{p,x}^\tau} \bigl( p(\ell|x) - \tau \bigr)} \biggr)
		- \bigl( p(m|x') - \tau) \biggl( 1 - \frac{\sum_{j \in \mathcal{M}_{p,x'}^\tau} \bigl( \tau - p(j|x') \bigr)}{\sum_{\ell \notin \mathcal{M}_{p,x'}^\tau} \bigl( p(\ell|x') - \tau)} \biggr) \biggr|
		\nonumber\\
		& = \biggl| \bigl( p(m|x) - \tau \bigr) \cdot \frac{1 - M \tau}{\sum_{\ell \notin \mathcal{M}_{p,x}^\tau} \bigl( p(\ell|x) - \tau)} - \bigl( p(m|x') - \tau \bigr) \cdot \frac{1 - M \tau}{\sum_{\ell \notin \mathcal{M}_{p,x'}^\tau} \bigl( p(\ell|x') - \tau \bigr)} \biggr|
		\nonumber\\
		& = \biggl| \bigl( p(m|x) - \tau \bigr) \cdot \frac{1 - M \tau}{\sum_{\ell \notin \mathcal{M}_{p,x}^\tau} \bigl( p(\ell|x) - \tau \bigr)} - \bigl( p(m|x') - \tau \bigr) \cdot \frac{1 - M \tau}{\sum_{\ell \notin \mathcal{M}_{p,x'}^\tau} \bigl( p(\ell|x) - \tau \bigr)}  
		\nonumber\\
		& \phantom{=}
		+ \bigl( p(m|x') - \tau \bigr) \cdot \frac{1 - M \tau}{\sum_{\ell \notin \mathcal{M}_{p,x}^\tau} \bigl( p(\ell|x) - \tau \bigr)} - \bigl( p(m|x') - \tau \bigr) \cdot \frac{1 - M \tau}{\sum_{\ell \notin \mathcal{M}_{p,x'}^\tau} \bigl( p(\ell|x') - \tau \bigr)} \biggr|
		\nonumber\\
		& \leq \bigl| p(m|x) - p(m|x') \bigr| 
		+ \Biggl| \bigl( p(m|x') - \tau \bigr) \cdot \frac{\sum_{\ell \notin \mathcal{M}_{p,x'}^\tau} \bigl( p(\ell|x') - \tau \bigr) - \sum_{\ell \notin \mathcal{M}_{p,x}^\tau} \bigl( p(\ell|x) - \tau \bigr)}{\bigl( \sum_{\ell \notin \mathcal{M}_{p,x}^\tau} \bigl( p(\ell|x) - \tau \bigr) \bigr) \cdot \bigl( \sum_{\ell \notin \mathcal{M}_{p,x'}^\tau} \bigl( p(\ell|x') - \tau \bigr) \bigr)} \Biggr|
		\nonumber\\
		& \leq c_{\alpha} \|x - x'\|_2^{\alpha}
		+ \frac{\bigl| \sum_{\ell \notin \mathcal{M}_{p,x'}^\tau} \bigl( p(\ell|x') - \tau \bigr) - \sum_{\ell \notin \mathcal{M}_{p,x}^\tau} \bigl( p(\ell|x) - \tau \bigr)|}{\bigl( \sum_{\ell \notin \mathcal{M}_{p,x}^\tau} \bigl( p(\ell|x) - \tau \bigr) \cdot \bigl( \sum_{\ell \notin \mathcal{M}_{p,x'}^\tau} \bigl( p(\ell|x') - \tau \bigr) \bigr)}.
	\end{align}
	By the triangle inequality, we have 
	\begin{align}\label{eq::notindiff}
		& \Biggl| \sum_{\ell \notin \mathcal{M}_{p,x'}^\tau} \bigl( p(\ell|x') - \tau \bigr) - \sum_{\ell \notin \mathcal{M}_{p,x}^\tau} \bigl( p(\ell|x) - \tau \bigr) \Biggr|
		\nonumber\\
		& \leq \sum_{\ell \notin (\mathcal{M}_{p,x}^\tau \cup \mathcal{M}_{p,x'}^\tau)} \bigl| p(\ell|x') - p(\ell|x) \bigr| + \sum_{\ell \notin \mathcal{M}_{p,x}^\tau, \ell\in \mathcal{M}_{p,x'}^\tau} \bigl| p(\ell|x') - \tau \bigr| + \sum_{\ell \in \mathcal{M}_{p,x}^\tau, \ell\notin \mathcal{M}_{p,x'}^\tau} \bigl| p(\ell|x) - \tau \bigr|.
	\end{align}
	For $\ell \in \mathcal{M}_{p,x'}^\tau$ and $\ell \notin \mathcal{M}_{p,x}^\tau$, by Assumption \ref{ass::predictor} \textit{(i)}, we have 
	\begin{align*}
		|p(\ell|x')-\tau| = \tau - p(\ell|x') \leq p(\ell|x) - p(\ell|x') \leq c_{\alpha} \|x-x'\|_2^{\alpha}. 
	\end{align*}
	Similarly, for $\ell \in \mathcal{M}_{p,x}^\tau$ and $\ell \notin \mathcal{M}_{p,x'}^\tau$, by Assumption \ref{ass::predictor} \textit{(i)}, we have 
	\begin{align*}
		|p(\ell|x)-\tau| = \tau - p(\ell|x) \leq p(\ell|x') - p(\ell|x) \leq c_{\alpha} \|x-x'\|_2^{\alpha}. 
	\end{align*}
	Therefore, combining \eqref{eq::notindiff} and Assumption \ref{ass::predictor} \textit{(i)}, we obtain
	\begin{align*}
		\bigg|\sum_{\ell \notin \mathcal{M}_{p,x'}^\tau}(p(\ell|x')-\tau)-\sum_{\ell \notin \mathcal{M}_{p,x}^\tau}(p(\ell|x)-\tau)\bigg|
		\leq c_{\alpha} \sum_{\ell \notin (\mathcal{M}_{p,x'}^\tau\cup \mathcal{M}_{p,x}^\tau)} \|x-x'\|_2^{\alpha} 
		\leq c_{\alpha} M \|x-x'\|_2^{\alpha}.
	\end{align*}
	This together with \eqref{eq::residuelower} and \eqref{eq::diffxx'} yields
	\begin{align}\label{eq::4K+1}
		|p^\tau(m|x) - p^\tau(m|x')| 
		& \leq c_{\alpha} \|x-x'\|_2^{\alpha}
		+ \frac{c_{\alpha}  M \|x-x'\|_2^{\alpha}}{(1/2)\cdot (1/2)}
		\nonumber\\
		& = (4 M + 1) c_{\alpha} \|x-x'\|_2^{\alpha} 
		=: c_L \|x-x'\|_2^{\alpha},
	\end{align}
	where $c_L := (4 M + 1) c_{\alpha}$. 
	Therefore, we finish the proof of the second inequality. Moreover, if we assume that for any $m = 1, \ldots, M$, $|p(m|x)-p(m|x')|= c_L\|x-x'\|_2^{\alpha}$, then similar to the analysis of \eqref{eq::4K+1}, we can prove that
	\begin{align*}
		|p^\tau(m|x) - p^\tau(m|x')| 
		\leq (4 M + 1) c_{\alpha} \|x - x'\|_2^{\alpha} 
		= (4 M + 1) |p(m|x) - p(m|x')|,
	\end{align*}
	which yields the first inequality.
\end{proof}

\begin{proof}[Proof of Lemma \ref{lem::lipg1}]
	Lemma \ref{lem::holder} yields that for any $x,x'\in\mathcal{X}$, there holds
	\begin{align}\label{eq::holderetaPdelta}
		|p^\tau(m|x)-p^\tau(m|x')|\leq c_L\|x-x'\|_2^{\alpha}.
	\end{align} 
	Lemma \ref{lem::ln} together with \eqref{eq::holderetaPdelta} implies that if $p^\tau(m|x)\geq p^\tau(m|x')$, then we have
	\begin{align*}
		|g_m^\tau(x)-g_m^\tau(x')|&=\log  p^\tau(m|x)-\log  p^\tau(m|x')
		=\log \bigg(1+\frac{p^\tau(m|x)-p^\tau(m|x')}{p^\tau(m|x')}\bigg)
		\\
		&\leq \frac{\log \big(1+p^\tau(m|x)-p^\tau(m|x')\big)}{p^\tau(m|x')} 
		\leq \frac{\log \big(1+ c_L\|x-x'\|_2^{\alpha}\big)}{p^\tau(m|x')} .
	\end{align*}
	Otherwise, if $p^\tau(m|x)< p^\tau(m|x')$, by using Lemma \ref{lem::ln} and inequality \eqref{eq::holderetaPdelta} again, we get
	\begin{align*}
		|g_m^\tau(x)-g_m^\tau(x')|&=\log  p^\tau(m|x') - \log  p^\tau(m|x)
		=\log \bigg(1+\frac{p^\tau(m|x')-p^\tau(m|x)}{p^\tau(m|x)}\bigg)
		\\
		&\leq \frac{\log \big(1+p^\tau(m|x')-p^\tau(m|x)\big)}{p^\tau(m|x)} 
		\leq \frac{\log \big(1+ c_L\|x-x'\|_2^{\alpha}\big)}{p^\tau(m|x)}.
	\end{align*}
	Therefore, we have 
	\begin{align*}
		|g_m^\tau(x)-g_m^\tau(x')| \leq \frac{\log \big(1+ c_L\|x-x'\|_2^{\alpha}\big)}{p^\tau(m|x) \wedge p^\tau(m|x')},
	\end{align*}
	which yields the assertion.
\end{proof}

With the aid of Lemma \ref{lem::lipg1}, we are able to present the following proposition concerning $\widetilde{f}_{m}^\tau(x)$, which is crucial to establish the approximation error bound of $p_{\widetilde{f}^\tau}(m|x)$.

\begin{proposition}\label{prop::convolution}
	Let Assumptions \ref{ass:labelshift} and \ref{ass::predictor} hold. 
	Furthermore, for $\tau \in (0, 1/(2 M))$, let $p^\tau(m|x)$, $f_m^{*\tau}$ and  $\widetilde{f}_{m}^\tau$ be defined as in \eqref{eq::etaPdelta}, \eqref{eq::fdelta}, and \eqref{eq::ftildem}, respectively. 
	Moreover, let $\gamma \in (0, 2^{-1/\alpha})$
	and $\tau := \gamma^{\alpha} \in (0,1/2)$.
	Then for any $m\in [M]$, there holds $\widetilde{f}_m^\tau \in  H$ and 
	\begin{align*} 
		\bigl| \widetilde{f}_{m}^\tau(x) -  f_m^{*\tau} \bigr|
		\leq
		c_1 \gamma^{\alpha}(p^\tau(m|x)^{-1} + p^\tau(M|x)^{-1}), 
		\qquad m \in [M],
	\end{align*}
	where $c_1$ is a constant which will be specified in the proof.
\end{proposition}

\begin{proof}[Proof of Proposition \ref{prop::convolution}]
	Let the function $K : \mathbb{R}^d \to \mathbb{R}$ be defined as in \eqref{Conv}. Then for any $x \in \mathcal{X}$ and $m \in [M]$, there holds
	\begin{align*}
		K * g_m^\tau(x) 
		&=  \int_{\mathbb{R}^d}
		\biggl( \frac{2}{\gamma^2 \pi} \biggr)^{d/2} 
		\exp \biggl( - \frac{2 \|x - z\|_2^2}{\gamma^2} \biggr)  g_m^\tau(z) \, dz
		\\
		&= \int_{\mathbb{R}^d} \biggl( \frac{2}{\gamma^2 \pi} \biggr)^{d/2} 
		\exp \biggl( - \frac{2 \|h\|_2^2}{\gamma^2} \biggr) g_m^\tau(x+h) \, dh.
	\end{align*}
	Since the functions $g_m^\tau$, have a compact support and are bounded, we have $g_m^\tau\in L_2(\mathbb{R}^d)$. This together with Proposition 4.46 in \cite{steinwart2008support} yields
	\begin{align}\label{eq::KconvhiinRKHS}
		K * g_m^\tau \in  H.
	\end{align} 
	Moreover, we have
	\begin{align*}
		g_m^\tau(x) = \int_{\mathbb{R}^d} \biggl( \frac{2}{\gamma^2 \pi} \biggr)^{d/2} 
		\exp \biggl( - \frac{2 \|h\|_2^2}{\gamma^2} \biggr) g_m^\tau(x) \, dh.
	\end{align*}
	Then for any $x\in \mathcal{X}$, there holds
	\begin{align*}
		\bigl| K * g_m^\tau(x) - g_m^\tau(x) \bigr| 
		& = \biggl| 
		\int_{\mathbb{R}^d}  \biggl( \frac{2}{\gamma^2 \pi} \biggr)^{\frac{d}{2}} 
		\exp \biggl( - \frac{2\|h\|_2^2}{\gamma^2} \biggr) 
		\bigl(g_m^\tau(x+ h) - g_m^\tau(x) \bigr) \, dh \biggr|
		\\
		& \leq
		\int_{\mathbb{R}^d}  \biggl( \frac{2}{\gamma^2 \pi} \biggr)^{\frac{d}{2}} 
		\exp \biggl( - \frac{2\|h\|_2^2}{\gamma^2} \biggr) 
		\bigl| g_m^\tau(x+ h) - g_m^\tau(x) \bigr| \, dh.
	\end{align*}
	For $m \in [M]$, let 
	$A_{m,x} := \{h\in \mathbb{R}^d : g_m^\tau(x+h) \geq g_m^\tau(x)\}$.
	Using Lemma \ref{lem::lipg1} and the fact that $\log (1+x) \leq x$, $x>0$, we get
	\begin{align}\label{eq::g1deltadiff}
		\bigl| K * g_m^\tau(x) - g_m^\tau(x) \bigr| 
		& = 
		\int_{A_{m,x}}  \biggl( \frac{2}{\gamma^2 \pi} \biggr)^{\frac{d}{2}} 
		\exp \biggl( - \frac{2\|h\|_2^2}{\gamma^2} \biggr) 
		\bigl( g_m^\tau(x+ h) - g_m^\tau(x) \bigr) \, dh 
		\nonumber\\
		& \phantom{=} 
		+  \int_{\mathbb{R}^d\setminus A_{m,x}}  \biggl( \frac{2}{\gamma^2 \pi} \biggr)^{\frac{d}{2}} 
		\exp \biggl( - \frac{2\|h\|_2^2}{\gamma^2} \biggr) 
		\bigl( g_m^\tau(x) - g_m^\tau(x+h) \bigr) \, dh
		\nonumber\\
		& \leq 
		\int_{A_{m,x}}  \biggl( \frac{2}{\gamma^2 \pi} \biggr)^{\frac{d}{2}} 
		\exp \biggl( - \frac{2\|h\|_2^2}{\gamma^2} \biggr) 
		\frac{\log \big(1+ c_L\|h\|_2^{\alpha}\big)}{p^\tau(m|x)} \, dh 
		\nonumber\\
		& \phantom{=} 
		+ \int_{\mathbb{R}^d\setminus A_{m,x}}  \biggl( \frac{2}{\gamma^2 \pi} \biggr)^{\frac{d}{2}} 
		\exp \biggl( - \frac{2\|h\|_2^2}{\gamma^2} \biggr) 
		\frac{\log \big(1+ c_L\|h\|_2^{\alpha}\big)}{p^\tau(m|x+h)} \, dh
		\nonumber\\
		& \leq
		\int_{\mathbb{R}^d}  \biggl( \frac{2}{\gamma^2 \pi} \biggr)^{\frac{d}{2}} 
		\exp \biggl( - \frac{2\|h\|_2^2}{\gamma^2} \biggr) 
		\frac{c_L\|h\|_2^{\alpha}}{p^\tau(m|x)} \, dh 
		\nonumber\\
		& \phantom{=} 
		+ \int_{\mathbb{R}^d}  \biggl( \frac{2}{\gamma^2 \pi} \biggr)^{\frac{d}{2}} 
		\exp \biggl( - \frac{2\|h\|_2^2}{\gamma^2} \biggr) 
		\frac{c_L\|h\|_2^{\alpha}}{p^\tau(m|x+h)} \, dh
		=: (I) + (II).
	\end{align}
	For the first term $(I)$ in \eqref{eq::g1deltadiff}, using the rotation invariance of $x \mapsto \exp(- 2 \|x\|_2^2 / \gamma^2)$ and $\Gamma(1+c) = c \Gamma(c)$, $c > 0$, we get
	\begin{align}\label{eq::1/eta}
		(I)
		& = \frac{c_L}{p^\tau(m|x)} \biggl( \frac{\gamma}{\sqrt{2}} \biggr)^{\alpha} \int_{\mathbb{R}^d}  \biggl( \frac{1}{\pi} \biggr)^{\frac{d}{2}} \exp (- \|h\|_2^2) \|h\|_2^{\alpha} \, dh
		\nonumber\\
		& = \frac{c_L}{p^\tau(m|x)} \Big(\frac{\gamma}{\sqrt{2}}\Big)^{\alpha} \frac{2}{\Gamma(d/2)} \int_{0}^{\infty} e^{-r^2} r^{\alpha+d-1} dr
		\nonumber\\
		& = \frac{c_L}{p^\tau(m|x)} \Gamma(d/2)^{-1} \Gamma \biggl( \frac{d + \alpha}{2} \biggr) 2^{-\alpha/2} \gamma^{\alpha}.
	\end{align}
	Using \eqref{eq::holderetaPdelta} and $p^\tau(m|x+h) \geq \tau$, for any $x \in \mathcal{X}$ and $h \in \mathbb{R}^d$, we get 
	\begin{align}\label{eq::1/etax+h}
		p^\tau(m|x+h)^{-1} 
		& \leq p^\tau(m|x)^{-1} +  \bigl| p^\tau(m|x+h)^{-1} - p^\tau(m|x)^{-1} \bigr|
		\nonumber\\
		& \leq p^\tau(m|x)^{-1} +  \frac{\bigl| p^\tau(m|x+h) -  p^\tau(m|x) \bigr|}{p^\tau(m|x+h) p^\tau(m|x)}
		\leq p^\tau(m|x)^{-1} + \frac{c_L \|h\|_2^{\alpha}}{\tau p^\tau(m|x)}.
	\end{align}
	For the second term $(II)$ of \eqref{eq::g1deltadiff}, using \eqref{eq::1/eta} and \eqref{eq::1/etax+h}, we obtain
	\begin{align*}
		(II)
		& \leq c_L \int_{\mathbb{R}^d} \biggl( p^\tau(m|x)^{-1} +  \frac{c_L \|h\|_2^{\alpha}}{\tau p^\tau(m|x)} \biggr) \biggl( \frac{2}{\gamma^2 \pi} \biggr)^{\frac{d}{2}} \exp \biggl( - \frac{2 \|h\|_2^2}{\gamma^2} \biggr) \|h\|_2^{\alpha} \, dh 
		\nonumber\\
		& = \int_{\mathbb{R}^d} \frac{c_L}{p^\tau(m|x)} \biggl( \frac{2}{\gamma^2 \pi} \biggr)^{\frac{d}{2}} \exp \biggl( - \frac{2\|h\|_2^2}{\gamma^2} \biggr) \|h\|_2^{\alpha} \, dh 
		\nonumber\\
		& \phantom{=} 
		+ \int_{\mathbb{R}^d} \frac{c_L^2}{\tau p^\tau(m|x)} \biggl( \frac{2}{\gamma^2 \pi} \biggr)^{\frac{d}{2}} \exp \biggl( - \frac{2 \|h\|_2^2}{\gamma^2} \biggr) \|h\|_2^{2 \alpha} \, dh
		\nonumber\\
		& = \frac{c_L}{p^\tau(m|x)} \cdot \Gamma(d/2)^{-1} \Gamma \biggl( \frac{d + \alpha}{2} \biggr) 2^{- \alpha/2} \gamma^{\alpha}
		\nonumber\\
		& \phantom{=} 
		+ \frac{c_L^2}{\tau p^\tau(m|x)} \biggl( \frac{\gamma}{\sqrt{2}} \biggr)^{2 \alpha} \int_{\mathbb{R}^d} \pi^{- \frac{d}{2}} 
		\exp (- \|h\|_2^2) \|h\|_2^{2 \alpha} \, dh.
	\end{align*}
	The rotation invariance of $x \mapsto \exp(- 2 \|x\|_2^2 / \gamma^2)$ together with $\Gamma(1 + c) = c \Gamma(c)$, $c > 0$, yields
	\begin{align*}
		\int_{\mathbb{R}^d} \pi^{- \frac{d}{2}} \exp (- \|h\|_2^2) \|h\|_2^{2 \alpha} \, dh  
		& = \int_0^{\infty} \frac{2}{\Gamma(d/2)} \exp(- r^2) r^{2\alpha+d-1} \, dr
		\\
		& = \Gamma(d/2)^{-1} \int_0^{\infty} \exp(- r) r^{\alpha+d/2-1} \, dr 
		= \frac{\Gamma(\alpha+d/2)}{\Gamma(d/2)}
	\end{align*}
	and consequently we have
	\begin{align}\label{eq::1/eta'}
		(II)
		\leq \frac{c_L}{p^\tau(m|x)} \Gamma(d/2)^{-1} \Gamma \biggl( \frac{d + \alpha}{2} \biggr) 2^{- \alpha/2} \gamma^{\alpha} 
		+ \frac{c_L^2}{\tau p^\tau(m|x)} \biggl( \frac{\gamma}{\sqrt{2}} \biggr)^{2\alpha} \frac{\Gamma(\alpha+d/2)}{\Gamma(d/2)}.
	\end{align}
	Combining \eqref{eq::1/eta}, \eqref{eq::1/eta'} and \eqref{eq::g1deltadiff},  and taking $\tau := \gamma^{\alpha}$, we obtain
	\begin{align}\label{eq::g1deltaerror}
		& |K * g_m^\tau(x) - g_m^\tau(x)| 
		\nonumber\\
		& \leq \frac{c_L}{p^\tau(m|x)} \Gamma(d/2)^{-1} \Gamma \biggl( \frac{d + \alpha}{2} \biggr) 2^{1-\alpha/2} \gamma^{\alpha} 
		+ \frac{c_L^2}{\tau p^\tau(m|x)} \frac{\Gamma(\alpha+d/2)}{\Gamma(d/2)} \biggl( \frac{\gamma}{\sqrt{2}} \biggr)^{2\alpha}
		\nonumber\\
		&:= c_1 (p^\tau(m|x))^{-1}\gamma^{\alpha},
	\end{align}
	where $c_1 :=  2^{1-\alpha/2} c_L \Gamma(d/2)^{-1} \Gamma \bigl( \frac{d+\alpha}{2} \bigr) + 2^{-\alpha} c_L^2 \Gamma(\alpha+d/2) \Gamma(d/2)^{-1}$.
	By the definition of $\widetilde{f}_m^\tau$, we have 
	\begin{align}\label{eq::tildefdelta}
		\widetilde{f}_m^\tau 
		= K * g_m^\tau - K * g_M^\tau 
		= K * f_m^{*\tau}.
	\end{align} 
	Then by \eqref{eq::KconvhiinRKHS} and the linearity of the RKHS, we have $\widetilde{f}_m^\tau \in  H$.
	Using the triangle inequality \eqref{eq::g1deltaerror}, we obtain that for any $x\in\mathcal{X}$, there holds
	\begin{align*}
		|\widetilde{f}_m^\tau(x) - f_{m}^{*\tau}(x)| 
		& = |(K * g_m^\tau - K * g_M^\tau) - (g_m^\tau - g_M^\tau)|
		\\
		& \leq \bigl| K * g_m^\tau(x) - g_m^\tau(x) \bigr| +  \bigl| K * g_M^\tau(x) - g_M^\tau(x) \bigr| 
		\\
		& \leq c_1 \bigl( p^\tau(m|x)^{-1} + p^\tau(M|x)^{-1} \bigr) \gamma^{\alpha},
	\end{align*}
	which finishes the proof.
\end{proof}

Based on the upper bound of the pointwise distance between the score functions $\widetilde{f}_m^\tau(x)$ and $ f_{m}^{*\tau}(x)$, we derive the approximation error bound for the conditional probability estimator $p_{\widetilde{f}^\tau}(\cdot|x)$ in the following proposition.

\begin{proposition}\label{prop::etaPapprox}
	Let Assumptions \ref{ass:labelshift} and \ref{ass::predictor} hold. Furthermore, let $ H$ be the Gaussian RKHS with the bandwidth parameter $\gamma$ and $p(m|x)$ be the true predictor of the distribution $P$. Moreover, let $c_1$ be the constant as in Proposition \ref{prop::convolution}. In addition, let $\gamma \in \big(0,2^{-1/\alpha}\big)$ and $\tau:=\gamma^{\alpha}$. Finally, let $\widetilde{f}^\tau:=(\widetilde{f}_m^\tau)_{m\in [M]}$ be as in \eqref{eq::ftildem}. Then its induced estimator $p_{\widetilde{f}^\tau}(m|x)$ in \eqref{eq::pfmx} satisfies 
	\begin{enumerate}
		\item[(i)] 
		$p_{\widetilde{f}^\tau}(m|x) \geq \gamma^{\alpha} / M$;
		\item[(ii)] 
		$|p(m|x) - p_{\widetilde{f}^\tau}(m|x)| 
		\leq c_2 \gamma^{\alpha}$,
		where $c_2 := M + M e^{2 c_1} c_1 (1 + 2 c_1 e^{2 c_1})^2$.
	\end{enumerate}
\end{proposition}

\begin{proof}[Proof of Proposition \ref{prop::etaPapprox}]
	Using the definition of $p_{\widetilde{f}^\tau}(m|x)$ and \eqref{eq::gk-gK}, we get 
	\begin{align*}
		p_{\widetilde{f}^\tau}(m|x) 
		:= \frac{\exp( \widetilde{f}_m^\tau )}{\sum_{j=1}^M \exp( \widetilde{f}_m^\tau )} 
		= \frac{\exp( K * f_m^{*\tau} )}{\sum_{j=1}^M \exp( K * f_m^{*\tau} )} 
		= \frac{\exp ( K * g_m^\tau )}{\sum_{j=1}^M \exp (K * g_m^\tau)}.
	\end{align*}
	Using \eqref{eq::g1delta} and $p^\tau(m|x) \in [\tau, (1-(M-1)\tau)]$, we get $g_m^\tau \in [\log  \tau, \log  (1-(M-1)\tau)]$ and thus $K * g_m^\tau \in [\log  \tau, \log  (1-(M-1)\tau)]$. 
	Therefore, we have
	\begin{align*}
		p_{\widetilde{f}^\tau}(m|x) 
		\geq \frac{\tau}{\sum_{j=1}^M  (1-(M-1)\tau)} 
		\geq \frac{\tau}{M},
	\end{align*}
	which prove the first assertion \textit{(i)}.

	Using the definitions of $p_{\widetilde{f}^\tau}(m|x)$ and $f_m^{*\tau}$ in \eqref{eq::fdelta}, we get
	\begin{align*}
		& p^\tau(m|\cdot) - p_{\widetilde{f}^\tau}(m|\cdot) 
		= \frac{\exp(f_m^{*\tau})}{\sum_{j=1}^M  \exp(f_j^{*\tau})} - \frac{\exp(\widetilde{f}_m^\tau)}{\sum_{j=1}^M  \exp(\widetilde{f}_j^\tau)}
		\nonumber\\
		& = \frac{\exp(f_m^{*\tau}) (\sum_{j=1}^M  \exp(\widetilde{f}_j^\tau)) - \exp(\widetilde{f}_m^\tau)(\sum_{j=1}^M  \exp(f_j^{*\tau})) }{(\sum_{j=1}^M  \exp(f_j^{*\tau}))(\sum_{j=1}^M  \exp(\widetilde{f}_j^\tau))}
		\nonumber\\
		& =\frac{\exp(f_m^{*\tau}) \sum_{j=1}^M  (\exp(\widetilde{f}_j^\tau) - \exp(f_j^{*\tau})) + (\exp(f_m^{*\tau}) - \exp(\widetilde{f}_m^\tau))(\sum_{j=1}^M  \exp(f_j^{*\tau})) }{(\sum_{j=1}^M  \exp(f_j^{*\tau}))(\sum_{j=1}^M  \exp(\widetilde{f}_j^\tau))}
		\nonumber\\
		&= p_{\widetilde{f}^\tau}(m|\cdot) \bigg(\sum_{j=1}^M  (\exp(\widetilde{f}_j^\tau) - \exp(f_j^{*\tau}))p^\tau(m|\cdot) +\exp(f_m^{*\tau}) - \exp(\widetilde{f}_m^\tau)\bigg)
		\nonumber\\
		& = p_{\widetilde{f}^\tau}(m|\cdot) \bigg(\sum_{j \neq m} (\exp(\widetilde{f}_j^\tau) - \exp(f_j^{*\tau}))p^\tau(m|\cdot) 
		\nonumber\\
		& \qquad \qquad \qquad 
		+  (\exp(f_m^{*\tau}) - \exp(\widetilde{f}_m^\tau))(1-p^\tau(m|\cdot))\bigg)
		\nonumber\\
		&= p_{\widetilde{f}^\tau}(m|\cdot) \bigg(\sum_{j \neq m} (\exp(\widetilde{f}_j^\tau - f_j^{*\tau}) - 1) \exp(f_j^{*\tau}) p^\tau(m|\cdot) 
		\nonumber\\
		& \qquad \qquad \qquad 
		+ (1 - \exp(\widetilde{f}_m^\tau - f_m^{*\tau})) \exp(f_m^{*\tau}) (1 - p^\tau(m|\cdot))\bigg)
		\nonumber\\
		&= p_{\widetilde{f}^\tau}(m|\cdot) \bigg(\sum_{j \neq m} (\exp(\widetilde{f}_j^\tau - f_j^{*\tau}) - 1) \exp(f_m^{*\tau})p^\tau(j|\cdot) 
		\nonumber\\
		& \qquad \qquad \qquad 
		+  (1 - \exp(\widetilde{f}_m^\tau - f_m^{*\tau})) \exp(f_m^{*\tau})(1-p^\tau(m|\cdot))\bigg)
		\nonumber\\
		&= p_{\widetilde{f}^\tau}(m|\cdot) \bigg(\sum_{j \neq m} ( \exp(\widetilde{f}_j^\tau - f_j^{*\tau}) - \exp(\widetilde{f}_m^\tau - f_m^{*\tau}) ) \exp(f_m^{*\tau})p^\tau(j|\cdot)\bigg)
		\nonumber\\
		&= \frac{p_{\widetilde{f}^\tau}(m|\cdot)}{p^\tau(m|\cdot)} \bigg(\sum_{j \neq m} ( \exp(\widetilde{f}_j^\tau - f_j^{*\tau}) - \exp(\widetilde{f}_m^\tau - f_m^{*\tau}) ) p^\tau(m|\cdot) p^\tau(j|\cdot) \bigg)
		\nonumber\\
		&= \frac{p_{\widetilde{f}^\tau}(m|\cdot)}{p^\tau(m|\cdot)} \bigg(\sum_{j \neq m} (\exp( \widetilde{f}_j^\tau - f_j^{*\tau} - ( \widetilde{f}_m^\tau - f_m^{*\tau})) - 1) \exp(\widetilde{f}_m^\tau - f_m^{*\tau}) p^\tau(m|\cdot)p^\tau(j|\cdot)\bigg).
	\end{align*}
	By the triangle inequality, we have 
	\begin{align}\label{eq::etaPdeltadiff}
		& |p^\tau(m|\cdot) - p_{\widetilde{f}^\tau}(m|\cdot)|
		\nonumber\\
		& \leq \frac{p_{\widetilde{f}^\tau}(m|\cdot)}{p^\tau(m|\cdot)} \bigg(\sum_{j \neq m} \bigl| \exp(\widetilde{f}_j^\tau - f_j^{*\tau} - (\widetilde{f}_m^\tau - f_m^{*\tau})) - 1 \bigr| p^\tau(m|\cdot) p^\tau(j|\cdot) \biggr) \exp(\widetilde{f}_m^\tau - f_m^{*\tau}).
	\end{align}
	Using \eqref{eq::gk-gK} and $\widetilde{f}_m^\tau = K * f_m^{*\tau}$, we get
	\begin{align}\label{eq::diffkernel}
		\widetilde{f}_j^\tau-f_j^{*\tau}-(\widetilde{f}_m^\tau-f_m^{*\tau}) = K*g_j^\tau - g_j^\tau - (K*g_m^\tau - g_m^\tau).
	\end{align}
	Then by using the triangle inequality, \eqref{eq::g1deltaerror} and $\tau := \gamma^{\alpha}$, for any $k=1,\ldots, K$, we obtain
	\begin{align}
		|\widetilde{f}_j^\tau-f_j^{*\tau}-(\widetilde{f}_m^\tau-f_m^{*\tau})|  
		&\leq |K*g_j^\tau - g_j^\tau| + |K*g_m^\tau - g_m^\tau|
		\nonumber\\
		&\leq  c_1 \gamma^{\alpha}(p^\tau(j|\cdot)^{-1} + p^\tau(m|\cdot)^{-1})
		\label{eq::fdeltadiffvalue}
		\\
		&\leq 2c_1 \gamma^{\alpha}\tau^{-1} = 2c_1.
		\label{eq::jkbounded}
	\end{align}
	For any function $h_1$ and $h_2$ satisfying $|h_1 - h_2| \leq 2c_1$, if $h_1>h_2$, then by using the Lagrange mean value theorem, there exists $p \in (0,h_1-h_2)$ such that
	\begin{align}\label{eq::fdeltalarge}
		|\exp(h_1 - h_2) - 1| 
		& = \exp(h_1 - h_2) - 1 
		= e^p (h_1 - h_2) 
		\nonumber\\
		& \leq \exp(h_1 - h_2) \cdot (h_1 - h_2) 
		\leq e^{2 c_1} \cdot (h_1 - h_2).
	\end{align}
	Otherwise if $h_1<h_2$, then
	by using the Lagrange mean value theorem once again, there exists $p \in (h_1-h_2,0)$ such that
	\begin{align}\label{eq::fdeltasmall}
		|\exp(h_1 - h_2) - 1| 
		= 1 - \exp(h_1 - h_2) 
		= e^p (h_2 - h_1) 
		\leq h_2 - h_1.
	\end{align}
	Combining \eqref{eq::fdeltalarge} and \eqref{eq::fdeltasmall}, we find 
	\begin{align}\label{eq::efdeltadiff1}
		|\exp(h_1 - h_2) - 1|  
		\leq e^{2 c_1} |h_1 - h_2|.
	\end{align}
	Applying \eqref{eq::efdeltadiff1} with $h_1 := \widetilde{f}_j^\tau-f_j^{*\tau}$ and $h_2 := \widetilde{f}_m^\tau-f_m^{*\tau}$, we obtain 
	\begin{align}\label{eq::lipschitz}
		\bigl| \exp(\widetilde{f}_j^\tau - f_j^{*\tau} - (\widetilde{f}_m^\tau - f_m^{*\tau})) - 1 \bigr| 
		& \leq e^{2 c_1} |\widetilde{f}_j^\tau - f_j^{*\tau} - (\widetilde{f}_m^\tau - f_m^{*\tau})|
		\nonumber\\
		& \leq e^{2 c_1} c_1 \gamma^{\alpha} \bigl( p^\tau(j|\cdot)(x)^{-1} + p_{\widetilde{f}^\tau}(m|x)^{-1} \bigr),
	\end{align}
	where the last inequality is due to \eqref{eq::fdeltadiffvalue}.
	Similar to the analysis in \eqref{eq::diffkernel} and \eqref{eq::jkbounded}, we have
	\begin{align*}
		|\widetilde{f}_m^\tau-f_m^{*\tau}| &= |K*g_m^\tau - K*g_M^\tau - (g_m^\tau - g_M^\tau)| 
		\leq |K*g_m^\tau - g_m^\tau| + |K*g_M^\tau - g_M^\tau| 
		\nonumber\\
		&\leq  c_1 \gamma^{\alpha}(p_{\widetilde{f}^\tau}(m|x)^{-1} + p_{\widetilde{f}^\tau}(M|x)^{-1})
		\leq 2c_1 \gamma^{\alpha}\tau^{-1} = 2c_1.
	\end{align*}
	Applying \eqref{eq::efdeltadiff1} once again with $h_1 := \widetilde{f}_m^\tau$ and $h_2 := f_m^{*\tau}$, we obtain 
	$$
	|\exp(\widetilde{f}_m^\tau - f_m^{*\tau}) - 1| 
	\leq e^{2 c_1} |\widetilde{f}_m^\tau - f_m^{*\tau}|
	\leq 2 c_1 e^{2 c_1},
	$$
	which implies
	\begin{align}\label{eq::ekbounded}
		\exp(\widetilde{f}_m^\tau - f_m^{*\tau})
		\leq 1 + 2 c_1 e^{2 c_1}.
	\end{align}
	Combining \eqref{eq::etaPdeltadiff}, \eqref{eq::lipschitz} and \eqref{eq::ekbounded}, we obtain
	\begin{align}\label{eq::diffetahat}
		& |p^\tau(m|x) - p_{\widetilde{f}^\tau}(m|x)| 
		\nonumber\\
		&\leq (1 + 2c_1 e^{2 c_1}) \frac{p_{\widetilde{f}^\tau}(m|x)}{p^\tau(m|x)}
		\sum_{j \neq m} \bigl( e^{2 c_1} c_1 \gamma^{\alpha}(p^\tau(j|x)(x)^{-1} + p_{\widetilde{f}^\tau}(m|x)^{-1}) \big) p^\tau(m|x) p^\tau(j|x)
		\nonumber\\
		& = (1 + 2 c_1 e^{2 c_1}) e^{2 c_1} c_1 \gamma^{\alpha} 
		\frac{p_{\widetilde{f}^\tau}(m|x)}{p^\tau(m|x)} \sum_{j \neq m}\big(p^\tau(j|x)(x) + p_{\widetilde{f}^\tau}(m|x) \big)
		\nonumber\\
		& = M (1 + 2 c_1 e^{2 c_1}) e^{2 c_1} c_1 \gamma^{\alpha} p_{\widetilde{f}^\tau}(m|x) / p^\tau(m|x).
	\end{align}
	Thus, we show that $p_{\widetilde{f}^\tau}(m|x)/ p^\tau(m|x)$ is bounded. Similar to \eqref{eq::ekbounded}, we can also show that $\exp(f_m^{*\tau} - \widetilde{f}_m^\tau) \leq 1 + 2 c_1 e^{2 c_1}$ and consequently we have 
	\begin{align*}
		\frac{p_{\widetilde{f}^\tau}(m|x)}{p^\tau(m|x)} 
		= \frac{\sum_{m=1}^M  \exp(f_m^{*\tau})}{\sum_{m=1}^M  \exp(\widetilde{f}_m^\tau)} 
		\leq \bigvee_{m=1}^M  \exp(f_m^{*\tau} - \widetilde{f}_m^\tau)
		\leq 1 + 2 c_1 e^{2 c_1}.
	\end{align*}
	This together with \eqref{eq::diffetahat} yields $|p^\tau(m|x) - p_{\widetilde{f}^\tau}(m|x)| 
	\leq M e^{2 c_1} c_1 (1 + 2 c_1 e^{2 c_1})^2 \gamma^{\alpha}$.
	By the definition of $p^\tau(m|x)$ in \eqref{eq::etaPdelta}, we have $|p^\tau(m|x) - p(m|x)| \leq M\tau$.
	Using the triangle inequality, we then get
	\begin{align*}
		|p_{\widetilde{f}^\tau}(m|x) - p(m|x)|  
		& \leq |p_{\widetilde{f}^\tau}(m|x) - p^\tau(m|x)| + |p^\tau(m|x) - p(m|x)| 
		\\
		& \leq M (1 + 2 c_1 e^{2 c_1}) e^{2 c_1} c_1 (1 + 2 c_1 e^{2 c_1}) \gamma^{\alpha} + M\tau 
		\\
		& = \bigl(M + M (1 + 2 c_1 e^{2 c_1}) e^{2 c_1} c_1 (1 + 2 c_1 e^{2 c_1}) \bigr) \gamma^{\alpha},
	\end{align*}
	which finishes the proof.
\end{proof}

Before we present the proof of Proposition \ref{prop::approx} that presents the upper bound of the excess CE risk for $p^t_{f_0}(m|x)$ with $f_0 := \widetilde{f}^{\tau}$, we need the following proposition that gives an upper bound of the excess CE risk for any estimator $p_f(\cdot|x)$.

\begin{proposition} \label{prop::riskbound}
	Let $P$ be the probability distribution on $\mathcal{X} \times \mathcal{Y}$. 
	Moreover, let $f:=(f_m)_{m\in [M]}$ with $f_m:\mathcal{X} \to \mathbb{R}$ be the score function and its corresponding conditional probability estimator be $p_f(m|\cdot)$ as in \eqref{eq::pfmx}. Then we have
	\begin{align*}
		\mathcal{R}_{L_{\mathrm{CE}},P}(p_f(\cdot|x)) - \mathcal{R}_{L_{\mathrm{CE}},P}^*
		\leq \mathbb{E}_{X \sim p} \sum_{m=1}^M \frac{(p(m|X)-p_f(m|X))^2}{p_f(m|X)}.
	\end{align*}
\end{proposition}

\begin{proof}[Proof of Proposition \ref{prop::riskbound}]
	By the definition of $L_{\mathrm{CE}}$ and $p_f(m|\cdot)$, we have
	\begin{align*}
		\mathcal{R}_{L_{\mathrm{CE}},P}(p_f(\cdot|x))
		&=-\int_{\mathcal{X}} \sum_{m=1}^M  p(m|x)\log p_f(m|x)\,dP_X(x).
	\end{align*}
	Then we have
	$\mathcal{R}_{L_{\mathrm{CE}},P}^*= - \int_{\mathcal{X}} \sum_{m=1}^M  p(m|x)\log p(m|x) \,dP_X(x)$.
	Consequently, we obtain
	\begin{align*}	
		\mathcal{R}_{L_{\mathrm{CE}},P}(p_f(\cdot|x)) - \mathcal{R}_{L_{\mathrm{CE}},P}^* 
		= \mathbb{E}_{X \sim p} \sum_{m=1}^M  p(m|X) \log \frac{p(m|X)}{p_f(m|X)}.
	\end{align*}
	Using Lemma 2.7 in \cite{tsybakov2008introduction}, we get
	\begin{align*}
		\mathbb{E}_{X \sim p} \sum_{m=1}^M p(m|X) \log \frac{p(m|X)}{p_f(m|X)} 
		\leq \mathbb{E}_{X \sim p} \sum_{m=1}^M  \frac{(p(m|X)-p_f(m|X))^2}{p_f(m|X)},
	\end{align*}
	which finishes the proof.
\end{proof}

The following proposition is needed in deriving the approximation error bound under the small value bound assumption on $p(y|x)$.

\begin{proposition}\label{prop::excessriskupperbpund}
	Let the probability distribution $P$ satisfy Assumption \ref{ass::predictor}  \textit{(ii)}. Then for any $s \in (0,1]$ and any $\beta \geq 0$, we have          
	\begin{align*}
		\int_{\{ p(m|x) \geq s \}} \frac{1}{p(m|x)} \, dP_X(x)
		\leq  
		\begin{cases}
			c_{\beta}(1-\beta)^{-1} s^{\beta-1}, & \text{ for } 0 \leq \beta < 1;
			\\
			c_{\beta} s^{-1}, & \text{ for }  \beta \geq 1.
		\end{cases}
	\end{align*}
\end{proposition}

\begin{proof}[Proof of Proposition \ref{prop::excessriskupperbpund}]
	Since $p(m|x)$ is a probability, we have $p(m|x)\leq 1$ and consequently $C\geq 1$. For any nonnegative function $h$ and random variable $Z\sim P_Z$, there holds
	$\int h(Z) \, dP(Z)
	= E[h(Z)]
	= \int_0^{\infty} P_Z(h_Z \geq u) \, du$. 
	Hence we have
	\begin{align*}
		\int_{\{ p(m|x) \geq s \}} \frac{1}{p(m|x)} \,dP_X(x)
		& = \int_0^{\infty} P_X \biggl( \frac{\eins \{ p(m|x) \geq s \}}{p(m|x)} \geq u \biggr) \, du 
		\\
		& \leq \int_0^{1/s} P_X (p(m|x) \leq 1/u) \, du,
	\end{align*}
	where the last inequality follows from the fact that 
	$\eins\{p(m|x)\geq s\}/p(m|x) \geq u$ implies $s<p(m|x)\leq 1/u$ and $u\leq 1/s$.
	By Assumption \ref{ass::predictor} \textit{(ii)} with $0<\beta<1$, we have
	\begin{align}\label{eq::beta01}
		\int_0^{1/s} P_X \bigl( p(m|x) \leq 1/u \bigr) \, du 
		\leq c_{\beta} \int_0^{1/s} u^{-\beta} \, du 
		= \frac{c_{\beta} s^{\beta-1}}{1 - \beta}.
	\end{align}
	Since $P_X(p(m|x) \leq t) \leq 1$, we have for all $t \in [0, 1]$, 
	\begin{align*}
		\int_0^{1/s} P_X(p(m|x) \leq 1/u) \, du 
		\leq \int_0^{1/s} 1 \, du
		= 1/s 
		\leq c_{\beta} s^{-1}.
	\end{align*}
	Therefore, \eqref{eq::beta01} also holds if $\beta = 0$ and thus we obtain the first assertion.

	For $\beta > 1$, we have $P_X(p(x|k) \leq t) \leq c_{\beta} t^{\beta} \leq c_{\beta} t$, $t\in [0,1]$. If $C \leq s^{-1}$, then we have 
	\begin{align*}
		\int_0^{1/s} & P_X (p(m|x) \leq 1/u) \, du = \int_0^{c_{\beta}} P_X (p(m|x) \leq 1/u) \, du + \int_{c_{\beta}}^{1/s} P_X (p(m|x) \leq 1/u) \, du
		\\
		&\leq \int_0^{c_{\beta}} 1 \, du + \int_{c_{\beta}}^{1/s} c_{\beta}/u \, du = c_{\beta} + c_{\beta}(\log s^{-1} - \log c_{\beta})
		\leq c_{\beta}(1 + \log s^{-1}) \leq c_{\beta} s^{-1},
	\end{align*}
	where the last inequality is due to $1+\log(x^{-1})\leq x^{-1}$ for any $x\in (0,1]$.
	For $c_{\beta} \geq s^{-1}$, the integral can be upper bounded by $c_{\beta}$. Then Proposition \ref{prop::excessriskupperbpund} follows from simplifying the expressions using that $c_{\beta} \geq 1$.
\end{proof}

With all the above preparations, now we are able to establish the approximation error for $p^t_{f_0}(y|x)$ with $f_0 := \widetilde{f}^\tau$ w.r.t.~the CE loss.

\begin{proof}[Proof of Proposition \ref{prop::approx}]
	Let $p_{\widetilde{f}^\tau}(m|x) :=\exp(\widetilde{f}_m^\tau)/(\exp(\widetilde{f}_m^\tau)+1)$ and $\tau :=\gamma^{\alpha}$.
	Moreover, let $c_1$ and $c_2$ be the constants as in Propositions \ref{prop::convolution} and \ref{prop::etaPapprox}, respectively.
	By Proposition \ref{prop::riskbound}, we have
	\begin{align}\label{eq::approxexcessrisk}
		& \mathcal{R}_{L_{\mathrm{CE}},P}(p_{\widetilde{f}^\tau}(y|x)) - \mathcal{R}_{L_{\mathrm{CE}},P}^* 
		\leq \mathbb{E}_{x \sim p} \sum_{m=1}^M  \frac{(p(m|x) - p_{\widetilde{f}^\tau}(m|x))^2}{p_{\widetilde{f}^\tau}(m|x)}
		\nonumber\\
		& = \sum_{m=1}^M \mathbb{E}_{x \sim p} \biggl( \frac{(p(m|x) - p_{\widetilde{f}^\tau}(m|x))^2}{p_{\widetilde{f}^\tau}(m|x)} \cdot \eins \{ p(m|x) \leq (1/M + c_2) \gamma^{\alpha} \} \biggr)
		\nonumber\\
		& \phantom{=}
		+ \sum_{m=1}^M \mathbb{E}_{x \sim p} \bigg( \frac{(p(m|x) - p_{\widetilde{f}^\tau}(m|x))^2}{p_{\widetilde{f}^\tau}(m|x)} \cdot \eins \{ p(m|x)\geq (1/M + c_2) \gamma^{\alpha} \} \biggr)
		=: (I) + (II).
	\end{align}
	By Proposition \ref{prop::etaPapprox}, we have $|p(m|x)-p_{\widetilde{f}^\tau}(m|x)| \leq c_2 \gamma^{\alpha}$ and $p_{\widetilde{f}^\tau}(m|x) \geq \gamma^{\alpha} / M$. Thus 
	for the first term $(I)$ in \eqref{eq::approxexcessrisk}, 
	by Assumption \ref{ass::predictor} \textit{(ii)}, we have 
	\begin{align}\label{eq::approxsmall}
		(I)
		\leq \sum_{m=1}^M M c_2^2 \gamma^{\alpha} \cdot P_X \bigl( p(m|x) 
		\leq (1/M  + c_2) \gamma^{\alpha} \bigr)
		\leq M^2 c_2^2 c_{\beta} (1/M  + c_2)^{\beta} \gamma^{\alpha(1+\beta)}.
	\end{align}
	For the second term $(II)$ in \eqref{eq::approxexcessrisk}, if $p(m|x)\geq (1/M  + c_2)\gamma^{\alpha} = ( (1/M +c_2) / c_2 ) \cdot c_2 \gamma^{\alpha}$, then we have
	$p(m|x) - c_2 \gamma^{\alpha} 
	\geq p(m|x) \bigl( 1 - c_2 / (1/M +c_2) \bigr) 
	= p(m|x) / (1+Mc_2)$.
	Consequently, applying Proposition \ref{prop::etaPapprox}, we get 
	\begin{align}\label{eq::hatlowerPk}
		p_{\widetilde{f}^\tau}(m|x) 
		\geq p(m|x) - c_2 \gamma^{\alpha}
		\geq p(m|x) / (1 + M c_2).
	\end{align}
	Using Assumption \ref{ass::predictor} \textit{(ii)} and Proposition \ref{prop::excessriskupperbpund}, we get 
	\begin{align*}
		\int_{\{ p(m|x) \geq s \}} \frac{1}{p(m|x)} \, dP_X(x)
		\leq c_{\beta} \cdot \frac{s^{\beta \wedge 1-1}}{1-\beta\eins \{ \beta < 1 \}}.
	\end{align*}
	This together with \eqref{eq::hatlowerPk} and Proposition \ref{prop::etaPapprox} yields
	\begin{align}\label{eq::approxlarge}
		(II)
		& \leq \sum_{m=1}^M (1 + M c_2) c_2^2 \gamma^{2 \alpha} \int_{\{ p(m|x) \geq (1/M  + c_2) \gamma^{\alpha} \}} \frac{1}{p(m|x)} \, dP_X(x)
		\nonumber\\
		& \leq M^2 \cdot \frac{c_2^2 c_{\beta} (1/M + c_2)^{\beta \wedge 1} }{1 - \beta \eins \{ \beta < 1 \}} \cdot \gamma^{\alpha(\beta\wedge 1+1)}.
	\end{align}
	Combining \eqref{eq::approxexcessrisk}, \eqref{eq::approxsmall} and \eqref{eq::approxlarge}, we obtain
	\begin{align*}
		& \mathcal{R}_{L_{\mathrm{CE}},P}(p_{\widetilde{f}^\tau}(y|x)) - \mathcal{R}_{L_{\mathrm{CE}},P}^*
		\\
		&\leq M^2 \biggl(c_2^2 c_{\beta} (1/M +c_2)^{\beta}+ \frac{c_2^2c_{\beta}(1/M +c_2)^{\beta\wedge 1}}{1-\beta\eins \{ \beta < 1 \}}\biggr) \gamma^{\alpha(\beta\wedge 1+1)}
		\lesssim \gamma^{\alpha(\beta\wedge 1+1)}.
	\end{align*}
	This together with $f_0 := \widetilde{f}^\tau$ and $t \leq \tau$ yields the assertion.
\end{proof}

\subsubsection{Proofs Related to Section \ref{sec::sampleerror}}

Before we proceed, we need to introduce the following concept of entropy numbers \cite{vandervaart1996weak} to measure the capacity of a function set.
\begin{definition}[Entropy Numbers] \label{def::entropy numbers}
	Let $(\mathcal{X}, d)$ be a metric space, $A \subset \mathcal{X}$ and $i \geq 1$ be an integer. The $i$-th entropy number of $(A, d)$ is defined as
	\begin{align*}
		e_i(A, d) 
		= \inf \biggl\{ \varepsilon > 0 : \exists x_1, \ldots, x_{2^{i-1}} \in \mathcal{X} 
		\text{ such that } A \subset \bigcup_{j=1}^{2^{i-1}} B_d(x_j, \varepsilon) \biggr\}.
	\end{align*}
\end{definition}

The following lemma gives the upper bound of the entropy number for Gaussian kernels.

\begin{lemma}\label{lem::entropygaussian}
	Let $\mathcal{X} \subset \mathbb{R}^d$, $p(x)$ be a distribution on $\mathcal{X}$ and let $\mathrm{supp}(P_X)\subset \mathcal{X}$ be the support of $P_X$.
	Moreover, for $\gamma> 0$, let $ H(A)$ be the RKHS of the Gaussian RBF kernel $k_{\gamma}$ over the set $A$. Then, for all $N \in \mathbb{N}^*$, there exists a constant $c_{N,d} > 0$ such that
	\begin{align*}
		e_i(\mathrm{id}: H(\mathcal{X})\to L_2(P_X)) \leq 2^N c_{N,d} \gamma^{-N}i^{-\frac{N}{d}}, \qquad i>1.
	\end{align*}
\end{lemma}

\begin{proof}[Proof of Lemma \ref{lem::entropygaussian}]
	Let us consider the commutative diagram
	\begin{align*}
		\xymatrix{
			H(\mathcal{X}) \ar[rr]^{\mathrm{id}} \ar[d]_{\mathcal{I}_{\mathrm{supp}(P_X)}} & & L_2(P_X) 
			\\
			H(\mathrm{supp}(P_X)) \ar[rr]_{\mathrm{id}} & & \ell_{\infty}(\mathrm{supp}(P_X)) \ar[u]_{\mathrm{id}}
		}
	\end{align*}
	where the extension operator $\mathcal{I}_{\text{supp}(\mathcal{X})} : H_{\gamma}(\mathcal{X}) \to H_{\gamma}(\mathrm{supp}(P_X))$ 
	given by Corollary 4.43 in \cite{steinwart2008support} are isometric isomorphisms such that 
	$\|\mathcal{I}_{\mathrm{supp}(P_X)}: H_{\gamma}(\mathcal{X}) \to H_{\gamma}(\mathrm{supp}(P_X))\| = 1$.
	
	Let $\ell_{\infty}(B)$ be the space of all bounded functions on $B$. Then
	for any $f \in \ell_{\infty}(B)$, we have
	$\|f\|_{L_2(P_X)} 
	= ( \frac{1}{n}\sum_{i=1}^n |f(x_i)|^2 )^{1/2}
	\leq \|f\|_{\infty}$
	and thus
	$\|\mathrm{id} : \ell_{\infty}(\text{supp}(\mathcal{X})) \to L_2(D)\| \leq 1$. 
	This together with (A.38), (A.39) and Theorem 6.27 in \cite{steinwart2008support} implies that for all $i \geq 1$ and $N \geq 1$, there holds
	\begin{align*}
		& e_i(\mathrm{id} : H_{\gamma}(\mathcal{X}) \to L_2(p(x)))
		\\
		& \leq \|\mathcal{I}_{\mathrm{supp}(P_X)} : H_{\gamma}(\mathcal{X}) \to H_{\gamma}(\text{supp}(\mathcal{X}))\|
		\cdot e_i(\mathrm{id} :  H(\mathrm{supp}(P_X)) \to \ell_{\infty}(\mathrm{supp}(P_X)) )
		\\ &\qquad \cdot \|\mathrm{id} : \ell_{\infty}(\text{supp}(\mathcal{X})) \to L_2(P_X)\|
		\\
		& \leq  2^N c_{N,d} \gamma^{-N}i^{-\frac{N}{d}},
	\end{align*}
	where $c_{N,d}$ is the constant as in \cite[Theorem 6.27]{steinwart2008support}.
\end{proof}

Before we proceed, we need to introduce some notations. To this end, let us define $(L_{\mathrm{CE}} \circ p_f^t) (x,y) := L_{\mathrm{CE}}(y,p_f^t(x))$ and $h_{p_f^t} := L_{\mathrm{CE}} \circ p_f^t - L_{\mathrm{CE}} \circ p$. 
Similarly, for the upper part $L^u_{\mathrm{CE}}$ \eqref{eq::Llarge} and the lower part $L^l_{\mathrm{CE}}$ \eqref{eq::Lsmall} of the CE loss, we define $(L^u_{\mathrm{CE}} \circ p_f^t)(x, y) := L^u_{\mathrm{CE}}(y, p_f^t(x))$ and $(L^l_{\mathrm{CE}} \circ p_f^t)(x, y) := L^l_{\mathrm{CE}}(y, p_f^t(x))$.

Let the function space $\mathcal{F}$ be as in \eqref{def::SpaceF} and $r^* := \inf_{f \in \mathcal{F}}  (\lambda\sum_{m=1}^{M-1} \|f_m\|_{H}^2 + \mathbb{E}_P h_{p_f^t})$. For any $r\geq r^*$, we define the function space 
\begin{align*}
	\mathcal{F}_r := \biggl\{ f \in \mathcal{F} : \lambda \sum_{m=1}^{M-1} \|f_m\|_{H}^2 + \mathbb{E}_P h_{p_f^t} \leq r \biggr\}
\end{align*}
and denote the upper part of the loss difference of the functions in $\mathcal{F}_r$ as
\begin{align}\label{eq::Glr}
	\mathcal{G}^u_r := \bigl\{L^u_{\mathrm{CE}} \circ p_f^t - L^u_{\mathrm{CE}} \circ p: f\in \mathcal{F}_r\}.
\end{align}

Let $r_l^* := \inf_{f \in \mathcal{F}}  (\lambda\sum_{m=1}^{M-1} \|f_m\|_{H}^2 + \mathbb{E}_P (L^l_{\mathrm{CE}} \circ p_f^t)$. 
For any $r\geq r_l^*$, we define the function space concerning the lower part by
\begin{align*}
	\mathcal{F}^l_r := \biggl\{ f \in \mathcal{F} : \lambda \sum_{m=1}^{M-1} \|f_m\|_{H}^2 + \mathbb{E}_P (L^l_{\mathrm{CE}} \circ p_f^t) \leq r \biggr\}
\end{align*}
and denote the lower part of the loss of the functions in $\mathcal{F}^l_r$ as
\begin{align}\label{eq::Gsr}
	\mathcal{G}^l_r := \bigl\{L^l_{\mathrm{CE}} \circ p_f^t: f\in \mathcal{F}^l_r\}.
\end{align}

\begin{lemma}\label{lem::entropyGr}
	Let $\mathcal{G}^u_r$ and $\mathcal{G}^l_r$ be defined as in \eqref{eq::Glr} and \eqref{eq::Gsr}, respectively. 
	Then we have 
	\begin{align*}
		e_i (\mathcal{G}^u_r, L_2(D_p)) 
		& \leq 2c_{\xi,d} M^{1+1/(2\xi)} (r/\lambda)^{1/2} \gamma^{-d/(2\xi)} i^{-1/(2\xi)},
		\nonumber\\
		e_i (\mathcal{G}^l_r, L_2(D_p)) 
		& \leq 2c_{\xi,d} M^{1+1/(2\xi)} (r/\lambda)^{1/2} \gamma^{-d/(2\xi)} i^{-1/(2\xi)},
	\end{align*}
	where $c_{\xi,d}$ is a constant depending only on $\xi$ and $d$.
\end{lemma}

\begin{proof}[Proof of Lemma \ref{lem::entropyGr}]
	Since for any $f \in \mathcal{F}_r$, we have $\lambda \|f_m\|_{H}^2 \leq r$, $m\in [M-1]$. Therefore,
	\begin{align*}
		\mathcal{F}_r \subset \{ f \in \mathcal{F} :  f_m \in (r/\lambda)^{1/2}  B_{H}, m \in [M-1] \},
	\end{align*}
	where $B_{H}:= \{h \in  H: \|h\|_{H} \leq 1\}$ is the unit ball in the space $ H$.
	By applying Lemma \ref{lem::entropygaussian} with $\xi:=d/(2N)$, we obtain
	$e_i \big(\mathrm{id}: H(\mathcal{X})\to L_2(\mathrm{D_p})\big) 
	\leq a i^{-1/(2\xi)}$,
	where $a := c_{\xi,d} \gamma^{-d/(2\xi)}$ with the constant $c_{\xi,d}$ depending only on $\xi$ and $d$. Thus we have
	\begin{align*}
		e_i ((r/\lambda)^{1/2}  B_{H}, L_2(D_p)) 
		\leq (r/\lambda)^{1/2} a i^{-1/(2\xi)}.
	\end{align*}
	By Definition \ref{def::entropy numbers} and Lemma \ref{lem::entropygaussian}, there exists an $\epsilon:= (r/\lambda)^{1/2} a i^{-1/(2\xi)}$-net  $\mathcal{N}$ of $(r/\lambda)^{1/2}B_{H}$ w.r.t.~$L_2(D_p)$ with $|\mathcal{N}| = 2^{i-1}$. 
	Define the function set 
	\begin{align*}
		\mathcal{B} := \{ g := (g_m)_{m=1}^M: g_M = 0, g_m \in \mathcal{N}, m \in [M-1] \}.
	\end{align*} 
	Then we have $|\mathcal{B}| = 2^{(i-1)(M-1)}$. Moreover, for any function $f \in \mathcal{F}_r$, there exists a $g \in \mathcal{B}$ such that $\|f_m -g_m\|_{L_2(D_p)} \leq \epsilon$ for $m \in [M-1]$. 
	Let us define 
	\begin{align*}
		p_f(y|x) := \frac{\exp(f_y(x))}{\sum_{m=1}^M  \exp(f_m(x))}
	\end{align*} 
	and truncate $p_f(y|x)$ to obtain $p_f^t(y|x)$ as in \eqref{eq::etakdelta}. By Lemma \ref{lem::holder}, we get
	\begin{align}\label{eq::lipLcircf}
		\|L^u_{\mathrm{CE}} \circ p_f^t & - L^u_{\mathrm{CE}} \circ p - (L^u_{\mathrm{CE}} \circ p_g^t - L^u_{\mathrm{CE}} \circ p ) \|_{L_2(D_p)} 
		= \|L^u_{\mathrm{CE}} \circ p_f^t - L^u_{\mathrm{CE}} \circ p_g^t \|_{L_2(D_p)} 
		\nonumber\\
		& \leq \big\| - \log p_f^t + \log p_g^t \big\|_{L_2(D_p)} 
		\leq (4 M + 1) \big\| - \log p_f + \log p_g \big\|_{L_2(D_p)} 
		\nonumber\\
		& \leq \|f_Y(X) - g_Y(X)\|_{L_2(D_p)} + \bigg\| \log  \frac{\sum_{m=1}^M  \exp(f_m(X))}{\sum_{m=1}^M  \exp(g_m(X))} \biggr\|_{L_2(D_p)}.
	\end{align}
	For any $a > 0$ and $z \in \mathbb{R}$, the derivative function of the function $h(z) := \log (a + \exp(z))$ is $h'(z) = \exp(z) / (a + \exp(z)) \in (0,1)$. Therefore, by the Lagrange mean value theorem, we have 
	$|h(z) - h(z')| 
	= | h'(\theta z + (1-\theta) z') \cdot (z - z')| 
	\leq |z - z'|$. 
	Applying this to $a := \sum_{m=1}^{\ell-1} \exp(g_m(X)) + \sum_{m=\ell+1}^M \exp(f_m(X))$, $z = f_{\ell}(X)$ and $z' = g_{\ell}(X)$ for $\ell \in [K]$, we get
	\begin{align*}
		& \biggl| \log \frac{\sum_{m=1}^M \exp(f_m(X))}{\sum_{m=1}^M  \exp(g_m(X))} \biggr| 
		= \biggl| \sum_{\ell=1}^M \log \frac{\sum_{m=1}^{\ell-1} \exp(g_m(X)) + \sum_{m=\ell}^M \exp(f_m(X))}{\sum_{m=1}^{\ell} \exp(g_m(X)) + \sum_{m=\ell+1}^M \exp(f_m(X))} \biggr|
		\\
		& \leq \sum_{\ell=1}^M \biggl| \log \frac{\sum_{m=1}^{\ell-1} \exp(g_m(X)) + \sum_{m=\ell}^M \exp(f_m(X))}{\sum_{m=1}^{\ell} \exp(g_m(X)) + \sum_{m=\ell+1}^M \exp(f_m(X))} \biggr|
		\leq \sum_{\ell=1}^M |f_{\ell}(X) - g_{\ell}(X)|.
	\end{align*}
	This together with \eqref{eq::lipLcircf} yields
	\begin{align}\label{eq::diffGr}
		&\|L^u_{\mathrm{CE}} \circ p_f^t - L^u_{\mathrm{CE}} \circ p - (L^u_{\mathrm{CE}} \circ p_g^t - L^u_{\mathrm{CE}} \circ p) \|_{L_2(D_p)} 
		\nonumber\\
		& \leq \|f_Y(X) - g_Y(X)\|_{L_2(D_p)}  + \sum_{\ell=1}^M \|f_{\ell}(X) - g_{\ell}(X)\|_{L_2(D_p)} 
		\leq M \epsilon.
	\end{align}
	Therefore, we get 
	$\| L^u_{\mathrm{CE}} \circ p_f^t - L^u_{\mathrm{CE}} \circ p - (L^u_{\mathrm{CE}} \circ p_g^t - L^u_{\mathrm{CE}} \circ p ) \|_{L_2(D_p)} \leq M \epsilon$.
	Thus, the function set $
	\{L^u_{\mathrm{CE}} \circ p_f^t - L^u_{\mathrm{CE}} \circ p: f \in \mathcal{B}\}$
	is a $(M \epsilon)$-net of $\mathcal{G}^u_r$. Similar analysis yields that the function set $
	\{L^l_{\mathrm{CE}} \circ p_f^t: f \in \mathcal{B}\}$
	is a $(M \epsilon)$-net of $\mathcal{G}^l_r$.
	These together with (A.36) in \cite{steinwart2008support} yield
	\begin{align*}
		e_{(M-1)i}(\mathcal{G}_r^u, L_2(D_p)) 
		\leq 2M\varepsilon 
		= 2M (r/\lambda)^{1/2} a i^{-1/(2\xi)}, 
		\\
		e_{(M-1)i}(\mathcal{G}_r^l, L_2(D_p)) 
		\leq 2M\varepsilon 
		= 2M (r/\lambda)^{1/2} a i^{-1/(2\xi)},
	\end{align*}
	which are equivalent to 
	\begin{align*}
		e_i (\mathcal{G}_r^u, L_2(D_p)) 
		\leq 2M\varepsilon 
		& = 2M^{1+\frac{1}{2\xi}} (r/\lambda)^{\frac{1}{2}} a i^{-\frac{1}{2\xi}} 
		= c_{\xi,d} 2M^{1+\frac{1}{2\xi}} (r/\lambda)^{\frac{1}{2}} \gamma^{-\frac{d}{2\xi}} i^{-\frac{1}{2\xi}},
		\nonumber\\
		e_i (\mathcal{G}_r^l, L_2(D_p)) 
		\leq 2 M \varepsilon 
		& = 2 M^{1+\frac{1}{2\xi}} (r/\lambda)^{\frac{1}{2}} a i^{-\frac{1}{2\xi}} 
		= c_{\xi,d} 2M^{1+\frac{1}{2\xi}} (r/\lambda)^{\frac{1}{2}} \gamma^{-\frac{d}{2\xi}} i^{-\frac{1}{2\xi}}.
	\end{align*}
	This finishes the proof.
\end{proof}

\begin{lemma}\label{lem::variancebound}
	Let $P$ be a probability distribution on $\mathcal{X} \times \mathcal{Y}$. Let $t \in (0, 1 / (2 M))$ and $p_f^t$ be the truncation of $p_f$ as in \eqref{eq::etakdelta}. Then for any $V \geq -2 \log  t + 2$, there holds
	\begin{align*}
		\mathbb{E}_{P} \bigl( L_{\mathrm{CE}}(Y, p_f^t(\cdot|X)) - L_{\mathrm{CE}}(Y, p(\cdot|X)) \bigr)^2
		\leq V \cdot \mathbb{E}_{P} \bigl( L_{\mathrm{CE}}(Y, p_f^t(\cdot|X)) - L_{\mathrm{CE}}(Y, p(\cdot|X)) \bigr).
	\end{align*}
\end{lemma}

\begin{proof}[Proof of Lemma \ref{lem::variancebound}]
	By definition of the CE loss, we have
	\begin{align*}
		\mathbb{E}_{P}\big(L_{\mathrm{CE}}(Y,p_f^t(\cdot|X)) - L_{\mathrm{CE}}(Y,p(\cdot|X))\big)^2
		& = \mathbb{E}_{x \sim p} \sum_{m=1}^M  p(m|x) \biggl( \log \frac{p(m|x)}{p^t_f(m|x)} \biggr)^2,
		\\
		\mathbb{E}_{P}\big(L_{\mathrm{CE}}(Y,p_f^t(\cdot|X)) - L_{\mathrm{CE}}(Y,p(\cdot|X))\big)
		& = \mathbb{E}_{x \sim p} \sum_{m=1}^M  p(m|x) \biggl( \log \frac{p(m|x)}{p^t_f(m|x)} \biggr).
	\end{align*}
	For $\theta\in\mathbb{R}$, we define the function $h$ by
	\begin{align*}
		h(p^t_f(\cdot|x)) 
		& := \sum_{m=1}^M p(m|x) \biggl( \log \frac{p(m|x)}{p^t_f(m|x)} \biggr)^2  
		\nonumber\\
		& \phantom{=}
		- V \sum_{m=1}^M p(m|x) \biggl( \log \frac{p(m|x)}{p^t_f(m|x)} \biggr) + \theta  \biggl( \sum_{m=1}^M  p^t_f(m|x) - 1 \biggr).
	\end{align*}
	Then we have 
	\begin{align*}
		\frac{\partial h(p^t_f(\cdot|x))}{\partial p^t_f(m|x)}  
		&= 2 \cdot \frac{p(m|x)}{p^t_f(m|x)} \cdot \log  \frac{p^t_f(m|x)}{p(m|x)} + V \cdot \frac{p(m|x)}{p^t_f(m|x)} + \theta
		\nonumber\\
		&= \frac{p(m|x)}{p^t_f(m|x)} \cdot \biggl( - 2 \log  \frac{p(m|x)}{p^t_f(m|x)} + V \biggr) + \theta.
	\end{align*}
	Let $g(x) := x(-2\log  x + V) + \theta$ for $x \in [0, 1/t]$ and $V \geq -2 \log  t + 2$. Since the derivative of $g$ is $g'(x) = -2(\log  x+1) + V \geq 0$, the function $g(x)$ is non-decreasing w.r.t.~$x$. Therefore, the zero point of $\partial h(p^t_f(\cdot|x)) / \partial p^t_f(m|x)$ is the same for all $m \in [M]$. In other words, $p(m|x) / p^t_f(m|x)$ should be the same and thus we have $p^t_f(m|x) = p(m|x)$ due to the constraint $\sum_{m=1}^M  p^t_f(m|x) = 1$. Therefore, $p^t_f(m|x) := p(m|x)$ attains the minimum of 
	\begin{align*}
		\sum_{m=1}^M  p(m|x) \biggl( \log \frac{p(m|x)}{p^t_f(m|x)} \biggr)^2  - V \biggl( \sum_{m=1}^M  p(m|x) \biggl( \log \frac{p(m|x)}{p^t_f(m|x)} \biggr) \biggr)
	\end{align*} 
	which turns out to be zero. Consequently, for any $p^t_f(m|x)$ satisfying $\sum_{m=1}^M  p^t_f(m|x) = 1$, there holds 
	\begin{align*}
		\sum_{m=1}^M  p(m|x) \biggl( \log \frac{p(m|x)}{p^t_f(m|x)} \biggr)^2  
		\geq V \biggl( \sum_{m=1}^M  p(m|x) \biggl( \log \frac{p(m|x)}{p^t_f(m|x)} \biggr) \biggr), 
	\end{align*}
	which finishes the proof.
\end{proof}

The following lemma provides the variance bound for the lower part of the CE loss function and the upper bound for the lower part of the CE risk of the truncated estimator.

\begin{lemma}\label{lem::Epsmall}
	Let $L^l_{\mathrm{CE}}$ be the lower part of the CE loss function as in \eqref{eq::Llarge} with $\delta \in (0, 1)$. Then for any $f : \mathcal{X} \to \mathbb{R}^M$ and any $t \in (0, 1/(2M))$, we have 
	\begin{align*}
		\mathbb{E}_P (L^l_{\mathrm{CE}} \circ p_f^t)^2 
		& \leq (-\log t) \cdot \mathbb{E}_P (L^l_{\mathrm{CE}} \circ p_f^t), 
		\\
		\mathbb{E}_P (L^l_{\mathrm{CE}} \circ p_f^t) 
		& \leq -M \delta \log t.
	\end{align*}
\end{lemma}

\begin{proof}[Proof of Lemma \ref{lem::Epsmall}]
	By the definition of $L^l_{\mathrm{CE}} \circ p_f^t$, we have 
	\begin{align*}
		\mathbb{E}_P (L^l_{\mathrm{CE}} \circ p_f^t)^2 
		& = \mathbb{E}_{x \sim p} \sum_{m=1}^M p(m|x) \eins \{ p(m|x) < \delta \} (- \log p^t_f(m|x))^2,
		\nonumber\\
		\mathbb{E}_P (L^l_{\mathrm{CE}} \circ p_f^t) 
		& = \mathbb{E}_{x \sim p} \sum_{m=1}^M p(m|x) \eins \{ p(m|x) < \delta \} (- \log p^t_f(m|x)).
	\end{align*}
	Since $p^t_f(m|x) \in [t, 1)$ for any $m \in [M]$, we have $- \log p^t_f(m|x) \in (0, - \log t]$. Thus we obtain
	\begin{align*}
		\mathbb{E}_P & (L^l_{\mathrm{CE}} \circ p_f^t)^2 \leq \mathbb{E}_{x\sim p} \sum_{m=1}^M p(m|x) \eins\{p(m|x) < \delta\} (-\log t) \cdot (-\log p^t_f(m|x))
		\nonumber\\
		& \leq (-\log t) \cdot \mathbb{E}_{x\sim p} \sum_{m=1}^M p(m|x) \eins\{p(m|x) < \delta\}  (-\log p^t_f(m|x))
		= (-\log t) \mathbb{E}_P (L^l_{\mathrm{CE}} \circ p_f^t),
	\end{align*}
	which proves the first assertion. 
	Moreover, we have
	\begin{align*}
		\mathbb{E}_P (L^l_{\mathrm{CE}} \circ p_f^t) &= \mathbb{E}_{x \sim p} \sum_{m=1}^M \eins \{ p(m|x) < \delta \} p(m|x) (-\log p_f^t(m|x))
		\\
		& \leq \mathbb{E}_{x \sim p} \sum_{m=1}^M \eins \{ p(m|x) < \delta \} \delta ( - \log t) 
		\leq - M \delta \log t,
	\end{align*}
	which proves the second assertion.
\end{proof}

Before we proceed, we need to introduce another concept to measure the capacity of a function set, which is a type of expectation of superma with repect to the Rademacher sequence, see e.g., Definition 7.9 in \cite{steinwart2008support}.

\begin{definition}
	[Empirical Rademacher Average] \label{def::RademacherDefinition}
	Let $\{\varepsilon_i\}_{i=1}^m$ be a Rademacher sequence with respect to some distribution $\nu$, that is, a sequence of i.i.d.~random variables, such that $\nu(\varepsilon_i = 1) = \nu(\varepsilon_i = -1) = 1/2$. The $n$-th empirical Rademacher average of $\mathcal{F}$ is defined as
	\begin{align*}
		\mathrm{Rad}_D (\mathcal{F}, n)
		:= \mathbb{E}_{\nu} \sup_{h \in \mathcal{F}} 
		\biggl| \frac{1}{n} \sum_{i=1}^n \varepsilon_i h(x_i) \biggr|.
	\end{align*}
\end{definition}

\begin{proof}[Proof of Theorem \ref{thm::oracle}]
	Let $h_{p_f^t} := L_{\mathrm{CE}} \circ p_f^t - L_{\mathrm{CE}} \circ p$. By the definition of $L^u_{\mathrm{CE}}$ in \eqref{eq::Llarge} and $L^l_{\mathrm{CE}}$ in \eqref{eq::Lsmall}, we have $L_{\mathrm{CE}} = L^u_{\mathrm{CE}} + L^l_{\mathrm{CE}}$. Let us denote $h^u_{p_f^t} := L^u_{\mathrm{CE}} \circ p_f^t - L^u_{\mathrm{CE}} \circ p$ and $g^l_{p_f^t} := L^l_{\mathrm{CE}} \circ p_f^t$. For the sake of notation simplicity, we write $f_D := f_{D_p}$ and $L := L_{\mathrm{CE}}$. By \eqref{eq::hatetaP}, we have $\mathcal{R}_{L_{\mathrm{CE}}, P}(\widehat{p}(y|x)) = \mathcal{R}_{L, P}(p_{f_{D}}^t(y|x))$. 
	Let the empirical risk $\mathcal{R}_{L, D_p}(p_{f}^t(y|x)) := \mathbb{E}_{(x,y) \sim D_p} L_{\mathrm{CE}}(y, p_{f}^t(\cdot|x)) := n_p^{-1} \sum_{i=1}^{n_p} L_{\mathrm{CE}}(Y_i, p_{f}^t(\cdot|X_i))$.
	Then by \eqref{eq::KLR}, for any $f_0 \in \mathcal{F}$, we have $\lambda \|f_{D}\|^2_{H} + \mathcal{R}_{L, D_p}(p_{f_{D}}^t(y|x)) \leq \lambda \|f_0\|^2_{H} + \mathcal{R}_{L, D_p}(p_{f_0}^t(y|x))$ and consequently 
	\begin{align}\label{eq::newdecomp}
		&\lambda \|f_{D}\|^2_{H} + \mathcal{R}_{L, P}(p_{f_{D}}^t(y|x))-\mathcal{R}_{L, P}^* = \lambda\|f_{D}\|^2_{H} + \mathbb{E}_{p} h_{p^t_{f_D}}
		\nonumber\\
		&= \lambda\|f_{D}\|^2_{H} + \mathbb{E}_{D_p} h_{p^t_{f_D}} - \mathbb{E}_{D_p} h_{p^t_{f_D}} + \mathbb{E}_p h_{p^t_{f_D}}
		\nonumber\\
		&\leq \lambda\|f_0\|^2_{H} + \mathbb{E}_{D_p} h_{p^t_{f_0}} - \mathbb{E}_{D_p} h_{p^t_{f_D}} + \mathbb{E}_p h_{p^t_{f_D}}
		\nonumber\\
		&= \lambda\|f_0\|^2_{H} + \mathbb{E}_{p} h_{p^t_{f_0}} + \mathbb{E}_{D_p} h_{p^t_{f_0}} - \mathbb{E}_{p} h_{p^t_{f_0}} + \mathbb{E}_{p} h_{p^t_{f_D}} - \mathbb{E}_{D_p} h_{p^t_{f_D}} 
		\nonumber\\
		&= \lambda\|f_0\|^2_{H} + \mathbb{E}_{p} h_{p^t_{f_0}} + (\mathbb{E}_{D_p} h^l_{p^t_{f_0}} - \mathbb{E}_{p} h^l_{p^t_{f_0}} + \mathbb{E}_{p} h^l_{p^t_{f_D}} - \mathbb{E}_{D_p} h^l_{p^t_{f_D}})
		\nonumber\\
		&\quad + (\mathbb{E}_{D_p} h^u_{p^t_{f_0}} - \mathbb{E}_{p} h^u_{p^t_{f_0}} + \mathbb{E}_{p} h^u_{p^t_{f_D}} - \mathbb{E}_{D_p} h^u_{p^t_{f_D}})
		\nonumber\\
		&= \lambda\|f_0\|^2_{H} + \mathbb{E}_{p} h_{p^t_{f_0}} + (\mathbb{E}_{D_p} L^l_{\mathrm{CE}} \circ p^t_{f_0} - \mathbb{E}_{p} L^l_{\mathrm{CE}} \circ p^t_{f_0} + \mathbb{E}_p L^l_{\mathrm{CE}} \circ p^t_{f_D} - \mathbb{E}_{D_p} L^l_{\mathrm{CE}} \circ p^t_{f_D})
		\nonumber\\
		&\quad + (\mathbb{E}_{D_p} h^u_{p^t_{f_0}} - \mathbb{E}_{p} h^u_{p^t_{f_0}} + \mathbb{E}_p h^u_{p^t_{f_D}} - \mathbb{E}_{D_p} h^u_{p^t_{f_D}})
		\nonumber\\
		&= \lambda\|f_0\|^2_{H} + \mathbb{E}_{p} h_{p^t_{f_0}} + \mathbb{E}_{D_p} (g^l_{p^t_{f_0}} - \mathbb{E}_{p} g^l_{p^t_{f_0}})  + \mathbb{E}_{D_p} (\mathbb{E}_{p} g^l_{p^t_{f_D}} - g^l_{p^t_{f_D}})
		\nonumber\\
		&\quad + \mathbb{E}_{D_p} (h^u_{p^t_{f_0}} - \mathbb{E}_{p} h^u_{p^t_{f_0}}) + \mathbb{E}_{D_p}(\mathbb{E}_p h^u_{p^t_{f_D}} -  h^u_{p^t_{f_D}}),
	\end{align}
	where $\mathbb{E}_{D_p} h_{p^t_{f}} := n_p^{-1} \sum_{i=1}^{n_p} h_{p^t_{f}}(X_i, Y_i)$.
	In the following, we provide the estimates for the last four terms in \eqref{eq::newdecomp}.

	For any $f\in \mathcal{F}$, we observe that $\|g^l_{p^t_f} - \mathbb{E}_{p} g^l_{p^t_f}\|_{\infty} = \|L^l_{\mathrm{CE}} \circ p_f^t - \mathbb{E}_{p} L^l_{\mathrm{CE}} \circ p_f^t\|_{\infty} \leq -2\log t$. 
	By Lemma \ref{lem::Epsmall}, we have 
	\begin{align}\label{eq::vbgs}
		\mathbb{E}_{p}(g^l_{p^t_f} - \mathbb{E}_{p} g^l_{p^t_f})^2  &\leq \mathbb{E}_{p}(g^l_{p^t_f})^2 \leq (-\log t) \cdot \mathbb{E}_{p}g^l_{p^t_f}.
	\end{align}
	Applying Bernstein's inequality in 
	\cite[Theorem 6.12]{steinwart2008support}
	to $\{g^l_{p^t_f}(X_i,Y_i) - \mathbb{E}_{p} g^l_{p^t_f}: i\in [n_p]\}$, we obtain that 
	\begin{align}\label{eq::concenf0s}
		\mathbb{E}_{D_p} (g^l_{p^t_{f_0}} - \mathbb{E}_{p} g^l_{p^t_{f_0}}) &\leq \sqrt{\frac{2\zeta(-\log t)\mathbb{E}_{p} g^l_{p^t_{f_0}}}{n_p}} + \frac{4(-2\log t)\zeta}{3n_p}
		\nonumber\\
		& \leq \frac{\zeta(-\log t/2)}{n_p} + \mathbb{E}_{p} g^l_{p^t_{f_0}} + \frac{-8\zeta \log t}{3n_p} \leq \mathbb{E}_{p} g^l_{p^t_{f_0}} - \frac{10\zeta \log t}{3n_p}
	\end{align}
	holds with probability $\mathrm{P}^{n_p}$ at least $1-e^{-\zeta}$, where the last inequality is due to $2ab \leq a^2 + b^2$. To estimate the term $\mathbb{E}_{D_p}g^l_{p^t_{f_0}} - \mathbb{E}_{p} g^l_{p^t_{f_0}}$, 
	let us define the function 
	\begin{align*}
		G_{f,r} := \frac{\mathbb{E}_p g^l_{p^t_f} - g^l_{p^t_f}}{\lambda \|f\|^2_{H} + \mathbb{E}_p g^l_{p^t_f} + r}, 
		\qquad 
		f \in \mathcal{F},
		\; 
		r > r_l^*.
	\end{align*}
	where $r_l^*:= \inf\{f\in \mathcal{F}: \lambda \|f\|^2_{H} + \mathbb{E}_p g^l_{p^t_f}\}$. Then we have $\|G_{f,r}\|_{\infty} \leq -2(\log t)/r$. By \eqref{eq::vbgs}, we have 
	\begin{align*}
		\mathbb{E}_p G_{f,r}^2 \leq \frac{\mathbb{E}_{p}(g^l_{p^t_f} - \mathbb{E}_{p} g^l_{p^t_f})^2}{(\mathbb{E}_p g^l_{p^t_f} + r)^2} \leq \frac{\mathbb{E}_p (g^l_{p^t_f})^2}{2r \mathbb{E}_p g^l_{p^t_f}} \leq \frac{-\log t}{2r}.
	\end{align*}
	Let $\mathcal{G}^l_r:=\{g^l_{p^t_f}: f\in \mathcal{F}^l_r\}$ and $\mathcal{F}^l_r:= \{f \in \mathcal{F}: \lambda \|f\|^2_{H} + \mathbb{E}_p g^l_{p^t_f} \leq r\}$. Symmetrization in Proposition 7.10 of \cite{steinwart2008support} yields 
	\begin{align*}
		\mathbb{E}_{D_p\sim p^{n_p}} \sup_{f\in\mathcal{F}_r^l} |\mathbb{E}_{D_p}(\mathbb{E}_p g^l_{p^t_f} - g^l_{p^t_f})| \leq 2\mathbb{E}_{D_p\sim p^{n_p}} \mathrm{Rad}_{D_p}(\mathcal{G}^l_r, n_p) \leq 2\psi_{n_p}(r).
	\end{align*}
	For any $L^l_{\mathrm{CE}} \circ p_f^t \in \mathcal{G}^l_r$, we have $\|L^l_{\mathrm{CE}} \circ p_f^t\|_{\infty} \leq -\log t$ and $\mathbb{E}_P (L^l_{\mathrm{CE}} \circ p_f^t)^2 \leq -\log t \cdot \mathbb{E}_P (L^l_{\mathrm{CE}} \circ p_f^t) \leq -r\log t$. By applying Theorem 7.16 in \cite{steinwart2008support} and Lemma \ref{lem::entropyGr}, we obtain 
	\begin{align}\label{eq::smallRadbound}
		& \mathbb{E}_{D_p \sim p^{n_p}} \mathrm{Rad}_{D_p}(\mathcal{G}^l_r, n_p) 
		\nonumber\\
		& \leq C \big(r^{\frac{1}{2}} \lambda^{-\frac{\xi}{2}} \gamma^{-\frac{d}{2}} (-\log t)^{\frac{1-\xi}{2}}n_p^{-\frac{1}{2}} 
		\vee (r/\lambda)^{\frac{\xi}{1+\xi}} \gamma^{-\frac{d}{1+\xi}} (-\log t)^{\frac{1-\xi}{1+\xi}} n_p^{-\frac{1}{1+\xi}}\big)
		=: \psi_{n_p}(r),
	\end{align}
	where $C := C_1(\xi) c_{\xi,d}^{\xi} M^{\xi/2+1} 2^{2-\xi} \vee C_2(\xi) c_{\xi,d}^{2\xi/(1+\xi)}  2 M^{(2\xi+1)/(1+\xi)} 2^{(1-\xi)/(1+\xi)}$.
	Thus we have 
	$$
	\mathbb{E}_{D_p\sim p^{n_p}} \sup_{f\in\mathcal{F}_r^l} |\mathbb{E}_{D_p}(\mathbb{E}_p g^l_{p^t_f} - g^l_{p^t_f})| \leq 2\psi_{n_p}(r).
	$$
	It is easy to verify that $\psi_n(4r) \leq 2\psi_n(r)$. Then by applying the peeling technique in Theorem 7.7 of \cite{steinwart2008support} on $\mathcal{F}^l_r$, we obtain
	\begin{align*}
		\mathbb{E}_{D_p\sim p^{n_p}} \sup_{f\in \mathcal{F}} |\mathbb{E}_{D_p} G_{f,r}| \leq \frac{8\psi_{n_p}(r)}{r}.
	\end{align*}
	Applying Talagrand’s inequality in Theorem 7.5 of \cite{steinwart2008support} to $\gamma := 1/4$, we obtain that for any $r > r_l^*$, with probability at least $1-e^{-\zeta}$, there holds
	\begin{align*}
		\sup_{f\in\mathcal{F}} \mathbb{E}_{D_p} G_{f,r} < \frac{10\psi_{n_p}(r)}{r} + \sqrt{\frac{-\log t \zeta}{n_p r}} + \frac{-28\zeta\log t}{3n_p r}.
	\end{align*}
	By the definition of $g_{f_D, r}$, we have 
	\begin{align}\label{eq::concenfDs}
		\mathbb{E}_p g^l_{p^t_{f_D}} - \mathbb{E}_{D_p} g^l_{p^t_{f_D}} &< \big(\lambda \|f_D\|^2_{H} + \mathbb{E}_p g^l_{p^t_{f_D}}\big)\biggl(\frac{10\psi_{n_p}(r)}{r} + \sqrt{\frac{- \zeta\log t}{n_p r}} + \frac{-28\zeta\log t}{3n_p r}\biggr)
		\nonumber\\
		&\quad\quad + 10\psi_{n_p}(r)  + \sqrt{\frac{-\zeta r\log t}{n_p}} + \frac{-28\zeta\log t}{3n_p}
	\end{align}
	with probability at least $1-e^{-\zeta}$. Subsequently, we estimate the term $\mathbb{E}_{D_p} h^u_{p^t_{f_0}} - \mathbb{E}_{p} h^u_{p^t_{f_0}}$ in \eqref{eq::newdecomp}. For any $f\in \mathcal{F}$, we observe that $\|h^u_{p^t_f} - \mathbb{E}_{p} h^u_{p^t_f}\|_{\infty} \leq -2\log t - 2\log \delta = -2\log(t\delta)$. Using the variance bound in Lemma \ref{lem::variancebound} and $t \leq 1/(2M) \leq 1/e$, we get  
	\begin{align*}
		\mathbb{E}_{p}(h^u_{p^t_f} - \mathbb{E}_{p} h^u_{p^t_f})^2 \leq \mathbb{E}_{p} (h^u_{p^t_f})^2 \leq \mathbb{E}_{p} (h_{p^t_f})^2 \leq (-2\log t + 2)\cdot \mathbb{E}_{p} (h_{p^t_f})  \leq -4\log t \mathbb{E}_{p} (h_{p^t_f}).
	\end{align*}
	Then by applying Bernstein's inequality in 
	\cite[Theorem 6.12]{steinwart2008support}
	to $\{h^u_{p^t_{f_0}}(X_i,Y_i) - \mathbb{E}_{p} h^u_{p^t_{f_0}}: i\in [n_p]\}$ and $2ab \leq a^2 + b^2$, we obtain  
	\begin{align}\label{eq::concenf0l}
		\mathbb{E}_{D_p} (h^u_{p^t_{f_0}} - \mathbb{E}_{p} h^u_{p^t_{f_0}}) &\leq \sqrt{\frac{2\zeta(-4\log t)\mathbb{E}_{p} h_{p^t_{f_0}}}{n_p}} + \frac{4(-2\log(t\delta))\zeta}{3n_p}
		\nonumber\\
		& \leq \frac{-2\zeta\log t}{n_p} + \mathbb{E}_{p} h_{p^t_{f_0}} + \frac{-8\zeta \log(t\delta)}{3n_p} = \mathbb{E}_{p} h_{p^t_{f_0}} - \frac{14\zeta\log(t\delta)}{3n_p}.
	\end{align}
	To estimate the term $\mathbb{E}_p h^u_{p^t_{f_D}} - \mathbb{E}_{D_p} h^u_{p^t_{f_D}}$ in \eqref{eq::newdecomp}, we define the function 
	\begin{align*}
		H_{f,r} := \frac{\mathbb{E}_p h^u_{p^t_f} - h^u_{p^t_f}}{\lambda \|f\|^2_{H} + \mathbb{E}_p h_{p^t_f} + r}, 
		\qquad 
		f \in \mathcal{F},
		\; 
		r > r^*.
	\end{align*}
	Then we have $\|H_{f,r}\|_{\infty} \leq -2\log(t\delta)/r$
	and the variance bound in Lemma \ref{lem::variancebound} yields
	\begin{align*}
		\mathbb{E}_p H_{f,r}^2 \leq \frac{\mathbb{E}_p (h^u_{p^t_f})^2}{(\mathbb{E}_p h_{p^t_f} + r)^2} \leq \frac{\mathbb{E}_p (h_{p^t_f})^2}{2r \mathbb{E}_p h_{p^t_f}} \leq \frac{1}{r}(-\log t + 1) \leq \frac{-2\log (t\delta)}{r}.
	\end{align*}
	Let $\mathcal{G}^u_r:=\{h^u_{p^t_f}: f\in \mathcal{F}_r\}$ and $\mathcal{F}_r:= \{f \in \mathcal{F}: \lambda \sum_{m\in[M-1]} \|f_m\|^2_{H} + \mathbb{E}_p h_{p^t_f} \leq r\}$. Symmetrization in Proposition 7.10 of \cite{steinwart2008support} yields  
	\begin{align*}
		\mathbb{E}_{D_p\sim P^{n_p}} \sup_{f\in\mathcal{F}_r} |\mathbb{E}_{D_p}(\mathbb{E}_p h^u_{p^t_f} - h^u_{p^t_f})| \leq 2\mathbb{E}_{D_p\sim P^{n_p}} \mathrm{Rad}_{D_p}(\mathcal{G}^u_r, n_p) \leq 2\psi_{n_p}(r),
	\end{align*}
	where the second inequality can be proved in a similar way as in proving \eqref{eq::smallRadbound}.
	Peeling in Theorem 7.7 of \cite{steinwart2008support} together with $\mathcal{F}_r$ hence gives
	\begin{align*}
		\mathbb{E}_{D_p\sim P^{n_p}} \sup_{f\in \mathcal{F}} |\mathbb{E}_{D_p} H_{f,r}| \leq \frac{8\psi_{n_p}(r)}{r}.
	\end{align*}
	By Talagrand’s inequality in the form of Theorem 7.5 of \cite{steinwart2008support} applied to $\gamma := 1/4$, we
	therefore obtain for any $r > r^*$, 
	\begin{align*}
		\sup_{f\in\mathcal{F}} \mathbb{E}_{D_p} H_{f,r} < \frac{10\psi_{n_p}(r)}{r} + \sqrt{\frac{2(-2\log t)\zeta}{n_p r}} + \frac{-28\zeta\log (t\delta)}{3n_p r} 
	\end{align*}
	holds with probability at least $1-e^{-\zeta}$. Using the definition of $H_{f_D, r}$, we obtain 
	\begin{align}\label{eq::concenfDl}
		\mathbb{E}_P h^u_{p^t_{f_D}} - \mathbb{E}_{D_p} h^u_{p^t_{f_D}} &< \big(\lambda \|f_D\|^2_{H} + \mathbb{E}_p h_{p^t_{f_D}}\big)\biggl(\frac{10\psi_{n_p}(r)}{r} + \sqrt{\frac{2(-2\log t)\zeta}{n_p r}} + \frac{-28\zeta\log (t\delta)}{3n_p r}\biggr)
		\nonumber\\
		&\quad\quad + 10\psi_{n_p}(r)  + \sqrt{\frac{2(-2\log t)r\zeta}{n_p}} + \frac{-28\zeta\log (t\delta)}{3n_p}
	\end{align}
	with probability at least $1-e^{-\zeta}$.
	Combining \eqref{eq::newdecomp}, \eqref{eq::concenf0s}, \eqref{eq::concenfDs}, \eqref{eq::concenf0l} and \eqref{eq::concenfDl}, we obtain 
	\begin{align*}
		\lambda \|f_D\|^2_{H} &+ \mathbb{E}_p h_{p^t_{f_D}} \leq \lambda\|f_0\|^2_{H} + 2\mathbb{E}_{p} h_{p^t_{f_0}} + \mathbb{E}_{p} g^l_{p^t_{f_0}} - \frac{8\zeta\log (t\delta)}{n_p}
		\nonumber\\
		& \phantom{=} 
		+ \bigl( 2 \lambda \|f_D\|^2_{H} + \mathbb{E}_p g^l_{p^t_{f_D}} + \mathbb{E}_p h_{p^t_{f_D}} \bigr) \cdot \biggl( \frac{10\psi_{n_p}(r)}{r} + \sqrt{\frac{-4\zeta\log t}{n_pr}} + \frac{-28\zeta\log (t\delta)}{3n_p r} \biggr)
		\nonumber\\
		& \phantom{=} 
		+ 20 \psi_{n_p}(r) + 3 \sqrt{\frac{- r \zeta \log t}{n_p}} - \frac{56 \zeta \log (t\delta)}{3n_p}
	\end{align*}
	with probability at least $1-4e^{-\zeta}$.

	Now, it suffices to bound the various terms. If we take $r \geq 900 C^2 \lambda^{-\xi} \gamma^{-d} (-\log t)^{1-\xi} n_p^{-1}$, then by elementary calculation, we get $\psi_{n_p}(r) \leq r/30$. Moreover, let $r \geq -2304\zeta\log t/n_p$ and thus we get 
	\begin{align*}
		\sqrt{-\frac{r\zeta\log t}{n_p}} \leq \frac{r}{48}, 
		\qquad 
		\frac{-28\zeta\log t}{3n_p r} \leq \frac{1}{60}, 
		\qquad 
		\frac{8\zeta\log t}{n_p} \leq \frac{r}{288}. 
	\end{align*}
	By Lemma \ref{lem::Epsmall}, we have $\mathbb{E}_p g^l_{p^t_{f_0}} \leq -M \delta \log t$. Therefore, by taking $\delta := t^2$,  for any $r \geq 900 C^2 \lambda^{-\xi} \gamma^{-d} (-\log t)^{1-\xi} n_p^{-1} \vee -2304\zeta\log t/n_p \vee r^* \vee r_l^*$, we get 
	\begin{align*}
		&\lambda \|f_D\|^2_{H} + \mathbb{E}_p h_{p^t_{f_D}} \leq 2(\lambda\|f_0\|^2_{H} + \mathbb{E}_{p} h_{p^t_{f_0}}) -M \delta \log t + \frac{r}{288} + \frac{r\log \delta}{288\log t}
		\nonumber\\
		&\quad + \big(2\lambda \|f_D\|^2_{H} + \mathbb{E}_p h_{p^t_{f_D}} -M \delta \log t \big)\bigg(\frac{1}{3}  + \frac{1}{24} + \frac{1}{60} + \frac{\log \delta}{60\log t}\bigg) + \frac{2r}{3}  +  \frac{r}{16} + \frac{r}{30} + \frac{\log \delta}{30\log t}
		\nonumber\\
		&\leq 2(\lambda\|f_0\|^2_{H} + \mathbb{E}_{p} h_{p^t_{f_0}}) + \frac{49}{60}\big(\lambda \|f_D\|^2_{H} + \mathbb{E}_p h_{p^t_{f_D}}\big) - \frac{167}{120}M t^2\log t  + \frac{9}{10}r
	\end{align*}
	with probability at least $1-4e^{-\zeta}$. 
	By the definition of $r^*_l$ and $r^*$, and Lemma \ref{lem::Epsmall}, we have
	$r_l^* \leq \lambda\|f_0\|^2_{H} + \mathbb{E}_{p} g^l_{p^t_{f_0}} \leq r^* - M t^2 \log t $ and $r^* \leq \lambda\|f_0\|^2_{H} + \mathbb{E}_{p} h_{p^t_{f_0}}$.
	By some elementary calculations and taking $r:= 900 C^2 \lambda^{-\xi} \gamma^{-d} (-\log t)^{1-\xi} n_p^{-1} -2304\zeta\log t/n_p + r^* + r_l^*$, we get
	\begin{align*}
		\lambda \|f_D\|^2_{H} &+ \mathbb{E}_p h_{p^t_{f_D}}
		\leq 18(\lambda\|f_0\|^2_{H} + \mathbb{E}_{p} h_{p^t_{f_0}}) - 8 M t^2\log t  + 5r
		\nonumber\\
		&\leq 18(\lambda\|f_0\|^2_{H} + \mathbb{E}_{p} h_{p^t_{f_0}}) - 8 M t^2\log t  + 4500 C^2 \lambda^{-\xi} \gamma^{-d} (-\log t)^{1-\xi} n_p^{-1} 
		\nonumber\\
		& \phantom{=}
		- 11520 \zeta\log t/n_p + 5r^* + 5r_l^*
		\nonumber\\
		&\leq 18(\lambda\|f_0\|^2_{H} + \mathbb{E}_{p} h_{p^t_{f_0}}) + C_0 (-\log t)\cdot (t^2  + \lambda^{-\xi} \gamma^{-d}  n_p^{-1} + \zeta/n_p)
	\end{align*}
	with probability at least $1-4e^{-\zeta}$, where $C_0 := (5+8 M) \vee 4500 C^2 \vee 11520$. This proves the assertion. 
\end{proof}

\subsection{Proofs Related to Section \ref{subsec::RatesSource}} \label{sec::proofrateP}

\begin{proof}[Proof of Theorem \ref{thm::convergencerate}]
	Taking $t := \gamma^{\alpha}$ in Proposition \ref{prop::approx}, we obtain 
	\begin{align}\label{eq::approerror}
		\mathcal{R}_{L_{\mathrm{CE}},P}(p^t_{\tilde{f}})-\mathcal{R}_{L_{\mathrm{CE}},P}^* \leq c_a \gamma^{\alpha(1+\beta\wedge 1)}.
	\end{align}
	The definition of $\tilde{f}_m^t \in H$ in \eqref{eq::tildefdelta} together with Proposition 4.46 in \cite{steinwart2008support} yields
	\begin{align}\label{eq::norm}
		\|\tilde{f}_m^t\|_H^2 
		& = \| ( \log p^t(m|x) - \log (p^t(M|x))) \eins_{\mathcal{X}} \|_{L_2}^2
		\nonumber\\
		& \leq \pi^{-d/2} \gamma^{-d} \log \bigl( (1-t) / t \bigr)
		\leq \pi^{-d/2} \gamma^{-d} \log (1/t).
	\end{align}
	Then, by \eqref{eq::approerror}, \eqref{eq::norm} and applying Proposition \ref{thm::oracle} to $f_0 := \tilde{f}^t$ in \eqref{eq::tildefdelta}, we obtain
	\begin{align*}
		&\lambda \|f_{D_P}\|_{H}^2+\mathcal{R}_{L_{\mathrm{CE}},P}(\widehat{p}(y|x))-\mathcal{R}_{L_{\mathrm{CE}},P}^*
		\nonumber\\
		&\lesssim \lambda\gamma^{-d} \log (\gamma^{-\alpha}) +  \gamma^{\alpha(1+\beta\wedge 1)} 
		+ (-\log t)\cdot (t^2  + \lambda^{-\xi} \gamma^{-d}  n_p^{-1} + \zeta/n_p).
	\end{align*}
	In order to minimize the right-hand side with respect to $\gamma$ and $\lambda$, we choose $\lambda = n_p^{-1}$, $\zeta = 2\log n_p$ and 
	$\gamma = n_p^{-1/((1+\beta\wedge 1)\alpha+d)}$ and thus obtain $t = n_p^{-\alpha/((1+\beta\wedge 1)\alpha+d)}$ and 
	\begin{align*}
		&\lambda \|f_{D_P}\|_{H}^2+\mathcal{R}_{L_{\mathrm{CE}},P}(\widehat{p}(y|x))-\mathcal{R}_{L_{\mathrm{CE}},P}^*
		\nonumber\\
		&\lesssim n_p^{-\frac{(1+\beta\wedge 1)\alpha}{(1+\beta\wedge 1)\alpha+d}} \log n_p + n_p^{-\frac{2\alpha}{(1+\beta\wedge 1)\alpha+d}}\log n_p + n_p^{-\frac{(1+\beta\wedge 1)\alpha}{(1+\beta\wedge 1)\alpha+d}} n_p^{\xi} \log n_p + n_p^{-1} \log^2 n_p.
	\end{align*}
	Therefore, there exists an $N \in \mathbb{N}$ such that for any $n_p \geq N$, we have $\log n_p \leq n_p^{\xi}$ and thus we get 
	\begin{align*}
		&\lambda \|f_{D_P}\|_{H}^2+\mathcal{R}_{L_{\mathrm{CE}},P}(\widehat{p}(y|x))-\mathcal{R}_{L_{\mathrm{CE}},P}^* 
		\lesssim n_p^{-\frac{(1+\beta\wedge 1)\alpha}{(1+\beta\wedge 1)\alpha+d} + 2\xi}
	\end{align*}
	with probability $P^{n_p}$ at least $1-1/n_p$. Replacing $2\xi$ by $\xi$, we obtain the assertion. 
\end{proof}

The proof of the lower bound (Theorem \ref{thm::lower}) is based on the construction of two families of distribution $P^{\sigma}$ and $Q^{\sigma}$ as well as Proposition \ref{prop::lower} \cite[Theorem 2.5]{tsybakov2008introduction} and the Varshamov-Gilbert bound in Lemma \ref{lem::VGbound} \cite{varshamov1957estimate}.

\begin{proposition}
	\label{prop::lower}
	Let $\{\Pi_h\}_{h\in H}$ be a family of distributions indexed over a subset $ H$ of a semi-metric $(\mathcal{F},\rho)$. 
	Assume that there exist $h_0, \ldots, h_L \in  H$ such that for some $L \geq 2$, 
	\begin{itemize}
		\item[(i)] 
		$\rho(h_j, h_i) \geq 2s > 0$ for all $0 \leq i < j \leq L$; 
		\item[(ii)] 
		$\Pi_{h_j} \ll \Pi_{h_0}$ for all $j \in [L]$;
		\item[(iii)] 
		the average KL divergence to $\Pi_{h_0}$ satisfies 
		$\frac{1}{L} \sum_{j=1}^L \mathrm{KL}(\Pi_{h_j}, \Pi_{h_0}) \leq \kappa \log L$
		for some $\kappa \in (0, 1/8)$.
	\end{itemize}
	Let $Z\sim \Pi_h$, and let $\widehat{h} : Z \to \mathcal{F}$ denote any improper learner of $h\in H$. Then we have 
	\begin{align*}
		\sup_{h\in H} \Pi_h \bigl( \rho(\widehat{h}(Z), h) \geq s \bigr) 
		\geq \bigl( \sqrt{L}/(1+\sqrt{L}) \bigr) \bigl( 1 - 2 \kappa - 2 \kappa / \log L \bigr) 
		\geq (3-2\sqrt{2})/8.
	\end{align*}
\end{proposition}

\begin{lemma}[Varshamov-Gilbert Bound] \label{lem::VGbound}
	Let $\ell \geq 8$ and $L \geq 2^{\ell/8}$.
	For all $0 \leq i < j \leq L$, let 
	$\overline{\rho}_H(\sigma^i,\sigma^j) := \#\{\ell \in [L] : \sigma^i_{\ell}\neq\sigma^j_{\ell}\}$ be the Hamming distance.
	Then there exists a subset $\{\sigma^0, \ldots, \sigma^L\}$ of $\{ -1, 1\}^{\ell}$ such that $
	\overline{\rho}_H(\sigma^i, \sigma^j) \geq \ell/8$,
	where $\sigma^0 := (1,\ldots, 1)$.
\end{lemma}

\begin{proof}[Proof of Theorem \ref{thm::lowerKLR}]
	Without loss of generality, we investigate the binary classification, i.e., $M=2$. Let the input space $\mathcal{X} := [0,1]^d$ and the output space as $\mathcal{Y} = \{-1,1\}$. Define $r:= c_r n_p^{-1/((\beta\wedge 1+1)\alpha+d)}$ with the constant $c_r>0$ to be determined later.
	In the unit cube $\mathcal{X}$, we find a grid of points with radius parameter $r$,
	\begin{align*}
		\mathcal{G} := \{(2k_1 r, 2k_2 r, \ldots, 2k_d r) : k_i = 1, 2, \ldots, (2r)^{-1}-1, i=1, 2, \ldots, d\}.
	\end{align*}
	Denote $\ell := |\mathcal{G}| = ((2r)^{-1} -1)^d$ and $\mathcal{G} = \{x_i\}_{i=1}^\ell$.  
	Without loss of generality, we let $(6r)^{-1}-1/3$ be an integer. 
	Define the set of grid points
	$\mathcal{G}_1 := \{(2k_1 r, 2k_2 r, \ldots, 2k_d r) : k_i = 1, \ldots, (6r)^{-1}-1/3, i \in [d]\} \subset \mathcal{G}$ and $\mathcal{G}_2 := \{(2k_1 r, 2k_2 r, \ldots, 2k_d r) : k_i = (r^{-1}-2)/3, \ldots, (2r)^{-1}-1, i \in [d]\} \subset \mathcal{G}$. Then we have $|\mathcal{G}_1| = |\mathcal{G}_2| =((6r)^{-1}-1/3)^d = 3^{-d} \ell$.

	\textit{Construction of the Conditional Probability Distribution $p(y|x)$}. 
	Since we consider the binary classification case $\mathcal{Y} = \{-1,1\}$, we denote the conditional probability of the positive class as $p(1|x) := p(y=1|x)$ and the nagative class as $p(-1|x) := 1 - p(1|x)$.
	Let the function $g_r(\cdot)$ on $[0,\infty]$ be defined by 
	\begin{equation*}
		g_r(z) := 
		\begin{cases}
			1 - z/r & \text{ if } 0 \leq z < r,
			\\
			0 & \text{ if } z > r.
		\end{cases}
	\end{equation*}
	Moreover, let $a_U := 1/3 + r/3$ and $a_L := 2/3 - 7r/3$, which are close to $1/3$ and $2/3$, respectively. 
	Given $\sigma \in \{-1,1\}^\ell$ and $c_{\alpha}>0$, we define 
	\begin{align*}
		p^{\sigma}(1|x) := 
		\begin{cases}
			1 -  c_{\alpha} r^{\alpha} + c_{\alpha} \eins_{\{\sigma_i=1\}} r^{\alpha} g_r^{\alpha}(\|x-x_i\|_2) & \text{ if } x \in \bigcup_{x_i \in \mathcal{G}_2}B(x_i, r),
			\\
			1 -  c_{\alpha} r^{\alpha} & \text{ if } x \in [a_L,1]^d \setminus \bigcup_{x_i \in \mathcal{G}_2} B(x_i, r),
			\\
			1/2 + c_{\alpha} \eins_{\{\sigma_i=1\}} r^{\alpha} g_r^{\alpha}(\|x-x_i\|_2) & \text{ if } x \in \bigcup_{x_i \in \mathcal{G}_1}B(x_i, r),
			\\
			1/2 & \text{ if } x \in [0, a_{U}]^d \setminus \bigcup_{x_i \in \mathcal{G}_1}B(x_i, r),
			\\
			\in [1/2, 1 -  c_{\alpha} r^{\alpha}] &  \text{ otherwise}.
		\end{cases}
	\end{align*}

	\textit{Construction of the Marginal Distribution $p(x)$}.
	First, we define the marginal density function $p(x)$ by
	\begin{align*}
		p(x) := 
		\begin{cases}
			r^{d+(\alpha-1)\beta} \|x-x_i\|_2^{\beta-d} & \text{ if } x \in \bigcup_{x_i \in \mathcal{G}_2}B(x_i, r)\setminus \{x_i\},
			\\
			\frac{1 - \sum_{x_i \in \mathcal{G}_2} P(B(x_i,r))}{\sum_{x_i \in \mathcal{G}_1}\mu(B(x_i,r))} & \text{ if } x \in \bigcup_{x_i \in \mathcal{G}_1}B(x_i, r),
			\\
			0 & \text{ otherwise}.
		\end{cases}
	\end{align*}
	Let us verify that $p$ is a density function by proving $\int_{\mathcal{X}} p(x)\, dx = 1$. To be specific, 
	\begin{align*}
		\int_{\mathcal{X}} p(x)\, dx 
		& = \int_{\bigcup_{x_i \in \mathcal{G}_2}B(x_i, r)} p(x) \,dx + \int_{\bigcup_{x_i \in \mathcal{G}_1} B(x_i, r)} p(x) \,dx
		\\
		&= |\mathcal{G}_2| \cdot P(B(x_1,r)) + 3^d\big(1 - 3^{-d} \ell P(B(x_1,r))\big) \cdot 3^{-d}
		\\
		&= 3^{-d} \ell \cdot P(B(x_1,r)) + 1 - 3^{-d} \ell P(B(x_1,r)) 
		= 1,
	\end{align*}
	where $x_1 \in \mathcal{G}_2$. 
	Finally, for any $\sigma^j \in \{-1,1\}^{\ell}$, we write $P_X^{\sigma^j} := P_X$.

	\textit{Verification of the H\"older Smoothness}.
	First, $g_r$ satisfies the Lipschitz continuity with $|g(x) - g(x')| \leq r^{-1} |x-x'|$. 
	Moreover, using the inequality $|a^{\alpha} - b^{\alpha}| \leq |a-b|^{\alpha}$, $\alpha\in (0,1)$, we obtain that
	for any $x,x' \in B(x_i, r)$, $x_i \in \mathcal{G}_1 \cup \mathcal{G}_2$, 
	there holds
	\begin{align*}
		|p^{\sigma}(1|x) - p^{\sigma}(1|x')| 
		& = c_{\alpha} r^{\alpha} \big|g_r^{\alpha}(\|x-x_i\|_2) - g_r^{\alpha}(\|x'-x_i\|_2) \big|
		\\
		& \leq c_{\alpha} r^{\alpha} \big|g_r(\|x-x_i\|_2) - g_r(\|x'-x_i\|_2)\big|^{\alpha}
		\\
		& \leq c_{\alpha} r^{\alpha} \big| \|x-x_i\|_2/r  - \|x'-x_i\|_2/r \big|^{\alpha}
		\leq c_{\alpha} \|x'-x\|_2^{\alpha}.
	\end{align*}
	Therefore, $p^{\sigma}(y|x)$ satisfies the H\"older smoothness assumption.

	\textit{Verification of the Small Value Bound Condition.} Using the inequality $1-(1-x)^{1/\alpha} \leq 1-(1-\alpha^{-1} x) = \alpha^{-1} x$ for any $x \in (0,1)$ and $\alpha\in (0,1)$, we obtain that for any $0< t \leq c_{\alpha} r^{\alpha}$, 
	\begin{align*}
		& P_X \bigl(p^{\sigma}(1|X) \geq 1-t \bigr) 
		\\
		&= \sum_{x_i \in\mathcal{G}_2} P_X \bigl(\{x\in B(x_i,r) : 1-c_{\alpha} r^{\alpha} + c_{\alpha} r^{\alpha} \eins \{ \sigma_i = 1 \} g_r^{\alpha}(\|x - x_i\|_2) \geq 1-t\} \bigr)
		\\
		& \leq |\mathcal{G}_2| \cdot P_X \bigl(\{x\in B(x_1,r) :1-c_{\alpha} r^{\alpha} + c_{\alpha} r^{\alpha} g_r^{\alpha}(\|x - x_1\|_2) \geq 1-t\} \bigr)
		\\
		& =3^{-d} \ell \cdot P_X \bigl(\{x\in B(x_1,r) :r^{\alpha} - (r-\|x - x_1\|_2)^{\alpha} \leq c_{\alpha}^{-1} t\} \bigr)
		\\
		&= 3^{-d} \ell \cdot P_X \bigl(\{x\in B(x_1,r) :\|x - x_1\|_2 \leq r\bigl(1-(1-c_{\alpha}^{-1}r^{-\alpha} t)^{1/\alpha} \bigr)\} \big)
		\\
		&\leq 3^{-d} \ell \cdot P_X \bigl(\{x\in B(x_1,r) :\|x - x_1\|_2 \leq \alpha^{-1}c_{\alpha}^{-1}r^{1-\alpha} t\} \bigr)
		\\
		& = 3^{-d} \ell \cdot P_X \bigl(B(x_1, \alpha^{-1}c_{\alpha}^{-1}r^{1-\alpha} t)\bigr)
		= 3^{-d} \ell \int_{B(x_1, \alpha^{-1}c_{\alpha}^{-1}r^{1-\alpha} t)} p(x) \, dx
		\\
		& = \frac{2\pi^{d/2}\ell r^{d+(\alpha-1)\beta}}{3^d \Gamma(d/2)} \int_{0}^{\alpha^{-1}c_{\alpha}^{-1}r^{1-\alpha} t} \rho^{d-1} \rho^{\beta-d} \,d\rho
		\\
		& = \frac{2\pi^{d/2}}{3^d\Gamma(d/2)\beta(c_{\alpha}\alpha)^{\beta}} \ell r^d t^{\beta}\leq \frac{2\pi^{d/2}}{6^d\Gamma(d/2)\beta(c_{\alpha}\alpha)^{\beta}} t^{\beta}.
	\end{align*}
	Choosing $c_{\beta} \geq 2\pi^{d/2}/(\Gamma(d/2)\beta(c_{\alpha}\alpha)^{\beta} 6^{d}) \vee 1$, the $\beta$-small value bound is satisfied.

	\textit{Verification of the Conditions in Proposition \ref{prop::lower}}.
	Let $L = 2^\ell-1$. For the sake of convenience, for any $\sigma^j \in \{-1,1\}^{\ell}$, $j=0, \ldots, L$, we write $P^j := P^{\sigma^j}$ and $Q^j := Q^{\sigma^j}$.
	Denote $\sigma^0:= (-1,\ldots,-1)$ and $P^0 = P^{\sigma^0}$.
	Define the full sample distribution by 
	\begin{align}\label{eq::Pj}
		\Pi_{j} := P^{j\otimes n_p},
		\qquad 
		j = 0, \ldots, L.
	\end{align} 
	Moreover, we define the semi-metric $\rho$ in Proposition \ref{prop::lower} by
	\begin{align*}
		\rho(p^i(\cdot|x), p^j(\cdot|x)) 
		&:= \int_{\mathcal{X}} \bigg(p^i(1|x)\log \frac{p^i(1|x)}{p^j(1|x)} + p^i(-1|x)\log \frac{p^i(-1|x)}{p^j(-1|x)}\bigg) p(x) \, dx = \mathrm{KL}(P^i, P^j) . 
	\end{align*}
	Therefore, for any predictor $\widehat{p}(y|x)$, we have $\mathcal{R}_{L_{\mathrm{CE}},P}(\widehat{p}(y|x)) - \mathcal{R}_{L_{\mathrm{CE}},P}^* = \rho(p(\cdot|x),\widehat{p}(\cdot|x))$.
	Now, we verify the first condition in Proposition \ref{prop::lower}. 
	For sufficient large $n_p$, we have $r^{\alpha\beta}/(6^d \beta) \leq 1/2$. For any $x \in \bigcup_{x_k \in \mathcal{G}_1} B(x_k,r)$, there holds 
	\begin{align*}
		p(x) =  \frac{1 - \sum_{x_i \in \mathcal{G}_2} P(B(x_i,r))}{\sum_{x_i \in \mathcal{G}_1}\mu(B(x_i,r))} = \frac{1 - \beta^{-1} 3^{-d} \ell r^{d+\alpha\beta}}{3^{-d} \ell \pi^{d/2} r^d/\Gamma(d/2+1)} \leq \frac{6^d\Gamma(d/2+1)}{2\pi^{d/2}}.
	\end{align*}
	Denote the Hellinger distance between $P^i$ and $P^j$ as $H(P^i, P^j):= \int (\sqrt{dP^i} - \sqrt{dP^j})^2$. Using the inequality $\mathrm{KL}(P^i, P^j) \geq 2H^2(P^i, P^j)$, $\sqrt{a}- \sqrt{b} = (a-b)/(\sqrt{a}+\sqrt{b})$ and Lemma \ref{lem::VGbound}, we obtain that
	for any $0 \leq i < j \leq L$, there holds
	\begin{align*}
		& \rho \bigl( p^i(\cdot|x), p^j(\cdot|x) \bigr) 
		= \mathrm{KL}(P^i, P^j) 
		\geq 2 H^2(P^i, P^j)
		\nonumber\\
		& = 2 \int_{\mathcal{X}} \Bigl( \bigl( p^i(1|x)^{\frac{1}{2}} - p^j(1|x)^{\frac{1}{2}} \bigr)^2 + \bigl( p^i(-1|x)^{\frac{1}{2}} - p^j(-1|x)^{\frac{1}{2}} \bigr)^2 \Bigr) p(x) \, dx 
		\nonumber\\
		& = 2 \int_{\mathcal{X}} \bigl( p^j(1|x) - p^i(1|x) \bigr)^2 \Bigl( \bigl( p^j(1|x)^{\frac{1}{2}} 
		+ p^i(1|x)^{\frac{1}{2}} \bigr)^{-2} 
		+ \bigl( p^j(-1|x)^{\frac{1}{2}} + p^i(-1|x)^{\frac{1}{2}} \bigr)^{-2} \Bigr) p(x) \, dx 
		\nonumber\\
		& \geq \int_{\bigcup_{x_k \in \mathcal{G}_1} B(x_k, r)} \bigl( p^j(1|x) -  p^i(1|x) \bigr)^2 \bigl( p^j(1|x) \vee p^i(1|x) \bigr)^{-1} p(x) \, dx 
		\nonumber\\
		& \phantom{=} 
		+ \int_{\bigcup_{x_k \in \mathcal{G}_2} B(x_k, r)} \bigl( p^j(1|x) -  p^i(1|x) \bigr)^2 \bigl( p^j(-1|x) \vee p^i(-1|x) \bigr)^{-1} p(x) \, dx 
		\nonumber\\
		& \geq 2 \rho_H(\sigma^i, \sigma^j) \int_{B(x_1,r)} (c_{\alpha}(r-\|x-x_1\|_2)^{\alpha})^2 \cdot \frac{6^d\Gamma(d/2+1)}{2\pi^{d/2}} \, dx
		\nonumber\\ 
		& \phantom{=} 
		+ \rho_H(\sigma^i, \sigma^j) \int_{B(x_1,r)} (c_{\alpha}(r-\|x-x_1\|_2)^{\alpha})^2 (c_{\alpha} r^{\alpha})^{-1} \cdot r^{d+(\alpha-1)\beta} \|x-x_1\|_2^{\beta-d}  \, dx
		\nonumber\\
		&\geq \rho_H(\sigma^i, \sigma^j) \bigg(6^dc_{\alpha}^2 d\int_0^r (1-\rho)^{2\alpha} \rho^{d-1} \, d\rho + \frac{2\pi^{d/2}c_{\alpha}}{\Gamma(d/2)} r^{d+(\alpha-1)\beta-\alpha} \int_0^{r} (r - \rho)^{2\alpha} \rho^{\beta-d} \rho^{d-1} \, d\rho\bigg)
		\nonumber\\
		&= \rho_H(\sigma^i, \sigma^j) \bigg(6^dc_{\alpha}^2 d r^{2\alpha+d} \int_0^1 (1-t)^{2\alpha} t^{d-1} \,dt + \frac{2\pi^{d/2}c_{\alpha}}{\Gamma(d/2)}  r^{d+(\beta+1)\alpha}\int_0^{1} (1 - t)^{2\alpha} t^{\beta-1} \, d t\bigg)
		\nonumber\\
		&\geq \frac{\ell}{8} \bigg(6^dc_{\alpha}^2 d r^{2\alpha+d} \mathrm{Beta}(2\alpha+1,d) + \frac{2\pi^{d/2}c_{\alpha}}{\Gamma(d/2)}  \mathrm{Beta}(2\alpha+1,\beta) r^{d+\alpha(1+\beta)}\bigg)
		\nonumber\\
		&\geq 2^{-d-3} \bigg(6^dc_{\alpha}^2 d r^{2\alpha} \mathrm{Beta}(2\alpha+1,d) + \frac{2\pi^{d/2}c_{\alpha}}{\Gamma(d/2)}  \mathrm{Beta}(2\alpha+1,\beta) r^{\alpha(1+\beta)}\bigg)\geq C_4 r^{\alpha(1+\beta\wedge 1)},
	\end{align*}
	where $C_4 := 2^{-d-3} \big(6^dc_{\alpha}^2 d \mathrm{Beta}(2\alpha+1,d) \wedge 2 \pi^{d/2} c_{\alpha}\Gamma(d/2)^{-1}  \mathrm{Beta}(2\alpha+1,\beta)\big)$. 
	By taking
	\begin{align*}
		s:= 2^{-1} C_4 r^{\alpha(1+\beta\wedge 1)} = 2^{-1} C_4 n_p^{-\frac{(1+\beta \wedge 1)\alpha}{(1+\beta \wedge 1)\alpha+d}},
	\end{align*} 
	we obtain $\rho(p^i(\cdot|x), p^j(\cdot|x)) \geq 2s$.
	The second condition of Proposition \ref{prop::lower} holds obviously. 
	Therefore, it suffices to verify the third condition in Proposition \ref{prop::lower}, which requires to consider the KL divergence between $P^j$ and $P^0$. Using Lemma 2.7 in \cite{tsybakov2008introduction} and $1-c_{\alpha} r^{\alpha} \geq 7/8$, we get
	\begin{align}
		& \mathrm{KL} (P^j, P^0) 
		\leq \int_{\mathcal{X}} \frac{(p^{j}(1|x) - p^{0}(1|x))^2}{p^{0}(1|x)p^{0}(-1|x)} p(x) \, dx
		\nonumber\\
		& \leq \sum_{x_k \in \mathcal{G}_1} \eins \{ \sigma^j_k = 1 \} \int_{B(x_k,r)} \frac{(r-\|x-x_k\|_2)^{2\alpha}}{(1/2+c_{\alpha}r^{\alpha})(1/2-c_{\alpha}r^{\alpha})}\cdot \frac{\Gamma(d/2+1)3^d}{\pi^{d/2}}(m r^d)^{-1} \, dx
		\nonumber\\
		& \phantom{=} 
		+ \sum_{x_k \in \mathcal{G}_2} \eins \{ \sigma^j_k = 1 \} \int_{B(x_k,r)} \frac{(r-\|x-x_k\|_2)^{2\alpha}}{c_{\alpha}r^{\alpha}(1-c_{\alpha}r^{\alpha})}\cdot r^{d+(\alpha-1)\beta} \|x-x_k\|_2^{\beta-d}  \, dx
		\nonumber\\
		&\leq \frac{5\Gamma(d/2+1)3^d}{\pi^{d/2} r^d}\int_{B(x_k,r)} (r-\|x-x_k\|_2)^{2\alpha} \, dx
		\nonumber\\
		& \phantom{=} 
		+ \frac{m r^{d+(\alpha-1)\beta}}{c_{\alpha}r^{\alpha}(1-c_{\alpha}r^{\alpha})}
		\int_{B(x_k,r)} (r-\|x-x_k\|_2)^{2\alpha}\cdot  \|x-x_k\|_2^{\beta-d}  \, dx
		\nonumber\\
		&\leq \frac{5d 3^d}{r^d}\int_{0}^r (r-\rho)^{2\alpha} \rho^{d-1}\, dx + \frac{2\pi^{d/2}}{\Gamma(d/2)}\frac{m r^{d+(\alpha-1)\beta}}{c_{\alpha}r^{\alpha}(1-c_{\alpha}r^{\alpha})}\int_{0}^r (r-\rho)^{2\alpha}\cdot \rho^{\beta-d} \rho^{d-1}  \, dx
		\nonumber\\
		& = 5 d 3^d r^{2\alpha} \int_0^1 (1-t)^{2\alpha} t^{d-1} \, dt + \frac{2\pi^{d/2}c_{\alpha} \cdot m r^{d+\alpha(1+\beta)}}{\Gamma(d/2)(1-c_{\alpha}r^{\alpha})} 
		\int_0^1 (1-t)^{2\alpha} \cdot t^{\beta-1} \, dt
		\nonumber\\
		& \leq 5 d 3^d \mathrm{Beta}(2\alpha+1, d) r^{2\alpha} + \frac{2 \pi^{d/2} c_{\alpha} \cdot \mathrm{Beta}(2\alpha+1,\beta)}{2^d \Gamma(d/2) (1-c_{\alpha}r^{\alpha})} 
		\cdot r^{\alpha(1+\beta)}
		\leq C_3 r^{\alpha(1+\beta\wedge 1)},
		\label{eq::KLuppbound}
	\end{align}
	where $C_3 := 5d 3^d \mathrm{Beta}(2\alpha+1, d) + 4\pi^{d/2}c_{\alpha}\mathrm{Beta}(2\alpha+1,\beta)/(2^d \Gamma(d/2))$. By the independence of samples and \eqref{eq::KLuppbound}, we have 
	for any $j\in \{0,1,\ldots, L\}$, 
	\begin{align*}
		\mathrm{KL}(\Pi_{j}, \Pi_{0}) 
		& = n_p \mathrm{KL}(P^j, P^0) \leq  C_3 n_p r^{(1+\beta\wedge 1)\alpha}
		\\
		& = C_3 c_r^{(1+\beta\wedge 1)\alpha+d} r^{-d} \leq C_3 c_r^{(1+\beta\wedge 1)\alpha+d} 4^d \ell 
		\leq 2(\log 2)^{-1} C_3 c_r^{(1+\beta\wedge 1)\alpha+d} 4^d \log L.
	\end{align*}
	By choosing a sufficient small $c_r$ such that $2(\log 2)^{-1} C_3 c_r^{(1+\beta\wedge 1)\alpha+d} 4^d = 1/16$, we verify the third condition. 
	Applying Proposition \ref{prop::lower}, 
	we obtain that for any estimator $\widehat{p}(y|x)$ built on $D_p$, 
	with probability $P^{n_p}$ at least 
	$(3-2\sqrt{2})/8$, there holds
	\begin{align*}
		\sup_{P \in \mathcal{P}}  \mathcal{R}_{L_{\mathrm{CE}},P}(\widehat{p}(y|x)) - \mathcal{R}_{L_{\mathrm{CE}},P}^* \geq (C_4/2) \cdot n_p^{-\frac{(1+\beta \wedge 1)\alpha}{(1+\beta \wedge 1)\alpha+d}},
	\end{align*}
	which finishes the proof. 
\end{proof}

\subsection{Proofs Related to Section \ref{sec::priorestmation}} \label{sec::proofprior}

\begin{proof}[Proof of Theorem \ref{lem::hold}]
	``$\Rightarrow$'' Suppose that the equation system \eqref{eq::equationsystem} holds for some weight $w:=(w(y))_{y\in [K]}$. Then \eqref{eq::equationsystem} together with Bayes' formula yields
	\begin{align*}
		p(y) &= \mathbb{E}_{x \sim q} \frac{p(y | x)}{\sum_{m=1}^M w(m) p(m|x)}
		= \mathbb{E}_{x \sim q} \frac{p(x|y) p(y) / p(x)}{\sum_{m=1}^M w(m) p(x|m) p(m) / p(x)}
		\\
		&= \mathbb{E}_{x \sim q}  \frac{p(y) p(x|y)}{\sum_{m=1}^M w(m)p(m) p(x|m)}.
	\end{align*}
	By Assumption \ref{ass::marginal}  \textit{(i)}, we have $p(y) > 0$. 
	Dividing both sides of the above equation by $p(y)$, we get
	\begin{align*}
		\mathbb{E}_{x \sim q} \frac{p(x|y)}{\sum_{m=1}^M w(m) p(x|m) p(m)} = 1,
	\end{align*}
	which is equivalent to
	\begin{align} \label{eq::equPXY} 
		\int_{\mathcal{X}} \frac{p(x|y)}{\sum_{m=1}^M w(m) p(m) p(x|m)} \cdot q(x)\, dx 
		= \int_{\mathcal{X}} p(x|y) \, dx.
	\end{align}
	Using the law of total probability and $p(x|m) = q(x|m)$ from Assumption \ref{ass:labelshift}, we get
	\begin{align}\label{eq::two}
		q(x) = \sum_{m=1}^K q(k) q(x|k) = \sum_{k=1}^K q(k) p(x|k).
	\end{align}
	Plugging \eqref{eq::two} into \eqref{eq::equPXY}, we obtain
	\begin{align*}
		\int_{\mathcal{X}} p(x|y) \cdot \frac{\sum_{m=1}^M q(m) p(x|m)}{\sum_{m=1}^M w(m) p(m)p(x|m)} \, dx 
		= \int_{\mathcal{X}} p(x|y) \, dx,
	\end{align*}
	which is equivalent to 
	\begin{align}\label{eq::equalto01}
		\int_{\mathcal{X}} p(x|y) \cdot \frac{\sum_{m=1}^M (q(m) - w(m) p(m)) p(x|m)}{\sum_{m=1}^M w(m)p(m)p(x|m)} \, dx = 0.
	\end{align}
	Multiplying both sides of \eqref{eq::equalto01} by $(q(y) - w(y) p(y))$ and taking the summation from $y=1$ to $K$, we obtain
	\begin{align*}
		\int_{\mathcal{X}} \frac{\big(\sum_{m=1}^M (q(m) - w(m) p(m)) p(x|m)\big)^2}{\sum_{m=1}^M w(m) p(m)p(x|m)} \, dx = 0.
	\end{align*}
	By Assumption \ref{ass:labelshift}, we have
	$p(x|k) = q(x|k)$ and thus
	\begin{align*}
		\int_{\mathcal{X}} \frac{\big(\sum_{m=1}^M (q(m) - w(m) p(m)) q(x|m)\big)^2}{\sum_{m=1}^M w(m) p(m) q(x|m)} \, dx = 0.
	\end{align*}
	Since $\sum_{m=1}^M w(m) p(m) q(x|m) > 0$, there must hold $\sum_{m=1}^M (q(m) - w(m) p(m)) q(x|m) = 0$. By Assumption \ref{ass::LinearIndependence}, we get $q(m) = w(m) p(m)$ and thus $w(m) = q(m)/p(m) = w^*(m)$, $m \in [M]$, i.e., $w = w^*$.

	``$\Leftarrow$'' Plugging $w = w^*$ into the right-hand side of \eqref{eq::equationsystem}, we obtain
	\begin{align*}
		\mathbb{E}_{x \sim q} & \frac{p(y|x)}{\sum_{m=1}^M w^*(m) p(m|x)} = \mathbb{E}_{x \sim q} \frac{p(x|y) p(y) / p(x)}{\sum_{m=1}^M w^*(m)  p(x|m) p(m) / p(x)}
		\\
		&= \mathbb{E}_{x \sim q}  \frac{p(y) p(x|y)}{\sum_{m=1}^M w^*(m) p(m)  p(x|m)}
		= \int_{\mathcal{X}} \frac{p(y) p(x|y)}{\sum_{m=1}^M q(m)  q(x|m)} q(x)\, dx
		\nonumber\\
		&= \int_{\mathcal{X}}\frac{p(y) p(x|y)}{q(x)} \cdot q(x)\, dx = \int_{\mathcal{X}} p(y) p(x|y)\, dx = p(y),
	\end{align*}
	which proves the assertion.
\end{proof}

To derive the error bound of the class probability ratio estimation in Proposition \ref{prop::decompweighterror}, we need the following lemmas.

\begin{lemma}\label{lem::errorPY}
	Let $\widehat{p}(y)$ be the class probability estimator in \eqref{eq::directCounting}. Then 
	with probability at least $1-1/n_p$,
	there holds
	$\sum_{y=1}^M  |\widehat{p}(y) - p(y)|^2 
	\lesssim \log n_p / n_p$.
\end{lemma}

\begin{proof}[Proof of Lemma \ref{lem::errorPY}]
	Let us define 
	the random variables $\xi_{i,m} := \eins \{ Y_i = m \} - p(m)$
	for $i \in [n_p]$ and $m \in [M]$.
	Then we have $\mathbb{E} \xi_{i,m} = 0$, $\|\xi_{i,m}\|_{\infty} \leq 1$ and $\mathbb{E} \xi_{i,m}^2 = p(m)(1 - p(m)) \leq p(m)$.
	Applying Bernstein's inequality in 
	\cite[Theorem 6.12]{steinwart2008support}
	to $(\xi_{i,m})_{i \in [n_p]}$, we obtain
	\begin{align*}
		\biggl| \frac{1}{n_p} \sum_{i=1}^{n_p} \xi_{i,m} \biggr| 
		= |\widehat{p}(m) - p(m)| 
		\leq \sqrt{\frac{2 p(m) \zeta}{n_p}} + \frac{2 \zeta}{3 n_p}
	\end{align*}
	with probability at least $1 - 2 e^{- \zeta}$. 
	Using the union bound and $(a + b)^2 \leq 2 (a^2 + b^2)$, we get
	\begin{align*}
		\sum_{m=1}^M |\widehat{p}(m) - p(m)|^2 
		\leq \sum_{m=1}^M \biggl( \sqrt{\frac{2 p(m) \zeta}{n_p}} + \frac{2 \zeta}{3 n_p} \biggr)^2 
		\leq \sum_{m=1}^M \biggl( \frac{4 p(m) \zeta}{n_p} + \frac{4 \zeta^2}{9 n_p^2} \biggr)
	\end{align*}
	with probability at least $1 - 2 M e^{-\zeta}$.
	Taking $\zeta := \log(2 M n_p)$, we obtain
	\begin{align*}
		\sum_{m=1}^M  |\widehat{p}(m) - p(m)|^2 
		\leq \frac{4 \log(2 M n_p)}{n_p} + \frac{4 \log^2(2 M n_p)}{9 n_p^2} 
		\lesssim \frac{\log n_p}{n_p}
	\end{align*}
	with probability at least $1 - 1/n_p$. 
	This proves the assertion. 
\end{proof}

In order to establish the upper bound of $\|\widehat{w} - w^*\|_2$ in Proposition \ref{prop::decompweighterror}, we also need the following lemma.

\begin{lemma}\label{lem::integralerror}
	Let Assumptions \ref{ass:labelshift}, \ref{ass::predictor}, \ref{ass::marginal} and \ref{ass::regularity} hold. 
	Then for any $w \in \{\widehat{w}, w^*\}$, there holds
	\begin{align*}
		& \sum_{m=1}^M  \biggl| \frac{1}{n_q} \sum_{i=n_p+1}^{n_p+n_q} \frac{\widehat{p}(m|X_i)}{\sum_{j=1}^M  w(j) \widehat{p}(j|X_i)} - \mathbb{E}_{x \sim q} \frac{p(m|x)}{\sum_{j=1}^M  w(j) p(j|x)} \biggr|^2
		\nonumber\\
		& \lesssim \log n_q / n_q + \mathbb{E}_{x \sim p} \|p(\cdot |x) - \widehat{p}(\cdot |x)\|_2^2
	\end{align*}
	with probability at least $1 - 1/n_q$.
\end{lemma}

To prove Lemma \ref{lem::integralerror}, we need the following lemma.

\begin{lemma}\label{lem::newregular}
	Let Assumptions \ref{ass::predictor}, \ref{ass::marginal} and \ref{ass::regularity} hold. Then there exist some $N \in \mathbb{N}$ and $c'_R>0$ such that for 
	all $n_p \wedge n_q \geq N$
	and
	all $x$ with $q(x)>0$, there holds
	$\sum_{m \in [M]} \widehat{w}(m) p(m|x) \geq c'_R
	$ 
	with probability $P^{n_p}$ at least $1-1/n_p$.
\end{lemma}

\begin{proof}[Proof of Lemma \ref{lem::newregular}]
	We prove this by contradiction. Let $n := n_p \wedge n_q$ in the following proof. Since $\widehat{w}(m)$ depends on $n$, we rewrite $\widehat{w}(m)$ as $\widehat{w}_n(m)$. 
	Assume that for any $N_0\in \mathbb{N}$ and any $k\in \mathbb{N}$, there exists an $x_0$ satisfying $q(x_0) > 0$ such that $\sum_{y \in [M]} \widehat{w}_{n_k}(y) p(y|x_0) < 1/k$ for at least one $n_k \geq N_0$.  This implies that for any class index $m \in [M]$, either there exists a subsequence $\{ \widehat{w}_{n_k}(m) \}_{k \in \mathbb{N}}$ of the weight sequence $\{ \widehat{w}_n(m) \}_{n \in \mathbb{N}}$ corresponding to different sample size $n_k$ converging to zero, or $p(m|x_0) = 0$. Denote $\mathcal{M}_0$ as the set of class indices $m$ for which the subsquences $\{ \widehat{w}_{n_k}(m) \}_{k \in \mathbb{N}}$ converges to zero. Then for any $m \in [M]\setminus \mathcal{M}_0$, we have $p(m|x_0)=0$. 
	Let $r := n^{-\alpha/(2d((1+\beta\wedge 1)\alpha+d))}$. 
	By Assumption \ref{ass::marginal}, we have
	\begin{align}\label{eq::L2balllower}
		\int_{B(x_0, r)} |p(m|z) - \widehat{p}(m|z)|^2 p(z)\, dz &\geq P(B(x_0,r)) \inf_{z \in B(x_0, r)} |p(m|z) - \widehat{p}(m|z)|^2  \nonumber\\
		&\geq \frac{c_{-}\pi^{d/2} r^d}{\Gamma(d/2+1)} \inf_{z \in B(x_0, r)} |p(m|z) - \widehat{p}(m|z)|^2. 
	\end{align}
	Combining Lemma \ref{lem::equilvalent} and Theorem \ref{thm::convergencerate}, we obtain
	\begin{align*}
		\int_{B(x_0, r)} |p(m|z) - \widehat{p}(m|z)|^2 p(z)\, dz &\leq \int_{\mathcal{X}} |p(m|z) - \widehat{p}(m|z)|^2 p(z)\, dz
		\\
		&\leq \mathcal{R}_{L_{\mathrm{CE}},P}(\widehat{p}(\cdot|x)) - \mathcal{R}_{L_{\mathrm{CE}},P}^* \lesssim n^{-\frac{(1+\beta\wedge 1)\alpha}{(1+\beta\wedge 1) \alpha + d}+\xi}
	\end{align*}
	This together with \eqref{eq::L2balllower} implies 
	\begin{align*}
		\inf_{z \in B(x_0, r)} |p(m|z) - \widehat{p}(m|z)| \lesssim \frac{ \Gamma(d/2+1)}{c_{-}\pi^{d/2}} r^{-d} n^{-\frac{(1+\beta\wedge 1)\alpha}{(1+\beta\wedge 1) \alpha + d}+\xi} \lesssim n^{-\frac{(1/2+\beta\wedge 1)\alpha}{(1+\beta\wedge 1) \alpha + d}+\xi}.
	\end{align*}
	This implies that there exist some $z_0 \in B(x_0, r)$ such that 
	\begin{align}\label{eq::zbx0r}
		|p(m|z_0) - \widehat{p}(m|z_0)| \lesssim n^{-\frac{(1/2+\beta\wedge 1)\alpha}{(1+\beta\wedge 1) \alpha + d}+\xi}.
	\end{align}

	In the following, 
	we show that the sequence 
	$\sum_{m\in [M]} \widehat{w}_{n_k}(m) \widehat{p}(m|z_0)$ concerning $k$ converging to zero.
	To this end, we consider the following decomposition:
	\begin{align}\label{eq::sumde1}
		\sum_{m\in [M]} \widehat{w}_{n_k}(m) \widehat{p}(m|z_0) & = \sum_{m\in [M]\setminus\mathcal{M}_0} \widehat{w}_{n_k}(m)\widehat{p}(m|z_0) + \sum_{m\in\mathcal{M}_0} \widehat{w}_{n_k}(m)\widehat{p}(m|z_0)
		\nonumber\\
		& =: (I) + (II).
	\end{align}
	For the first term $(I)$ in \eqref{eq::sumde1}, using the triangle inequality, \eqref{eq::zbx0r}, and Assumption \ref{ass::predictor} \textit{(i)}, we obtain
	\begin{align*}
		\widehat{p}(m|z_0) &\leq |\widehat{p}(m|z_0) -p(m|z_0)| + p(m|z_0) 
		\nonumber\\
		&\lesssim n^{-\frac{(1/2+\beta\wedge 1)\alpha}{(1+\beta\wedge 1) \alpha + d}+\xi} + |p(m|z_0) - p(m|x_0)| + p(m|x_0),
		\nonumber\\
		&\lesssim n^{-\frac{(1/2+\beta\wedge 1)\alpha}{(1+\beta\wedge 1) \alpha + d}+\xi} + c_{\alpha} r^{\alpha} + 0
		\lesssim n^{-\frac{\alpha^2}{2d((1+\beta\wedge 1)\alpha+d)}}, 
		\qquad
		m\in [M]\setminus\mathcal{M}_0,
	\end{align*}
	where we used the fact that $p(m|x_0) = 0$ holds for $m \in [M] \setminus \mathcal{M}_0$. Therefore, for any $\varepsilon > 0$, there exists an $N_1$ such that for any integer $k \geq N_1$, we have $\widehat{p}(m|z_0) \leq \widehat{p}(m) \varepsilon/2$
	and thus 
	\begin{align}\label{eq::boundI}
		(I) \leq \sum_{m\in [M]\setminus\mathcal{M}_0} \widehat{w}_{n_k}(m) \widehat{p}(m) \varepsilon/2 
		\lesssim \varepsilon/2.
	\end{align}
	For the second term $(II)$ in \eqref{eq::sumde1}, since the subsequence $\{ \widehat{w}_{n_k}(m) \}_{k \in \mathbb{N}}$ converging to zero for any $m\in \mathcal{M}_0$, there exists an $N_2$ such that $(II) \leq \varepsilon/2$ for any $k\geq N_2$. 
	This together with \eqref{eq::sumde1} and \eqref{eq::boundI} yields that
	for any $k \geq N := N_1 \vee N_2$, we have $\sum_{m \in [M]} \widehat{w}_{n_k}(m) \widehat{p}(m|z_0) = (I) + (II) \lesssim \varepsilon$. 
	This contradicts the regularity Assumption 
	\ref{ass::regularity} which states that there exist some constant $c_R$ and $N'\in \mathbb{N}$ such that for any $n\geq N'$ and any $x$ with $q(x) > 0$, we have
	$\sum_{m \in [M]} \widehat{w}_n(m) \widehat{p}(m|x) \geq c_R$. 
	Thus we finish the proof.
\end{proof}

\begin{proof}[Proof of Lemma \ref{lem::integralerror}]
	Using the triangle inequality, we get 
	\begin{align}\label{eq::diffetawk}
		& \biggl| \frac{\widehat{p}(m|x)}{\sum_{j=1}^M  w(j) \widehat{p}(j|x)}
		- \frac{p(m|x)}{\sum_{j=1}^M  w(j) p(j|x)} \biggr|
		\nonumber\\
		& = \frac{|\widehat{p}(m|x) \sum_{j \neq m} w(j) p(j|x) - p(m|x) \sum_{j \neq m} w(j) \widehat{p}(j|x)|}{\bigl( \sum_{j=1}^M  w(j) \widehat{p}(j|x) \bigr) \cdot \bigl( \sum_{j=1}^M  w(j) p(j|x) \bigr)}
		\nonumber\\
		& \leq \frac{\sum_{j \neq m} w(j) \bigl| p(j|x) \widehat{p}(m|x) - p(m|x) \widehat{p}(j|x) \bigr|}{\bigl( \sum_{j=1}^M  w(j) \widehat{p}(j|x) \bigr) \cdot \bigl( \sum_{j=1}^M  w(j) p(j|x) \bigr)}
		\nonumber\\
		& \leq \frac{\sum_{j \neq m} w(j) \bigl| p(j|x) \widehat{p}(m|x) - p(j|x) p(m|x) + p(m|x) p(j|x) -p(m|x) \widehat{p}(j|x) \bigr|}{\bigl( \sum_{j=1}^M  w(j) \widehat{p}(j|x) \bigr) \cdot \bigl( \sum_{j=1}^M  w(j) p(j|x) \bigr)}
		\nonumber\\
		&\leq \frac{|\widehat{p}(m|x) - p(m|x)|}{\sum_{j=1}^M  w(j) \widehat{p}(j|x)} + \frac{p(m|x) \sum_{j \neq m} w(j) |p(j|x) -\widehat{p}(j|x)|}{\bigl( \sum_{j=1}^M  w(j) \widehat{p}(j|x) \bigr) \cdot \bigl( \sum_{j=1}^M  w(j) p(j|x) \bigr)}.
	\end{align}
	By Assumptions \ref{ass::marginal}, \ref{ass::regularity} and Lemma \ref{lem::newregular}, there exists a constant $
	c_{\tau} := c_R \wedge c_R' \wedge \underline{c}$
	such that $\sum_{j=1}^M  w(j) p(j|x) \geq 
	\zeta$ and $\sum_{j=1}^M  w(j) \widehat{p}(j|x) \geq 
	\zeta$ for any $x \in \mathcal{X}$ and $\{\widehat{w}, w^*\}$.
	This together with the triangle inequality, \eqref{eq::diffetawk}, and $(a+b)^2 \leq 2(a^2 + b^2)$ for $a,b\geq 0$, $\sum_{m=1}^M  a_M^2 \leq (\sum a_m)^2$ for $a_m \geq 0$, $m \in [M]$, yields
	\begin{align}\label{eq::leftsquaresum}
		& \sum_{j=1}^M  \biggl| \int_{\mathcal{X}} \biggl( \frac{\widehat{p}(m|x)}{\sum_{j=1}^M  w(j) \widehat{p}(j|x)} - \frac{p(m|x)}{\sum_{j=1}^M  w(j) p(j|x)} \biggr) q(x) \, dx \biggr|^2
		\nonumber\\
		& \leq \int_{\mathcal{X}} \sum_{j=1}^M  \biggl| \frac{\widehat{p}(m|x)}{\sum_{j=1}^M  w(j) \widehat{p}(j|x)} - \frac{p(m|x)}{\sum_{j=1}^M  w(j) p(j|x)} \biggr|^2 q(x)\, dx
		\nonumber\\
		& \leq \int_{\mathcal{X}} \sum_{m=1}^M  \biggl( \frac{|\widehat{p}(m|x) - p(m|x)|}{\sum_{j=1}^M  w(j) \widehat{p}(j|x)} + \frac{p(m|x) \sum_{j \neq m} w(j) |p(j|x) - \widehat{p}(j|x)|}{\bigl( \sum_{j=1}^M  w(j) \widehat{p}(j|x) \bigr) \cdot \bigl( \sum_{j=1}^M  w(j) p(j|x) \bigr)} \biggr)^2 q(x) \, dx
		\nonumber\\
		& \leq 2 \int_{\mathcal{X}} \sum_{m=1}^M  \biggl( \biggl( \frac{|\widehat{p}(m|x) - p(m|x)|}{\sum_{j=1}^M  w(j) \widehat{p}(j|x)} \biggr)^2 + \biggl( \frac{p(m|x) \sum_{j \neq m} w(j) |p(j|x) -\widehat{p}(j|x)|}{\bigl( \sum_{j=1}^M  w(j) \widehat{p}(j|x) \bigr) \cdot \bigl( \sum_{j=1}^M  w(j) p(j|x) \bigr)} \biggr)^2 \biggr) q(x)\, dx
		\nonumber\\
		& \leq \frac{2}{c_{\tau}^2} \sum_{m=1}^M  \int_{\mathcal{X}} |\widehat{p}(m|x) - p(m|x)|^2 q(x) \, dx + \frac{2}{c_{\tau}^4} \int_{\mathcal{X}} \sum_{m=1}^M  \biggl(\sum_{j \neq m} w(j) |p(j|x) - \widehat{p}(j|x)| \biggr)^2 q(x) \, dx
		\nonumber\\
		& \lesssim \mathbb{E}_{x \sim q} \|p(\cdot |x) - \widehat{p}(\cdot |x)\|_2^2 + \int_{\mathcal{X}} \sum_{m=1}^M \biggl( \sum_{j=1}^M  w(j) |p(j|x) - \widehat{p}(j|x)| \biggr)^2 q(x) \, dx
		\nonumber\\
		& \leq \mathbb{E}_{x \sim q} \|p(\cdot |x) - \widehat{p}(\cdot |x)\|_2^2 + \int_{\mathcal{X}} M \|w\|_2^2 \biggl( \sum_{j=1}^M  |p(j|x) - \widehat{p}(j|x)| \biggr)^2 q(x) \, dx
		\nonumber\\
		& \lesssim \mathbb{E}_{x \sim q} \|p(\cdot |x) - \widehat{p}(\cdot |x)\|_2^2 \lesssim \mathbb{E}_{x \sim p} \|p(\cdot |x) - \widehat{p}(\cdot |x)\|_2^2.
	\end{align}
	Here, the second last inequality follow from Cauchy-Schwarz inequality and
	the last inequality holds due to Assumption \ref{ass:labelshift} and 
	\begin{align}\label{eq::qxoverpx}
		\frac{q(x)}{p(x)} 
		= \frac{\sum_{m \in [M]} q(x|m) q(m)}{\sum_{m \in [M]} p(x|m) p(m)} 
		= \frac{\sum_{m \in [M]} q(x|m) q(m)}{\sum_{m \in [M]} q(x|m) p(m)} 
		\leq \frac{\bigvee_{m \in [M]} q(m)}{\bigwedge_{m \in [M]} p(m)} 
		\leq \frac{1}{p_{\min}(y)},
	\end{align}
	where $p_{\min}(y) := \bigwedge_{m\in [M]}p(m)$.
	For $i \in [n_q]$ and $m \in [M]$, define the random variables
	\begin{align*}
		\xi_{i,m} := \frac{\widehat{p}(m|X_i)}{\sum_{j=1}^M  w(j) \widehat{p}(j|X_i)} - \mathbb{E}_{x \sim  q}\frac{\widehat{p}(m|x)}{\sum_{j=1}^M  w(j) \widehat{p}(j|x)}. 
	\end{align*}
	Then we have $\mathbb{E}_{x \sim q} \xi_{i,m} = 0$, 
	\begin{align*}
		\|\xi_{i,m}\|_{\infty} 
		\leq \frac{1}{\sum_{j=1}^M  w(j) \widehat{p}(j|X_i)} + \mathbb{E}_{x \sim q} \frac{1}{\sum_{j=1}^M  w(j) \widehat{p}(j|x)} \leq \frac{2}{\zeta}
	\end{align*} 
	and 
	\begin{align*}
		\mathbb{E}_{X_i \sim q} \xi_{i,m}^2 
		\leq \mathbb{E}_{x \sim  q} \biggl( \frac{\widehat{p}(m|x)}{\sum_{j=1}^M  w(j) \widehat{p}(j|x)} \biggr)^2 
		\leq \mathbb{E}_{x \sim  q} \frac{1}{(\sum_{j=1}^M  w(j) \widehat{p}(j|x))^2} 
		\leq \frac{1}{\zeta^2}.
	\end{align*} 
	Applying Bernstein's inequality in 
	\cite[Theorem 6.12]{steinwart2008support}
	to $(\xi_{i,m})_{i \in [n_q]}$, we get
	\begin{align}\label{eq::errornQ}
		\biggl| \frac{1}{n_q} \sum_{i=1}^{n_q} \xi_{i,m} \bigg| 
		= \biggl| \frac{1}{n_q} \sum_{i=1}^{n_q} \frac{\widehat{p}(m|X_i)}{\sum_{j=1}^M  w(j) \widehat{p}(j|X_i)} - \mathbb{E}_{x \sim  q} \frac{\widehat{p}(m|x)}{\sum_{j=1}^M  w(j) \widehat{p}(j|x)} \biggr|
		\leq \sqrt{\frac{2 \zeta}{\zeta^2 n_q}} + \frac{4 \zeta}{3 \zeta n_q}
	\end{align}
	with probability at least $1-2e^{-\zeta}$. Taking $\zeta := \log (2n_q)$, we obtain
	\begin{align*}
		\frac{1}{n_q} \sum_{i=1}^{n_q} \frac{\widehat{p}(m|X_i)}{\sum_{j=1}^M  w(j) \widehat{p}(j|X_i)} - \mathbb{E}_{x \sim  q} \frac{\widehat{p}(m|x)}{\sum_{j=1}^M  w(j) \widehat{p}(j|x)}
		\leq \sqrt{\frac{4 \log n_q}{\zeta^2 n_q}} + \frac{8 \log n_q}{3\zeta n_q}
	\end{align*}
	with probability at least $1 - 1/n_q$. 
	Using $(a + b)^2 \leq 2 (a^2 + b^2)$, \eqref{eq::errornQ} and \eqref{eq::leftsquaresum}, we obtain 
	\begin{align*}
		& \sum_{m=1}^M  \biggl| \frac{1}{n_q} \sum_{i=n_p+1}^{n_p+n_q} \frac{\widehat{p}(m|X_i)}{\sum_{j=1}^M  w(j) \widehat{p}(j|X_i)} - \mathbb{E}_{x \sim q} \frac{p(m|x)}{\sum_{j=1}^M  w(j) p(j|x)} \biggr|^2 
		\\
		& \leq 2 \sum_{m=1}^M  \biggl| \frac{1}{n_q} \sum_{i=n_p+1}^{n_p+n_q} \frac{\widehat{p}(m|X_i)}{\sum_{j=1}^M  w(j) \widehat{p}(j|X_i)} - \mathbb{E}_{x \sim q} \frac{\widehat{p}(m|x)}{\sum_{j=1}^M  w(j) \widehat{p}(j|x)} \biggr|^2
		\\
		& \phantom{=} 
		+ 2 \sum_{m=1}^M  \biggl| \mathbb{E}_{x \sim q} \frac{\widehat{p}(m|x)}{\sum_{j=1}^M  w(j) \widehat{p}(j|x)} - \mathbb{E}_{x \sim q} \frac{p(m|x)}{\sum_{j=1}^M  w(j) p(j|x)} \biggr|^2
		\\
		& \lesssim 2 \sum_{m=1}^M  \biggl(\sqrt{\frac{4 \log n_q}{\zeta^2 n_q}} + \frac{8 \log n_q}{3\zeta n_q} \biggr)^2 + 2 \mathbb{E}_{x \sim p} \|p(\cdot |x) - \widehat{p}(\cdot |x)\|_2^2
		\\
		& \lesssim \log n_q/n_q + \mathbb{E}_{x \sim p} \|p(\cdot |x) - \widehat{p}(\cdot |x)\|_2^2,
	\end{align*}
	which finishes the proof.
\end{proof}

The following lemma presents an upper bound of the class probability ratio estimation error $\|\widehat{w} - w^*\|_2$, which is crucial to prove Proposition \ref{prop::decompweighterror}.

\begin{lemma}\label{lem::weightsigma}
	Let Assumptions \ref{ass:labelshift}, \ref{ass::LinearIndependence}, \ref{ass::marginal} and \ref{ass::regularity} hold. 
	Moreover, let the class probability ratio $w^*:=(w^*(y))_{y\in[M]}$ and its estimator $\widehat{w}$ be defined as in \eqref{eq::wstar} and \eqref{eq::whatminimizer}, respectively. Then we have
	\begin{align*}
		\|\widehat{w} - w^*\|_2 
		\lesssim \biggl\| \mathbb{E}_{x \sim q} \frac{p(\cdot|x)}{\sum_{j=1}^M  \widehat{w}(j) p(j|x)} - p(\cdot) \biggr\|_2
	\end{align*}
	with probability at least $1 - 1/n_p$.
\end{lemma}

In order to prove Lemma \ref{lem::weightsigma}, we need the following lemma concerning the minimum eigenvalue of matrix.

\begin{lemma}\label{lem::comparematrix}
	Let $A$ and $B$ be two $d \times d$ real symmetric positive semi-definite matrices and their minimum eigenvalues are $\sigma_A$ and $\sigma_B$, respectively. Assume that $A \geq B$, i.e. $A-B$ is positive semi-definite. Then we have $\sigma_A \geq \sigma_B$.
\end{lemma}

\begin{proof}[Proof of Lemma \ref{lem::comparematrix}]
	Since $A$ is symmetric and invertible, it can be decomposed as $
	B = Q^{\top} D Q$,
	where $Q$ is an orthonormal matrix and $D$ is a diagonal matrix. Then for any unit vector $v\in \mathbb{R}^d$, we have 
	$$
	v^{\top} B v 
	= v^{\top} Q^{\top} D Q v 
	= (Q v)^{\top} D (Q v) 
	= \sum_{j=1}^M D_{jj} (Q v)_j^2 
	\geq \sigma_B \sum_{j=1}^M (Q v)_j^2 
	= \sigma_B,
	$$
	where $D_{jj}$ is the $j$-th diagonal entry in the diagonal matrix $D$. 
	Since $A \geq B$, for any unit vector $v\in \mathbb{R}^d$, we have
	$v^{\top} A v  
	\geq v^{\top} B v 
	\geq \sigma_B$.
	Applying this to $v := v_A$, we get $v_A^{\top} A v_A \geq \sigma_B$. 
	On the other hand, for the unit eigenvector $v_A$ corresponding to the smallest eigenvalue of the matrix $A$, there holds
	$v_A^{\top} A v_A  
	= v_A^{\top} \sigma_A v_A 
	= \sigma_A$.
	Therefore, we have $\sigma_A \geq \sigma_B$, which finishes the proof.
\end{proof}

\begin{proof}[Proof of Lemma \ref{lem::weightsigma}]
	Using Bayes' Formula, the law of total probability and Assumption \ref{ass:labelshift}, we get
	\begin{align*}
		& \mathbb{E}_{x \sim q} \frac{p(y|x)}{\sum_{j=1}^M  \widehat{w}(j) p(j|x)}
		= \mathbb{E}_{x \sim q} \frac{p(x|y) p(y) / p(x)}{\sum_{j=1}^M  \widehat{w}(j) p(x|j) p(j) / p(x)}
		= \mathbb{E}_{x \sim q} \frac{p(y) p(x|y)}{\sum_{j=1}^M  \widehat{w}(j) p(j) p(x|j)}
		\\
		& = p(y) \int_{\mathcal{X}} \frac{p(x|y)}{\sum_{j=1}^M  \widehat{w}(j) p(j) p(x|j)} q(x) \, dx
		= p(y) \int_{\mathcal{X}} p(x|y) \frac{\sum_{j=1}^M  q(j) q(x|j)}{\sum_{j=1}^M  \widehat{w}(j)p(j) p(x|j)} \, dx
		\\
		& = p(y) \int_{\mathcal{X}} q(x|y) \frac{\sum_{j=1}^M  q(j) q(x|j)}{\sum_{j=1}^M  \widehat{w}(j) p(j) q(x|j)} \, dx
		\\
		&=  p(y) + p(y) \int_{\mathcal{X}} q(x|y) \frac{\sum_{j=1}^M  (q(j)-\widehat{w}(j) p(j)) q(x|j)}{\sum_{j=1}^M  \widehat{w}(j) p(j) q(x|j)} \, dx
		\\
		& = p(y) + p(y) \sum_{j=1}^M  (q(j) - \widehat{w}(j) p(j)) \int_{\mathcal{X}} \frac{q(x|y) q(x|j)}{\sum_{j=1}^M  \widehat{w}(j) p(j) q(x|j)} \, dx
		\\
		& = p(y) + \sum_{j=1}^M  (w^*(j) - \widehat{w}(j)) \int_{\mathcal{X}} \frac{p(j) p(y) q(x|j) q(x|y)}{\sum_{j=1}^M  \widehat{w}(j) p(j) q(x|j)} \, dx.
	\end{align*}
	Let the entries of the matrix $\mathcal{C} := (c_{yj})_{y, j \in [M]}$ be defined by
	\begin{align*}
		c_{yj} := \int_{\mathcal{X}} \frac{p(j) p(y) q(x|j) q(x|y)}{\sum_{m=1}^M  \widehat{w}(m) p(m) q(x|m)} \, dx,
		\qquad 
		y, j \in [M].
	\end{align*}
	Then we have 
	\begin{align}\label{eq::expecminusPY}
		\mathbb{E}_{x \sim q} \frac{p(y|x)}{\sum_{m=1}^M  \widehat{w}(m) p(m|x)} - p(y) 
		= \sum_{j=1}^M  c_{yj} (w^*(j) - \widehat{w}(j)), 
		\qquad 
		y \in [M].
	\end{align}
	Denote the vectors $p(\cdot) = (p(y=1), \ldots, p(y=M))^T$ and $p(\cdot|x) = (p(y=1|x), \ldots, p(y=M|x))^T$. Then we write \eqref{eq::expecminusPY} in  vector form as 
	\begin{align*}
		\mathbb{E}_{x \sim q} \frac{p(\cdot|x)}{\sum_{m=1}^M  \widehat{w}(m) p(m|x)} - p(\cdot) 
		= \mathcal{C} (w^* - \widehat{w}).
	\end{align*}
	Now, we prove that $\mathcal{C}$ is invertible by showing that the rows of $\mathcal{C}$ are linearly independent. 
	To this end, assume that there exist $\alpha_1, \ldots, \alpha_m \in \mathbb{R}$ such that $\sum_{y=1}^M  \alpha_y c_{yj} = 0$, $j \in [M]$. Then we have
	\begin{align*}
		0 
		= \sum_{y=1}^M  \alpha_y c_{yj} 
		& = \sum_{y=1}^M  \alpha_y \int_{\mathcal{X}} \frac{p(j) p(y) q(x|j) q(x|y)}{\sum_{j=1}^M  \widehat{w}(j) p(j) q(x|j)} \, dx 
		\nonumber\\
		& = \int_{\mathcal{X}} \frac{\sum_{y=1}^M  \alpha_y p(y) q(x|y)}{\sum_{m=1}^M  \widehat{w}(m) p(m) q(x|m)} \cdot p(j) q(x|j) \, dx.
	\end{align*}
	Multiplying both sides of the above equation by $\alpha_j$ and taking the summation over $j$ from $1$ to $M$, we get 
	\begin{align*}
		0 
		= \sum_{j=1}^M  \alpha_j \int_{\mathcal{X}} \frac{\sum_{y=1}^M  \alpha_y p(y) q(x|y)}{\sum_{m=1}^M  \widehat{w}(m) p(m) q(x|m)} \cdot p(j) q(x|j) \, dx
		= \int_{\mathcal{X}} \frac{\bigl( \sum_{y=1}^M  \alpha_y p(y) q(x|y) \bigr)^2}{\sum_{m=1}^M  \widehat{w}(m) p(m) q(x|m)} \, dx,
	\end{align*}
	which implies $\sum_{y=1}^M  \alpha_y p(y) q(x|y) = 0$ for all $x\in \mathcal{X}$. Due to Assumption \ref{ass::LinearIndependence}, $\{q(x|y):y\in [M]\}$ are linearly independent and therefore $\alpha_y p(y) = 0$. Since $p(y) > 0$ for any $y\in [M]$ by Assumption \ref{ass::marginal}, we have $\alpha_y = 0$, $y\in [M]$.
	Therefore, the rows of $\mathcal{C}$ are linearly independent and thus $\mathcal{C}$ is invertible. Denote the minimum eigenvalue of the matrix $\mathcal{C}$ as $\sigma_{\mathcal{C}}$.
	By the Cauchy-Schwarz inequality, we get
	\begin{align*}
		\|\widehat{w} - w^*\|_2^2
		& = (\widehat{w} - w^*)^{\top} (\widehat{w} - w^*)
		\leq \|\widehat{w} - w^*\|_2 \cdot \|\widehat{w} - w^*\|_2
		\\
		& = \|\widehat{w} - w^*\|_2 \cdot \biggl\|\mathcal{C}^{-1} \biggl( \mathbb{E}_{x \sim q} \frac{p(\cdot|x)}{\sum_{m=1}^M \widehat{w}(m) p(m|x)} - p(\cdot) \biggr) \biggr\|_2
		\\
		& \leq \|\widehat{w} - w^*\|_2 \cdot \|\mathcal{C}^{-1} \| \cdot \biggl\| \mathbb{E}_{x \sim q} \frac{p(\cdot|x)}{\sum_{m=1}^M  \widehat{w}(m) p(m|x)} - p(\cdot) \biggr\|_2
		\\
		& \leq \frac{1}{\sigma_{\mathcal{C}}} \cdot \|\widehat{w} - w^*\|_2 \cdot \biggl\| \mathbb{E}_{x \sim q} \frac{p(\cdot|x)}{\sum_{m=1}^M  \widehat{w}(m) p(m|x)} - p(\cdot) \biggr\|_2,
	\end{align*}
	which yields 
	\begin{align}\label{eq::weightsigmaC}
		\|\widehat{w} - w^*\|_2 
		\leq \frac{1}{\sigma_{\mathcal{C}}} \cdot \biggl\|\mathbb{E}_{x \sim q} \frac{p(\cdot|x)}{\sum_{m=1}^M  \widehat{w}(m) p(m|x)} - p(\cdot) \biggr\|_2.
	\end{align}
	By Assumption \ref{ass:labelshift} and Bayes Formula, the entries $c_{yj}$ in the matrix $\mathcal{C}$ can be written as 
	\begin{align*}
		c_{yj} &= \int_{\mathcal{X}} \frac{p(j) p(y) p(x|j) p(x|y)}{\sum_{m=1}^M  \widehat{w}(m) p(m) p(x|m)} \, dx =  \int_{\mathcal{X}} \frac{p(x) p(x) p(j|x) p(y|x)}{\sum_{m=1}^M  \widehat{w}(m) p(x) p(m|x)} \, dx
		\nonumber\\
		&= \int_{\mathcal{X}} \frac{p(j|x)p(y|x)}{\sum_{m=1}^M  \widehat{w}(m) p(m|x)} p(x)\, dx = \mathbb{E}_{x\sim p} \frac{p(j|x)p(y|x)}{\sum_{m=1}^M  \widehat{w}(m) p(m|x)}.
	\end{align*}
	By Lemma \ref{lem::errorPY}, we have $\widehat{p}(m) - p(m) \geq -c_p \sqrt{{\log n_p}/{n_p}}$
	with probability at least $1-1/n_p$, where $c_p$ is the constant depending on the probability distribution $P$. 
	Therefore, for sufficiently large $n_p$ such that $n_p/\log n_p \geq 4c_p^2 p_{\min}^{-2}$,
	where $p_{\min}:=\bigwedge_{m\in [M]} p(m)$,
	we have $
	\widehat{p}(m) \geq p(m)/2$.
	Since $\sum_{m\in [M]} \widehat{w}(m) \widehat{p}(m) = 1$, we have $\widehat{w}(m) \widehat{p}(m) \leq 1$ for any $m\in [M]$ and thus $\widehat{w}(m) \leq 1/\widehat{p}(m) \leq 2/p(m) \leq 2/p_{\min}$.
	Therefore, we get
	\begin{align*}
		\sum_{m=1}^M  \widehat{w}(m) p(m|x) \leq \sum_{m=1}^M  \widehat{w}_{\max} p(m|x) \leq \widehat{w}_{\max}\sum_{m=1}^M p(m|x) = \widehat{w}_{\max}\leq 2/p_{\min},
	\end{align*}
	where $\widehat{w}_{\max} := \bigvee_{m \in [M]} \widehat{w}(m)$. 
	By Assumption \ref{ass::regularity}, for any $v \in \mathbb{R}^M$, we have 
	\begin{align}\label{eq::CgeqSigmap}
		v^{\top} \mathcal{C} v &= \mathbb{E}_{x\sim p} \frac{(v^{\top} p(\cdot|x))^2}{\sum_{m=1}^M  \widehat{w}(m) p(m|x)} \geq (p_{\min}/2) \mathbb{E}_{x\sim p} (v^{\top} p(\cdot|x))^2
		\nonumber\\
		&= (p_{\min}/2)  v^{\top} \mathbb{E}_{x\sim p}( p(\cdot|x) p(\cdot|x)^{\top})v = v^{\top} ((p_{\min}/2) \Sigma_p )v,
	\end{align}
	where $\Sigma_p := \mathbb{E}_{x\sim p}(p(\cdot|x) p(\cdot|x)^{\top})$. Since $\Sigma_p$ is a positive semi-definite matrix, $\mathcal{C}$ is also positive semi-definite. \eqref{eq::CgeqSigmap} implies that $\mathcal{C} \geq \frac{1}{2} p_{\min}\Sigma_p$ and thus by Lemma \ref{lem::comparematrix}, we obtain $
	\sigma_{\mathcal{C}} \geq (p_{\min}/2) \sigma_p$.
	This together with \eqref{eq::weightsigmaC} yields that
	\begin{align*}
		\|\widehat{w} - w^*\|_2 
		& \leq \frac{2}{\bigwedge_{m\in [M]} p(m)} \cdot \frac{1}{\sigma_p} \cdot \biggl\| \mathbb{E}_{x \sim q} \frac{p(\cdot|x)}{\sum_{m=1}^M  \widehat{w}(m) p(m|x)} - p(\cdot) \biggr\|_2
		\\
		&\lesssim \sigma_p^{-1} \biggl\| \mathbb{E}_{x \sim q} \frac{p(\cdot|x)}{\sum_{m=1}^M  \widehat{w}(m) p(m|x)} - p(\cdot) \biggr\|_2
	\end{align*}
	holds with probability at least $1-1/n_p$.
\end{proof}

Before we prove Proposition \ref{prop::decompweighterror}, we still need the following lemma, which shows that the $L_2$-distance between $\widehat{p}(y|x)$ and $p(y|x)$ can be upper bounded by the excess CE risk of $\widehat{p}(y|x)$ in the source domain.
\begin{lemma}\label{lem::equilvalent}
	Let $\widehat{p}(y|x)$ be the estimator of $p(y|x)$. Then for any $y \in [M]$, we have 
	\begin{align*}
		\int_{\mathcal{X}} (\widehat{p}(y|x) - p(y|x))^2 p(x) \, dx 
		\leq \mathcal{R}_{L_{\mathrm{CE}},P}(\widehat{p}(\cdot|x)) - \mathcal{R}_{L_{\mathrm{CE}},P}^*.
	\end{align*}
\end{lemma}

\begin{proof}[Proof of Lemma \ref{lem::equilvalent}]
	For any $m \in [M]$, there holds
	\begin{align} \label{eq::EP}
		& \mathcal{R}_{L_{\mathrm{CE}},P}(\widehat{p}(\cdot|x)) - \mathcal{R}_{L_{\mathrm{CE}},P}^* - \mathbb{E}_{x \sim p} |\widehat{p}(m|x) - p(m|x)|^2 
		\nonumber\\
		& = \mathbb{E}_{x \sim p} \Big(\sum_{j=1}^M - p(j|x) \log \frac{\widehat{p}(j|x)}{p(j|x)} - |\widehat{p}(m|x) - p(m|x)|^2\Big).
	\end{align}
	For a given $\ell \in [M]$ and for any vector $u := (u_1, \ldots, u_M) \in (0, 1)^M$, we define 
	\begin{align*}
		h_{\ell}(u) := - \sum_{j=1}^M  p(j|x) \log \frac{u_j}{p(j|x)} - |u_{\ell} - p(\ell|x)|^2.
	\end{align*} 
	To find the minimum of $h_{\ell}(u)$ under the constraints $\sum_{m=1}^M u_m = 1$, we consider the Lagrange function 
	$J_{\ell}(u) := h_{\ell}(u) + \alpha \bigl( \sum_{j=1}^M u_j - 1 \bigr)$
	with the multiplier $\alpha > 0$.
	Taking the derivative w.r.t.~$u_j$, $j = 1, \ldots, M$, and setting them to be zero, we have
	\begin{align}
		\frac{\partial J_{\ell}(u)}{\partial u_{\ell}} 
		& = - \frac{p(\ell|x)}{u_{\ell}} - 2 (u_{\ell} - p(\ell|x)) + \alpha = 0, 
		\label{eq::deruk}
		\\
		\frac{\partial J_{\ell}(u)}{\partial u_j} 
		& = - \frac{p(j|x)}{u_j} + \alpha = 0,
		\qquad 
		j \neq \ell.
		\label{eq::deruj}
	\end{align}
	Since $\sum_{j=1}^M p(j|x) = 1$ and $\sum_{j=1}^M u_j = 1$, \eqref{eq::deruj} yields $
	1 - p(\ell|x) 
	= \sum_{j \neq \ell} p(j|x)
	= \sum_{j \neq \ell} \alpha u_j 
	= \alpha (1 - u_{\ell}).
	$
	Therefore, we have 
	\begin{align}\label{eq::alphatemp}
		\alpha = (1 - p(\ell|x)) / (1 - u_{\ell}).
	\end{align}
	Plugging this into \eqref{eq::deruk}, we get
	$
	- p(\ell|x)/u_{\ell} - 2 (u_{\ell} - p(\ell|x)) + (1 - p(\ell|x))/(1 - u_{\ell}) = 0.
	$
	This implies that
	\begin{align}\label{eq::Aequ}
		\frac{(u_{\ell} - p(\ell|x))(1 - 2 u_{\ell} (1 - u_{\ell}))}{u_{\ell} (1 - u_{\ell})} 
		= 0
	\end{align}
	and thus $(u_{\ell} - p(\ell|x))(1 - 2 u_{\ell} (1 - u_{\ell})) = 0$.
	Since $u_{\ell} \in (0,1)$, we have 
	$u_{\ell} (1 - u_{\ell}) \leq 1/4$, which implies
	$2 u_{\ell} (1 - u_{\ell}) \leq 1/2 < 1$ and thus $1 - 2 u_{\ell} (1 - u_{\ell}) > 0$.
	Therefore, the solution of \eqref{eq::Aequ} is $u_{\ell} = p(\ell|x)$. This together with \eqref{eq::alphatemp} yields $\alpha = 1$. Then \eqref{eq::deruj} implies that $u_j = p(j|x)$ for $j \neq \ell$.
	Consequently, $u = (p(1|x), \ldots, p(M|x))$ is the minimizer of $h_{\ell}(u)$ for any $\ell \in [M]$.
	As a result, we have $h_m(u) \geq h_m(p(\cdot|x)) = 0$ for any $u \in (0,1)^M $ satisfying $\sum_{j=1}^M u_j = 1$. 
	Taking $u := \widehat{p}(\cdot|x)$, we get 
	\begin{align*}
		h_m(\widehat{p}(\cdot|x)) - h_m(p(\cdot|x)) 
		= \sum_{j=1}^M  - p(j|x) \log \frac{\widehat{p}(j|x)}{p(j|x)} - |\widehat{p}(m|x) - p(m|x)|^2 
		\geq 0. 
	\end{align*}
	This together with \eqref{eq::EP} yields the assertion.
\end{proof}

Now, with the aid of Lemmas \ref{lem::errorPY}, \ref{lem::integralerror}, \ref{lem::weightsigma} and \ref{lem::equilvalent}, we are able to prove Proposition \ref{prop::decompweighterror}.

\begin{proof}[Proof of Proposition \ref{prop::decompweighterror}]
	Using $(a + b + c)^2 \leq 3 (a^2 + b^2 + c^2)$, Lemma \ref{lem::integralerror} with $w := \widehat{w}$ and Lemma \ref{lem::errorPY}, we get
	\begin{align}\label{eq::PYXPY}
		& \sum_{m=1}^M  \biggl| \mathbb{E}_{x \sim q} \frac{p(m|x)}{\sum_{j=1}^M  \widehat{w}(j) p(j|x)} - p(m) \biggr|^2
		\nonumber\\
		& \leq 3 \sum_{m=1}^M  \biggl| \mathbb{E}_{x \sim q} \frac{p(m|x)}{\sum_{j=1}^M  \widehat{w}(j) p(j|x)} - \frac{1}{n_q} \sum_{i=n_p+1}^{n_p+n_q} \frac{\widehat{p}(m|X_i)}{\sum_{j=1}^M  \widehat{w}(j) \widehat{p}(j|X_i)} \biggr|^2 
		\nonumber\\
		& \phantom{=} 
		+ 3 \sum_{m=1}^M  \biggl| \frac{1}{n_q} \sum_{i=n_p+1}^{n_p+n_q} \frac{\widehat{p}(m|X_i)}{\sum_{j=1}^M  \widehat{w}(j) \widehat{p}(j|X_i)} - \widehat{p}(m) \biggr|^2 + 3 \sum_{m=1}^M  \biggl| \widehat{p}(m) - p(m) \biggr|^2
		\nonumber\\
		& \lesssim \frac{\log n_q}{n_q} + 
		\mathbb{E}_{x\sim p}\|p(\cdot|x)-\widehat{p}(\cdot|x)\|_2^2
		+ \sum_{m=1}^M  \biggl| \frac{1}{n_q} \sum_{i=n_p+1}^{n_p+n_q} \frac{\widehat{p}(m|X_i)}{\sum_{j=1}^M  \widehat{w}(j) \widehat{p}(j|X_i)} - \widehat{p}(m) \biggr|^2 + \frac{\log n_p}{n_p}
	\end{align}	
	with probability at least $1 - 1/n_p - 1/n_q$. Since $\widehat{w}$ is the minimizer of \eqref{eq::whatminimizer} and the inequality $(a + b + c)^2 \leq 3(a^2 + b^2 + c^2)$, we get
	\begin{align}\label{eq::hatPYXhatPY}
		& \sum_{m=1}^M  \biggl| \frac{1}{n_q} \sum_{i=n_p+1}^{n_p+n_q} \frac{\widehat{p}(m|X_i)}{\sum_{j=1}^M  \widehat{w}(j) \widehat{p}(j|X_i)} - \widehat{p}(m) \biggr|^2
		\leq \sum_{m=1}^M  \biggl| \frac{1}{n_q} \sum_{i=n_p+1}^{n_p+n_q} \frac{\widehat{p}(m|X_i)}{\sum_{j=1}^M  w^*(j) \widehat{p}(j|X_i)} - \widehat{p}(m) \biggr|^2 
		\nonumber\\
		& \leq 3 \sum_{m=1}^M  \biggl| \frac{1}{n_q} \sum_{i=n_p+1}^{n_p+n_q} \frac{\widehat{p}(m|X_i)}{\sum_{j=1}^M  w^*(j) \widehat{p}(j|X_i)} - \mathbb{E}_{x \sim q} \frac{p(m|x)}{\sum_{j=1}^M  w^*(j) p(j|x)} \biggr|^2 
		\nonumber\\
		& \phantom{=} 
		+ 3 \sum_{m=1}^M  \biggl| \mathbb{E}_{x \sim q} \frac{p(m|x)}{\sum_{j=1}^M  w^*(j) p(j|x)} - p(m) \biggr|^2 + 3 \sum_{m=1}^M  \bigl| p(m) - \widehat{p}(m) \bigr|^2
		\nonumber\\
		& \lesssim \log n_q/n_q + \mathbb{E}_{x \sim p} \|p(\cdot |x) - \widehat{p}(\cdot |x)\|_2^2 + (\log n_p)/n_p,
	\end{align}
	where the last inequality follows from Lemma \ref{lem::integralerror} 
	with $w := w^*$, Theorem \ref{lem::hold} and Lemma \ref{lem::errorPY}. Combining \eqref{eq::PYXPY} and \eqref{eq::hatPYXhatPY}, we obtain
	\begin{align*}
		\sum_{m=1}^M  \bigg|\mathbb{E}_{x \sim q} \frac{p(m|x)}{\sum_{j=1}^M  \widehat{w}(j) p(j|x)} - p(m)\bigg|^2 
		\lesssim (\log n_q)/n_q + \mathbb{E}_{x \sim p} \|p(\cdot |x) - \widehat{p}(\cdot |x)\|_2^2 + \log n_p / n_p
	\end{align*}
	with probability at least $1 - 1/n_p - 1/n_q$. This together with Lemma \ref{lem::weightsigma} and \ref{lem::equilvalent} yields that
	\begin{align*}
		\|\widehat{w} - w^*\|_2^2 
		& \lesssim \biggl\| \mathbb{E}_{x \sim q} \frac{p(\cdot|x)}{\sum_{j=1}^M  \widehat{w}(j) p(j|x)} - p(\cdot) \biggr\|_2^2
		\nonumber\\
		& \lesssim \mathbb{E}_{x \sim p} \|p(\cdot |x) - \widehat{p}(\cdot |x)\|_2^2 + \log n_q / n_q + \log n_p / n_p
		\nonumber\\
		& \lesssim \mathcal{R}_{L_{\mathrm{CE}},P}(\widehat{p}(y|x)) - \mathcal{R}_{L_{\mathrm{CE}},P}^* + \log n_q / n_q + \log n_p / n_p
	\end{align*}
	holds with probability at least $1 - 1/n_p - 1/n_q$. This finishes the proof. 
\end{proof}

\subsection{Proofs Related to Section \ref{sec::analysisQresult}} \label{sec::proofanalysisQ}

In order to prove Proposition \ref{prop::excesshattildediff}, we need the following lemma.

\begin{lemma}\label{lem::meanvalue}
	For any $a = (a_1, \ldots, a_M), z = (z_1, \ldots, z_M) \in \mathbb{R}^M$ with $a_m,z_m\in (0,1)$, $m \in [M]$ and $\sum_{m=1}^M a_m = \sum_{m=1}^M z_m = 1$, let $f(z) := - \sum_{m=1}^M a_m \log z_m$. Then for any $z' = (z'_1, \ldots, z'_M)$ with $z'_m\in(0,1)$, $m \in [M]$, and $\sum_{m=1}^M z'_m=1$, we have
	\begin{align*}
		|f(z) - f(z')| 
		\leq \sum_{m=1}^M \biggl( \frac{|a_m - z_m| \cdot |z'_m - z_m|}{z_m} + \frac{a_m (z'_m - z_m)^2}{z_m (z_m \wedge z'_m)} \biggr).
	\end{align*}
\end{lemma}

\begin{proof}[Proof of Lemma \ref{lem::meanvalue}]
	For any $z = (z_1, \ldots, z_M), z' = (z'_1, \ldots, z'_M)$ satisfying $z_m, z'_m\in (0,1)$, $m \in [M]$, and $\sum_{m=1}^M z_m = \sum_{m=1}^M z'_m = 1$, there holds
	\begin{align}\label{eq::twoterms}
		|f(z') - f(z)| 
		& = \biggl| \int_0^1 \bigl( \nabla f(z + t (z' - z)) \bigr)^{\top} (z' - z) \, dt \biggr|
		\nonumber\\
		& = \biggl| \int_0^1 \nabla f(z)^{\top} (z' - z) \, dt + \int_0^1 \bigl( \nabla f(z + t (z' - z)) - \nabla f(z)) \bigr)^{\top} (z' - z) \, dt \biggr|
		\nonumber\\
		& \leq \bigl| \nabla f(z)^{\top} (z' - z) \bigr| + \biggl| \int_0^1 \bigl( \nabla f(z + t (z' - z)) - \nabla f(z)) \bigr)^{\top} (z' - z) \, dt \biggr|.
	\end{align}
	Let us consider the first term in \eqref{eq::twoterms}. 
	By the definition of the function $f$, there holds
	\begin{align*}
		\nabla f(z) 
		= \nabla \biggl( - \sum_{m=1}^M a_m \log z_m \biggr)
		= \biggl( - \frac{a_1}{z_1}, \ldots, - \frac{a_M}{z_M} \biggr).
	\end{align*}
	Since $\sum_{m=1}^M z_m = \sum_{m=1}^M z'_m = 1$, we then have
	\begin{align}\label{eq::firstbound}
		\bigl| \nabla f(z)^{\top} (z' - z) \bigr| 
		& = \biggl| - \sum_{m=1}^M \frac{a_m}{z_m}(z'_m - z_m) \biggr| 
		= \biggl| \sum_{m=1}^M \biggl( 1 - \frac{a_m}{z_m} \biggr) (z'_m - z_m) \biggr|  
		\nonumber\\
		& \leq \sum_{m=1}^M \biggl| 1 - \frac{a_m}{z_m} \biggr| \cdot |z'_m - z_m| 
		= \sum_{m=1}^M \frac{|a_m - z_m| \cdot |z'_m - z_m|}{z_m}.
	\end{align}
	For the second term in \eqref{eq::twoterms}, there holds
	\begin{align}\label{eq::secondbound}
		& \biggl| \int_0^1 \bigl( \nabla f(z + t (z' - z)) - \nabla f(z) \bigr)^{\top} (z' - z) \, dt \biggr| 
		\nonumber\\
		& = \biggl| \int_0^1 \sum_{m=1}^M \biggl( - \frac{a_m}{z_m + t (z'_m - z_m)} + \frac{a_m}{z_m} \biggr) (z'_m - z_m) \, dt \biggr|
		\nonumber\\
		& = \biggl| \int_0^1 \sum_{m=1}^M \frac{a_m t (z'_m - z_m)^2}{z_m (z_m + t (z'_m - z_m))} \, dt \biggr|
		\leq \sum_{m=1}^M \frac{a_m (z'_m - z_m)^2}{z_m (z_m \wedge z'_m)}.
	\end{align}
	Combining \eqref{eq::twoterms}, \eqref{eq::firstbound}, and \eqref{eq::secondbound}, we obtain the assertion.
\end{proof}

\begin{proof}[Proof of Proposition \ref{prop::excesshattildediff}]
	By Assumption \ref{ass:labelshift} and \eqref{eq::qxoverpx}, we have $q(x) / p(x) 
	\leq 1 / p_{\min}(y)$. For $m\in [M]$ with $q(m)=0$, there holds $q(m|x) = 0$ for any $x\in \mathcal{X}$. This together with the definition of $\mathcal{R}_{L_{\mathrm{CE}},Q}$ implies
	\begin{align}\label{eq::truehattile}
		& \bigl| \mathcal{R}_{L_{\mathrm{CE}},Q}(\widetilde{q}(y|x)) - \mathcal{R}_{L_{\mathrm{CE}},Q}(\widehat{q}(y|x)) \bigr| 
		= \biggl| \mathbb{E}_{x \sim q} \sum_{m:q(m)>0} - q(m|x) \log \frac{\widehat{q}(m|x)}{\widetilde{q}(m|x)} \biggr|
		\nonumber\\
		&
		\leq \mathbb{E}_{x \sim q} \biggl| \sum_{m:q(m)>0} - q(m|x) \log \frac{\widehat{q}(m|x)}{\widetilde{q}(m|x)} \biggr|
		\lesssim \mathbb{E}_{x \sim p} \biggl| \sum_{m:q(m)>0} - q(m|x) \log \frac{\widehat{q}(m|x)}{\widetilde{q}(m|x)} \biggr|.
	\end{align}
	Applying Lemma \ref{lem::meanvalue} with $a := q(y|x)$, $z' := \widehat{q}(y|x)$ and $z := \widetilde{q}(y|x)$, we get
	\begin{align*}
		& \mathbb{E}_{x \sim p} \biggl| - \sum_{m:q(m)>0} q(m|x) \log \frac{\widehat{q}(m|x)}{\widetilde{q}(m|x)} \biggr|
		\nonumber\\
		& \leq \mathbb{E}_{x \sim p} \bigg(
		\sum_{m:q(m)>0} \frac{|q(m|x) - \widetilde{q}(m|x)|\cdot |\widetilde{q}(m|x) - \widehat{q}(m|x)|}{\widetilde{q}(m|x)}
		+ \frac{q(m|x)|\widetilde{q}(m|x) - \widehat{q}(m|x)|^2}{\widetilde{q}(m|x)(\widetilde{q}(m|x) \wedge \widehat{q}(m|x))}\bigg).
	\end{align*}
	This together with \eqref{eq::truehattile} yields
	\begin{align}\label{eq::ctildehat}
		& \bigl| \mathcal{R}_{L_{\mathrm{CE}},Q}(\widetilde{q}(y|x)) -  \mathcal{R}_{L_{\mathrm{CE}},Q}(\widehat{q}(y|x)) \bigr|
		\nonumber\\
		& \lesssim \mathbb{E}_{x \sim p} \bigg(\sum_{m:q(m)>0} \frac{|q(m|x) - \widetilde{q}(m|x)| \cdot |\widetilde{q}(m|x) - \widehat{q}(m|x)|}{\widetilde{q}(m|x)}
		+ \frac{q(m|x)|\widetilde{q}(m|x) - \widehat{q}(m|x)|^2}{\widetilde{q}(m|x)(\widetilde{q}(m|x)\wedge \widehat{q}(m|x))}\bigg).
	\end{align}
	Using the definition of $\widehat{q}(y|x)$ and $\widetilde{q}(y|x)$ as in \eqref{eq::hatetaQ} and \eqref{eq::tildeetaQ} and Assumption \ref{ass::regularity}, we then have 
	\begin{align}\label{eq::etaQhattildediff}
		\bigl| \widehat{q}(m|x) & - \widetilde{q}(m|x) \bigr| 
		= \biggl| \frac{\widehat{w}(m) \widehat{p}(m|x)}{\sum_{j=1}^M \widehat{w}(j) \widehat{p}(j|x)} -  \frac{w^*(m) \widehat{p}(m|x)}{\sum_{j=1}^M w^*(j) \widehat{p}(j|x)} \biggr|
		\nonumber\\
		& = \frac{|\sum_{j=1}^M \widehat{p}(j|x) (w^*(j) \widehat{w}(m) - w^*(m) \widehat{w}(j))| \cdot \widehat{p}(m|x)}{\bigl( \sum_{j=1}^M \widehat{w}(j) \widehat{p}(j|x)\bigr) \bigl( \sum_{j=1}^M w^*(j) \widehat{p}(j|x) \bigr)}
		\nonumber\\
		& \lesssim \sum_{j=1}^M (|w^*(j)\widehat{w}(m) - w^*(m) w^*(j)| + |w^*(m) w^*(j)-w^*(m) \widehat{w}(j)|) 
		\cdot \widehat{p}(m|x)
		\nonumber\\
		&\lesssim \widehat{p}(m|x) \|w^* - \widehat{w}\|_1.
	\end{align}
	In addition, we have 
	\begin{align}\label{eq::lowerhatetaQk}
		\widehat{q}(m|x)
		& = \frac{\widehat{w}(m) \widehat{p}(m|x)}{\sum_{j=1}^M \widehat{w}(j) \widehat{p}(j|x)}
		\geq \frac{\widehat{w}(m) \widehat{p}(m|x)}{\widehat{w}_{\max}}
		\gtrsim \widehat{p}(m|x),
	\end{align}
	where $\widehat{w}_{\max}:=\bigvee_{j=1}^M \widehat{w}(j)$.
	Similarly, for $m\in [M]$ with $q(m) > 0$, we have
	\begin{align}\label{eq::tildeetaQlower}
		\widetilde{q}(m|x)
		= \frac{w^*(m) \widehat{p}(m|x)}{\sum_{j=1}^M w^*(j) \widehat{p}(j|x)} 
		\geq \frac{w^*(m) \widehat{p}(m|x)}{w_{\max}^*} \gtrsim \widehat{p}(m|x),
	\end{align}
	where $w_{\max}^*:=\bigvee_{j=1}^M w^*(j)$. Plugging \eqref{eq::lowerhatetaQk} and \eqref{eq::tildeetaQlower} into \eqref{eq::ctildehat}, we obtain
	\begin{align}\label{eq::riskQtildehatdiff}
		& \bigl| \mathcal{R}_{L_{\mathrm{CE}},Q}(\widetilde{q}(y|x))- \mathcal{R}_{L_{\mathrm{CE}},Q}(\widehat{q}(y|x)) \bigr| 
		\nonumber\\
		& \lesssim  \mathbb{E}_{x \sim p} \bigg(\sum_{m:q(m)> 0} \frac{|q(m|x) - \widetilde{q}(m|x)| \cdot |\widetilde{q}(m|x) - \widehat{q}(m|x)|}{\widetilde{q}(m|x)} 
		+ \frac{q(m|x)|\widetilde{q}(m|x) - \widehat{q}(m|x)|^2}{\widetilde{q}(m|x)(\widetilde{q}(m|x) \wedge \widehat{q}(m|x))}\bigg)
		\nonumber\\
		&\lesssim \mathbb{E}_{x \sim p} \bigg(\sum_{m:q(m)> 0} \frac{|q(m|x) - \widetilde{q}(m|x)| \cdot |\widetilde{q}(m|x) - \widehat{q}(m|x)|}{ \widehat{p}(m|x)} 
		+ \frac{|\widetilde{q}(m|x) - \widehat{q}(m|x)|^2}{\widehat{p}(m|x)^2}\bigg).
	\end{align}
	Using the triangle inequality, Assumption \ref{ass::regularity} and \ref{ass::marginal} \textit{(iii)}, we get
	\begin{align}\label{eq::tildeetaQdiff}
		|\widetilde{q}(m|x) & - q(m|x)| 
		= \biggl| \frac{w^*(m) \widehat{p}(m|x)}{\sum_{j=1}^M w^*(j) \widehat{p}(j|x)} - \frac{w^*(m) p(m|x)}{\sum_{j=1}^M w^*(j)p(j|x)} \biggr|
		\nonumber\\
		& = \frac{w^*(m) |\sum_{j=1}^M w^*(j) ( p(j|x) \widehat{p}(m|x) - p(m|x) \widehat{p}(j|x) )|}{\bigl( \sum_{j=1}^M w^*(j) \widehat{p}(j|x) \bigr) \cdot \bigl( \sum_{j=1}^M w^*(j) p(j|x) \bigr)}
		\nonumber\\
		& \lesssim w^*(m) 
		\sum_{j=1}^M w^*(j) ( p(j|x) |\widehat{p}(m|x) - p(m|x)| + p(m|x) |p(j|x) - \widehat{p}(j|x)|)
		\nonumber\\
		& \lesssim \|\widehat{p}(\cdot|x) - p(\cdot|x)\|_1.
	\end{align}
	Combining \eqref{eq::riskQtildehatdiff}, \eqref{eq::etaQhattildediff} and \eqref{eq::tildeetaQdiff}, and using the inequality $ab \leq a^2 + b^2$ for $a,b>0$, and $\|v\|_1^2 \leq M \|v\|_2^2$ for the $M$-dimensional vector $v$, we obtain
	\begin{align*}
		& \bigl| \mathcal{R}_{L_{\mathrm{CE}},Q}(\widetilde{q}(y|x)) - \mathcal{R}_{L_{\mathrm{CE}},Q}(\widehat{q}(y|x)) \bigr|
		\\
		& \lesssim \mathbb{E}_{x \sim p} \bigg(\sum_{m:q(m)\neq 0} \|\widehat{p}(\cdot|x) - p(\cdot|x)\|_1 \cdot \|w^* - \widehat{w}\|_1
		+ \|w^* - \widehat{w}\|_1^2 \bigg)
		\\
		& \lesssim \mathbb{E}_{x \sim p} (\|\widehat{p}(\cdot|x) - p(\cdot|x)\|_1 \cdot \|w^* - \widehat{w}\|_1)
		+ \|w^* - \widehat{w}\|_1^2
		\\
		& \leq \mathbb{E}_{x \sim p} (\|\widehat{p}(\cdot|x) - p(\cdot|x)\|_1^2 + \|w^* - \widehat{w}\|_1^2)
		+ \|w^* - \widehat{w}\|_1^2 
		\\
		& \lesssim \mathbb{E}_{x \sim p} (\|\widehat{p}(\cdot|x) - p(\cdot|x)\|_2^2 + \|w^* - \widehat{w}\|^2_2)
		+ \|w^* - \widehat{w}\|_2^2
		\\
		& \lesssim \mathbb{E}_{x \sim p}\|\widehat{p}(\cdot|x) - p(\cdot|x)\|_2^2 + \|w^* - \widehat{w}\|^2_2
		\\
		& \lesssim \mathcal{R}_{L_{\mathrm{CE}},P}(\widehat{p}(y|x)) - \mathcal{R}_{L_{\mathrm{CE}},P}^* + \|w^* - \widehat{w}\|_2^2,
	\end{align*}
	where the last inequality follows from Lemma \ref{lem::equilvalent}, respectively.
	This finishes the proof.
\end{proof}

To prove Proposition \ref{prop::excesstildetruediff}, we need the following lemma.
\begin{lemma}\label{lem::etatransforminq}
	For any $m \in [M]$, let $a_m, b_m>0$, $c_m, z_m \in (0,1)$ satisfying $\sum_{m=1}^M c_m = 1$ and $\sum_{m=1}^M z_m = 1$. 
	Furthermore, let $a_{\max}:=\vee_{j=1}^M a_j$.
	Then we have
	\begin{align*}
		\sum_{m=1}^M \frac{a_m c_m}{a_{\max}} \log \frac{c_m / (\sum_{j=1}^M a_j c_j)}{z_m / (\sum_{j=1}^M a_j z_j)}
		\leq \sum_{m=1}^M \biggl( c_m \log \frac{c_m}{z_m} \biggr).
	\end{align*}
\end{lemma}

\begin{proof}[Proof of Lemma \ref{lem::etatransforminq}]
	Let the function $h : (0, 1)^M \to \mathbb{R}$ be defined by
	\begin{align*}
		h(z) 
		:= h(z_1, \ldots, z_M)
		:= \sum_{m=1}^M \biggl( \frac{a_m c_m}{a_{\max}} \log \frac{z_m/(\sum_{j=1}^M a_j z_j)}{c_m / (\sum_{j=1}^M a_j c_j)}
		+ c_m \log \frac{c_m}{z_m} \biggr) + \lambda \biggl( \sum_{m=1}^M z_m - 1 \biggr),
	\end{align*}
	where $\lambda > 0$ is the Lagrange multiplier. Then it suffices to prove that $h(z) \geq h(c) = 0$ for any $z$ satisfying $\sum_{j=1}^M z_j=1$ and $0< z_j < 1$, $j \in [M]$. Taking the partial derivative of $h(z)$ w.r.t.~$z_m$ and setting it to be zero, we have
	\begin{align}\label{eq::partial}
		\frac{\partial h(z)}{\partial z_m} = \frac{a_m c_m}{a_{\max}} \cdot \frac{1}{z_m} - \sum_{\ell=1}^M \frac{a_{\ell} c_{\ell}}{a_{\max}} \cdot \frac{a_{m}}{\sum_{j=1}^M a_j z_j} - \frac{c_m}{z_m} + \lambda = 0, 
		\qquad 
		m \in [M].
	\end{align}
	This implies 
	\begin{align*}
		\biggl( \frac{a_m}{a_{\max}} - 1 \biggr) c_m - \frac{a_m z_m}{a_{\max}} \cdot \frac{\sum_{j=1}^M a_j c_j}{\sum_{j=1}^M a_j z_j} + \lambda z_m = 0, 
		\qquad 
		m \in [M].
	\end{align*}
	Taking the summation over $m$ from $1$ to $M$ and using $\sum_{m=1}^M z_m = \sum_{m=1}^M c_m=1$, we get 
	\begin{align*}
		\sum_{m=1}^M \frac{a_m c_m}{a_{\max}} - 1 -  \frac{\sum_{j=1}^M a_j c_j}{a_{\max}} + \lambda = 0,
	\end{align*}
	which implies $\lambda = 1$. 
	This together with \eqref{eq::partial} yields 
	\begin{align}\label{eq::sumzero}
		\biggl( 1 - \frac{a_m}{a_{\max}} \biggr) \cdot \biggl( 1 -\frac{c_m}{z_m} \biggr) + \frac{a_m}{a_{\max}} \cdot \biggl( 1 - \frac{\sum_{j=1}^M a_j c_j}{\sum_{j=1}^M a_j z_j} \biggr) = 0, 
		\qquad 
		m \in [M].
	\end{align}
	If $\sum_{j=1}^M a_j c_j \neq \sum_{j=1}^M a_j z_j$, then there must exist some $\ell \in [M]$ such that $c_{\ell} \neq z_{\ell}$. Since $\sum_{j=1}^M z_j = \sum_{j=1}^M c_j = 1$, there exist some $i,j\in [M]$ such that $c_i > z_i$ and $c_j < z_j$. Without loss of generality, we assume that $\sum_{j=1}^M a_j c_j > \sum_{j=1}^M a_j z_j$. Therefore, we have both $1- c_i/z_i < 0$ and $1 - \sum_{j=1}^M a_j c_j / \sum_{j=1}^M a_j z_j <0$. 
	Thus we have
	\begin{align*}
		\biggl( 1 - \frac{a_i}{a_{\max}} \biggr) \cdot \biggl( 1 -\frac{c_i}{z_i} \biggr) + \frac{a_i}{a_{\max}} \cdot \biggl( 1 - \frac{\sum_{j=1}^M a_j c_j}{\sum_{j=1}^M a_j z_j} \biggr) < 0,
	\end{align*}
	which contradicts with \eqref{eq::sumzero} for $m=i$. 
	Therefore, we must have $\sum_{j=1}^M a_j c_j = \sum_{j=1}^M a_j z_j$, which together with \eqref{eq::sumzero} implies $c_m = a_m$ for any $m \in [M]$. 
	Thus, $h(z)$ attains its minimum at the point $z = (c_1, \ldots, c_m)$ under the constraint $\sum_{m=1}^M z_m = 1$, i.e., $h(z) \geq h(c) = 0$ holds for any $z$ with $\sum_{j=1}^M z_j=1$ and $z_j \in (0,1)$ for any $j \in [M]$. 
	This finishes the proof. 
\end{proof}

\begin{proof}[Proof of Proposition \ref{prop::excesstildetruediff}]
	Using Assumption \ref{ass:labelshift} and \eqref{eq::qxoverpx}, we obtain that for any $x \in \mathcal{X}$, there holds $q(x)/p(x) \leq 1/p_{\min}(y)$. Consequently we have 
	\begin{align}\label{eq::usingqx/px}
		\mathcal{R}_{L_{\mathrm{CE}},Q}(\widetilde{q}(y|x)) - \mathcal{R}_{L_{\mathrm{CE}},Q}^* 
		& = \mathbb{E}_{x \sim q} \sum_{m=1}^M q(m|x)\log \frac{q(m|x)}{\widetilde{q}(m|x)}
		\nonumber\\
		& \lesssim \mathbb{E}_{x \sim p} \sum_{m=1}^M q(m|x) \log \frac{q(m|x)}{\widetilde{q}(m|x)}.
	\end{align}
	For those $m \in [M]$ with $q(m)=0$, there holds $q(m|x)=0$ for any $x\in \mathcal{X}$. Using the definition of $\widetilde{q}(y|x)$ in
	\eqref{eq::tildeetaQ}, and applying Lemma \ref{lem::etatransforminq} with $a_m := q(m)/p(m)$, $b_m := (1-q(m))/(1-p(m))$, $c_m:=p(m|x)$, and $z_m:= \widetilde{p}(m|x)$, we obtain
	\begin{align*}
		& \sum_{m=1}^M q(m|x) \log \frac{q(m|x)}{\widetilde{q}(m|x)} = \sum_{m:q(m)>0} q(m|x) \log \frac{q(m|x)}{\widetilde{q}(m|x)}
		\\
		& = \sum_{m:q(m)>0} \frac{(q(m) / p(m)) p(m|x)}{\sum_{j=1}^M(q(j) / p(j)) p(j|x)} 
		\cdot \log \frac{(q(m) / p(m)) p(m|x)/(\sum_{j=1}^M(q(j) / p(j)) p(j|x))}{((q(m) / p(m)) \widehat{p}(m|x)/\sum_{j=1}^M(q(j) / p(j)) \widehat{p}(j|x))}
		\\
		& \leq \frac{\bigvee_{m=1}^M q(m)/p(m)}{\bigwedge_{m:q(m)>0} q(m)/p(m)}  
		\cdot 
		\\
		& \phantom{=} \qquad 
		\cdot \sum_{m:q(m)>0} \frac{(q(m) / p(m)) p(m|x)}{\bigvee_{m=1}^Mq(m)/p(m)} 
		\log \frac{(q(m) / p(m)) p(m|x)/(\sum_{j=1}^M(q(j) / p(j)) p(j|x))}{((q(m) / p(m)) \widehat{p}(m|x)/\sum_{j=1}^M(q(j) / p(j)) \widehat{p}(j|x))}
		\\
		& \lesssim  \sum_{m:q(m)>0}  p(m|x) \log \frac{p(m|x)}{\widehat{p}(m|x)} \leq \sum_{m\in [M]}  p(m|x) \log \frac{p(m|x)}{\widehat{p}(m|x)}
		\\
		& = \mathcal{R}_{L_{\mathrm{CE}},P}(\widehat{p}(y|x))-\mathcal{R}_{L_{\mathrm{CE}},P}^*.
	\end{align*}
	This together with \eqref{eq::usingqx/px} yields 
	the desired assertion. 
\end{proof}

\subsection{Proofs Related to Section \ref{subsec::RatesTarget}}\label{sec::proofRate}

\begin{proof}[Proof of Theorem \ref{thm::rateQ}]
	Combining \eqref{eq::ExcessRiskQdecomp},  Propositions \ref{prop::excesshattildediff} and \ref{prop::excesstildetruediff}, we get
	\begin{align*}
		\mathcal{R}_{L_{\mathrm{CE}},Q}(\widehat{q}(y|x)) - \mathcal{R}_{L_{\mathrm{CE}},Q}^* \lesssim \mathcal{R}_{L_{\mathrm{CE}},P}(\widehat{p}(y|x)) - \mathcal{R}_{L_{\mathrm{CE}},P}^* + \|w^* - \widehat{w}\|_2^2.
	\end{align*}
	Using Proposition \ref{prop::decompweighterror} and Theorem \ref{thm::convergencerate}, we obtain that for any $\xi\in (0,1/2)$, there holds
	\begin{align*}
		\mathcal{R}_{L_{\mathrm{CE}},Q}(\widehat{q}(y|x)) - \mathcal{R}_{L_{\mathrm{CE}},Q}^* &\lesssim  \mathcal{R}_{L_{\mathrm{CE}},P}(\widehat{p}(y|x)) - \mathcal{R}_{L_{\mathrm{CE}},P}^* + \log n_q / n_q + \log n_p / n_p 
		\nonumber\\
		&\lesssim n_p^{-\frac{(1+\beta)\alpha}{(1+\beta) \alpha + d}+\xi} + \log n_q / n_q 
	\end{align*}
	with probability at least $1-1/n_p-1/n_q$. 
\end{proof}

\begin{proof}[Proof of Theorem \ref{thm::lower}]
	Note that the lower bound of the excess risk consists of two parts depending on $n_p$ and $n_q$, respectively. Thus in the following, we prove the excess risk is larger than the two parts, respectively. First, we prove that the excess risk is larger than first part related to $n_p$. 
	To this end, we construct a sequence of the probability distribution $P$ as in Theorem \ref{thm::lowerKLR} and then we construct the probability distribution $Q$. To satisfy the label shift assumption, for any $\sigma^j \in \{-1,1\}^\ell$, $j=0, \ldots, 2^\ell-1$, we let $Q^{\sigma^j} := P^{\sigma^j}$. For the sake of convenience, we write $P^j := P^{\sigma^j}$ and $Q^j := Q^{\sigma^j}$. Correspondingly, we write $p^j(y|x) := P^{\sigma^j}(Y=y|X=x)$ and $q^j(y|x) := Q^{\sigma^j}(Y=y|X=x)$.

	\textit{Verification of the Conditions in Proposition \ref{prop::lower}}.
	Let $L = 2^\ell-1$, and we define the full sample distribution by
	$\Pi_{j} := P^{j\otimes n_p} \otimes Q_X^{j\otimes n_q}$,
	$j = 0, \ldots, L$.
	Moreover, we define the semi-metric $\rho$ in in Proposition \ref{prop::lower} by
	\begin{align*}
		\rho(q^i(\cdot|x), q^j(\cdot|x)) 
		&:= \int_{\mathcal{X}} \bigg(q^i(1|x)\log \frac{q^i(1|x)}{q^j(1|x)} + q^i(-1|x)\log \frac{q^i(-1|x)}{q^j(-1|x)}\bigg) q(x) \, dx = \mathrm{KL}(Q^i, Q^j) . 
	\end{align*}
	Therefore, for any predictor $\widehat{q}(y|x)$, we have $\mathcal{R}_{L_{\mathrm{CE}},Q}(\widehat{q}(y|x)) - \mathcal{R}_{L_{\mathrm{CE}},Q}^* = \rho(q(\cdot|x),\widehat{q}(\cdot|x))$.
	Since $Q^j = P^j$, the first and second condition in Proposition \ref{prop::lower} can be verified in the same way as in  Theorem \ref{thm::lowerKLR}. Thus it suffices to verify the third condition in Proposition \ref{prop::lower}. 
	By the independence of samples, $Q_X^j = Q_X^0 = P_X$ and \eqref{eq::KLuppbound},  we have for any $j\in \{0,1,\ldots, L\}$,
	\begin{align*}
		\mathrm{KL}(\Pi_{j}, \Pi_{0}) 
		& = n_p \mathrm{KL}(P^j, P^0) 
		+ n_q \mathrm{KL}(Q_X^j, Q_X^0) = n_p \mathrm{KL}(P^j, P^0) \leq  C_3 n_p r^{(1+\beta\wedge 1)\alpha}
		\\
		& = C_3 c_r^{(1+\beta\wedge 1)\alpha+d} r^{-d} \leq C_3 c_r^{(1+\beta\wedge 1)\alpha+d} 4^d \ell 
		\leq 2(\log 2)^{-1} C_3 c_r^{(1+\beta\wedge 1)\alpha+d} 4^d \log L,
	\end{align*}
	where the constant $C_3$ is defined in \eqref{eq::KLuppbound}. By choosing a sufficient small $c_r$ such that $2(\log 2)^{-1} C_3$ $c_r^{(1+\beta\wedge 1)\alpha+d} 4^d = 1/16$, we verify the third condition. Apply Proposition \ref{prop::lower}, 
	we obtain that for any estimator $\widehat{q}(y|x)$ built on $D_p \cup D_q^u$, with probability $P^{n_p} \otimes Q_X^{n_q}$ at least $(3-2\sqrt{2}) / 8$, there holds
	\begin{align}\label{eq::lower1part}
		\sup_{(P,Q) \in \mathcal{T}} \mathcal{R}_{L_{\mathrm{CE}},Q}(\widehat{q}(y|x)) - \mathcal{R}_{L_{\mathrm{CE}},Q}^* \geq (C_4/2) \cdot n_p^{-\frac{(1+\beta \wedge 1)\alpha}{(1+\beta \wedge 1)\alpha+d}}.
	\end{align}

	Next, we construct a new class of probability to prove the second part $n_Q^{-1}$ of the lower bound.
	Let $w := 1/16$ and $\delta>0$. Define the $1$-dimension class conditional densities:
	\begin{align*}
		q(x|1) := 
		\begin{cases}
			4w\delta & \text{ if } x \in [0,1/4],
			\\
			4(1-\delta) & \text{ if } x \in [3/8,5/8],
			\\
			4(1-w)\delta & \text{ if } x \in [3/4,1],
			\\
			0 & \text{ otherwise};
		\end{cases}
		\quad
		q(x|-1):= 
		\begin{cases}
			4(1-w)\delta & \text{ if } x \in [0,1/4],
			\\
			4(1-\delta) & \text{ if } x \in [3/8,5/8],
			\\
			4w\delta & \text{ if } x \in [3/4,1],
			\\
			0 & \text{ otherwise}.
		\end{cases}
	\end{align*}
	Let $\sigma \in \{-1,1\}$ and $\delta$ will be chosen later. We specify the class probabilities in the following way:
	$p(1) := p(y=1) := 1/2$, $q^{\sigma}(1) := q^{\sigma}(y=1) := (1+\sigma \theta) / 2$.
	Then we compute the conditional probability function $q^{\sigma}(1|x)$. 
	By the Bayes formula, we have 
	\begin{align*}
		q^{\sigma}(1|x) 
		& = \frac{q^\sigma(1) q(x|1)}{q^\sigma(1) q(x|1) + q^\sigma(-1) q(x|-1)} 
		\\
		& = \begin{cases}
			w (1 + \sigma \theta) / [w (1 + \sigma \theta) + (1 - w) (1 - \sigma \theta)] 
			& \text{ if } x \in [0,1/4], 
			\\
			(1 + \sigma \theta) / 2 & \text{ if } x \in [3/8,5/8],
			\\
			(1 - w) (1 + \sigma \theta) / [(1 - w) (1 + \sigma \theta) + w (1 - \sigma \theta)] 
			& \text{ if } x \in [3/4,1],
			\\
			1/2 & \text{ otherwise}.
		\end{cases}
	\end{align*}

	\textit{Verification of the Small Value Bound Condition.}
	Denote 
	$$
	t_1 := \frac{w(1- \theta)}{w(1 - \theta) + (1-w)(1 + \theta)}, 
	\qquad 
	t_2:= \frac{w(1+\theta)}{w(1+\theta) + (1-w)(1-\theta)}.
	$$
	For $t < t_1$, we have $Q^{\sigma}(q^{\sigma}(1|x) < t) = 0$. For $t \in [t_1, t_2)$, by taking $\theta := 1 / (16 \sqrt{n_q})$ and $\delta := c_{\beta} t_1^{\beta}$, we have 
	\begin{align*}
		Q^{\sigma}(q^{\sigma}(1|x) < t) 
		& = \eins\{\sigma = -1\}Q([0,1/4]) 
		\\
		& = \eins\{\sigma = -1\}\big((1-\theta)w\delta/2 + (1+\theta)(1-w)\delta/2 \big) 
		\\
		&= \eins\{\sigma = -1\}(1+\theta-2 \theta w)\delta/2 \leq \delta =   c_{\beta} t_1^{\beta} \leq c_{\beta}t^{\beta}. 
	\end{align*} 
	Moreover, for $t \in [t_2, (1-\theta)/2)$, we have 
	$$
	Q^{\sigma}(q^{\sigma}(1|x) < t) 
	= Q([0,1/4]) 
	= (1+\theta-2\theta w)\delta/2 
	\leq \delta 
	\leq  c_{\beta}t_1^{\beta} 
	\leq c_{\beta} t^{\beta}.
	$$
	Otherwise if $t \in [(1-\theta)/2, 1/2]$, by taking $c_{\beta} := 4^{\beta}$, there holds 
	$$
	Q^{\sigma}(q^{\sigma}(1|x) < t) = Q([0,1/4] \cup [3/8,5/8]) =  (1+\theta-2\theta w)\delta/2 + 1-\delta \leq 1 \leq  c_{\beta} t^{\beta}.
	$$

	We define our distribution class $\mathcal{K} := \{ \Pi^{\sigma}: \sigma \in \{-1,1\}\}$, where $\Pi^{\sigma}$ is defined as $\Pi^{\sigma} := P^{n_p} \otimes (Q^{\sigma})^{n_q}$.
	Then using the inequality $\log((1+x)/(1-x)) \leq 3x$ for $0\leq x \leq 1/2$, the Kullback-Leibler divergence between $\Pi^{-1}$ and $\Pi^{1}$ is 
	\begin{align*}
		& \mathrm{KL}(\Pi^{-1}|\Pi^{1}) = n_q D(Q^{1} | Q^{-1})
		\\
		&= n_q \big( \log \bigl( (1 + \theta) / (1 - \theta) \bigr) (1 + \theta) / 2  + \log \bigl( (1 - \theta) / (1 + \theta) \bigr) (1 - \theta) / 2 \big)
		\\
		&= 2\theta n_q \log \bigl( (1 + \theta) / (1 - \theta) \bigr)
		\leq 6\theta^2 n_q= 3/128.
	\end{align*}
	Since $|\mathcal{K}| = 2$, $\Pi^1 \ll \Pi^{-1}$ and $|\mathcal{K}|^{-1} \mathrm{KL}(\Pi^{-1}|\Pi^{1})  = 3/256  < (\log 2)/8$.
	Then we calculate the excess risk. Define $q^{\sigma}(y|x) := Q^{\sigma}(Y=y|X=x)$ and $q^{\sigma}(y) := Q^{\sigma}(Y=y)$. 
	The semi-metric $\rho$ is defined by
	\begin{align*}
		& \rho(q^{1}(y|x), q^{-1}(y|x)) 
		:= \int_{\mathcal{X}} \bigg(q^{1}(1|x)\log \frac{q^{1}(1|x)}{q^{-1}(1|x)} + q^{1}(-1|x) \log \frac{q^{1}(-1|x)}{q^{-1}(-1|x)}\bigg) q(x) \, dx
		\\
		& = \int_{[0,1/4]} \log(t_2/t_1)q^1(1) q(x|1) + \log((1-t_2)/(1-t_1))q^1(-1) q(x|-1) \,dx 
		\\
		& \phantom{=} 
		+ \int_{[3/8,5/8]} \log((1+\theta)/(1-\theta)) q^1(1) q(x|1) + \log((1-\theta)/(1+\theta)) q^1(-1) q(x|-1) \,dx
		\\
		& \phantom{=} 
		+ \int_{[3/4,1]} \log(t_2/t_1) q^1(1) q(x|1) + \log((1-t_2)/(1-t_1))q^1(-1) q(x|-1) \,dx 
		\\
		& = \theta\log((1+\theta)/(1-\theta)) + \theta t(-2w+1) \log((1+\theta-2\theta w)/(1-\theta+2\theta w))
		\\
		& \geq \frac{\theta^2}{1-\theta} + \frac{\theta^2 t (1-2w)^2}{1-\theta+2\theta w} \geq 256^{-1}(1+(7/8)^2 4^{\beta}) n_q^{-1} =: 2c_2 n_q^{-1},
	\end{align*}
	where the second last inequality is due to $\log(1+x) \geq x/2$ for $x \in (0,1)$ and $c_2 := 512^{-1}(1+(7/8)^2 4^{\beta})$. By Proposition \ref{prop::lower}, we then obtain that with probability $\Pi$ at least $(3-2\sqrt{2})/8$, there holds
	\begin{align}\label{eq::lower2part}
		\sup_{\Pi \in\mathcal{K}} \mathcal{R}_{L_{\mathrm{CE}}, Q}(\widehat{q}(y|x)) - \mathcal{R}_{L_{\mathrm{CE}}, Q}^* \geq c_2 n_q^{-1}.
	\end{align}     
	Combining \eqref{eq::lower1part} and \eqref{eq::lower2part}, we obtain that for any $\widehat{q}(y|x)$ built on $D_p \cup D_q$, with probability $P^{n_p} \otimes Q_X^{n_q}$ at least $(3-2\sqrt{2})/8$, there holds
	\begin{align*}
		\sup_{(P,Q) \in \mathcal{T}}  \mathcal{R}_{L_{\mathrm{CE}}, Q}(\widehat{q}(y|x)) - \mathcal{R}_{L_{\mathrm{CE}}, Q}^*
		\geq c_{\ell} \Bigl( n_p^{-\frac{(1+\beta\wedge 1)\alpha}{(1+\beta \wedge 1)\alpha+d}}+ n_q^{-1} \Bigr),
	\end{align*}
	where $c_{\ell} := C_4/2 \wedge c_2$. This finishes the proof.
\end{proof}

\section{Conclusion} \label{sec::Conclusion}

Domain adaptation involves two distinct challenges: covariate shift and label shift adaptation problems. Covariate shift adaptation, where data distribution differences are due to feature probability variations, is typically addressed by \textit{feature probability matching} (\textit{FPM}). In contrast, label shift adaptation, where variations in class probability solely cause distribution differences, traditionally also employs FPM in the multi-dimensional feature space to calculate class probability ratios in the one-dimensional label space. To more effectively tackle label shift adaptation, we introduce a new approach, \textit{class probability matching} (\textit{CPM}), inspired by a new representation of the source domain's class probability. This method aligns class probability functions in the one-dimensional label space, differing fundamentally from FPM's multi-dimensional feature space approach. Additionally, we integrate kernel logistic regression into the CPM framework for conditional probability estimation, resulting in a new algorithm, \textit{class probability matching using kernel methods} (\textit{CPMKM}), specifically for label shift adaptation. From a theoretical standpoint, we establish CPMKM's optimal convergence rates concerning the cross-entropy loss in multi-class label shift adaptation. Experimentally, CPMKM has shown superior performance over existing FPM-based and maximum-likelihood-based methods in real data comparisons.

\bibliographystyle{plain}
\small{\bibliography{CPMKM}}
\end{document}